\documentclass[twoside]{article}

\usepackage[accepted]{aistats2026}

\usepackage[round]{natbib}

\usepackage[utf8]{inputenc} %
\usepackage[T1]{fontenc}    %
\usepackage{hyperref}       %
\usepackage{url}            %
\usepackage{booktabs}       %
\usepackage{amsfonts}       %
\usepackage{nicefrac}       %
\usepackage{microtype}      %
\usepackage[dvipsnames]{xcolor}         %
\usepackage{subcaption}

\usepackage[algo2e, linesnumbered, ruled, vlined]{algorithm2e}
\SetKwInput{KwInput}{Input}
\SetKwInput{KwOutput}{Output}
\SetKwInput{Input}{Input}
\SetKwInput{Return}{Return}

\usepackage{cancel}

\def\safedef#1{%
   \ifx#1\undefined
      \expandafter\def\expandafter#1%
   \else
      \errmessage{The \string#1 is defined already}%
      \expandafter\def\expandafter\tmp
   \fi
}

\usepackage{amsthm,amsmath,bbm,amsfonts,amssymb}
\usepackage{dsfont} %
\usepackage{mathtools}
\usepackage{commath}
\usepackage{eqnarray}
\mathtoolsset{showonlyrefs=false}
\allowdisplaybreaks %

\usepackage{xspace}
\usepackage[T1]{fontenc}
\usepackage[utf8]{inputenc}
\usepackage[dvipsnames]{xcolor} %

\usepackage{hyperref,url}

\usepackage{wrapfig}
\usepackage[export]{adjustbox}
\usepackage{graphicx,tabularx}
\usepackage{multirow,hhline}%
\usepackage{booktabs}%
\usepackage{pbox} %
\usepackage{algorithm,algorithmic}

\usepackage[export]{adjustbox}

\usepackage{enumitem}

\newcommand{\kj}[1]{{\color{RedOrange}[#1]}}

{\color{kjsaved}}%

\usepackage{framed}

\usepackage{mdframed}
\usepackage{lipsum}
\definecolor{kjgray}{rgb}{.7,.7,.7}
\makeatletter



\makeatletter
\renewcommand{\paragraph}{%
  \@startsection{paragraph}{4}%
  {\z@}{0.50ex \@plus 1ex \@minus .2ex}{-1em}%
  {\normalfont\normalsize\bfseries}%
}
\makeatother

\usepackage[customcolors,shade]{hf-tikz} %

\usepackage{tabularx} 
\newcolumntype{P}[1]{>{\centering\arraybackslash}p{#1}}
\newcolumntype{M}[1]{>{\centering\arraybackslash}m{#1}}

\def\ddefloop#1{\ifx\ddefloop#1\else\ddef{#1}\expandafter\ddefloop\fi}

\def\ddef#1{\expandafter\def\csname #1#1\endcsname{\ensuremath{\mathbb{#1}}}}
\ddefloop ABCDFGHIJKLMNORSTUWXYZ\ddefloop 

\def\ddef#1{\expandafter\def\csname c#1\endcsname{\ensuremath{\mathcal{#1}}}}
\ddefloop ABCDEFGHIJKLMNOPQRSTUVWXYZ\ddefloop

\def\ddef#1{\expandafter\def\csname b#1\endcsname{\ensuremath{{\mathbf{#1}}}}}
\ddefloop ABCDEFGHIJKLMNOPQRSTUVWXYZ\ddefloop  
\def\ddef#1{\expandafter\def\csname b#1\endcsname{\ensuremath{{\boldsymbol{#1}}}}}
\ddefloop abcdeghijklmnopqrtsuvwxyz\ddefloop  %

\def\ddef#1{\expandafter\def\csname h#1\endcsname{\ensuremath{\hat{#1}}}}
\ddefloop ABCDEFGHIJKLMNOPQRSTUVWXYZabcdefghijklmnopqrsuvwxyz\ddefloop %
\def\ddef#1{\expandafter\def\csname hc#1\endcsname{\ensuremath{\hat{\mathcal{#1}}}}}
\ddefloop ABCDEFGHIJKLMNOPQRSTUVWXYZ\ddefloop
\def\ddef#1{\expandafter\def\csname hb#1\endcsname{\ensuremath{\hat{\mathbf{#1}}}}}
\ddefloop ABCDEFGHIJKLMNOPQRSTUVWXYZ\ddefloop %
\def\ddef#1{\expandafter\def\csname hb#1\endcsname{\ensuremath{\hat{\boldsymbol{#1}}}}}
\ddefloop abcdefghijklmnopqrstuvwxyz\ddefloop %

\def\ddef#1{\expandafter\def\csname t#1\endcsname{\ensuremath{\tilde{#1}}}}
\ddefloop ABCDEFGHIJKLMNOPQRSTUVWXYZabcdefgijklmnpqtsuvwxyz\ddefloop %
\def\ddef#1{\expandafter\def\csname tc#1\endcsname{\ensuremath{\tilde{\mathcal{#1}}}}}
\ddefloop ABCDEFGHIJKLMNOPQRSTUVWXYZ\ddefloop
\def\ddef#1{\expandafter\def\csname tb#1\endcsname{\ensuremath{\tilde{\mathbf{#1}}}}}
\ddefloop ABCDEFGHIJKLMNOPQRSTUVWXYZ\ddefloop
\def\ddef#1{\expandafter\def\csname tb#1\endcsname{\ensuremath{\tilde{\boldsymbol{#1}}}}}
\ddefloop abcdefghijklmnopqrstuvwxyz\ddefloop %

\def\ddef#1{\expandafter\def\csname bar#1\endcsname{\ensuremath{\bar{#1}}}}
\ddefloop ABCDEFGHIJKLMNOPQRSTUVWXYZabcdefghijklmnopqrtsuvwxyz\ddefloop
\def\ddef#1{\expandafter\def\csname barc#1\endcsname{\ensuremath{\bar{\mathcal{#1}}}}}
\ddefloop ABCDEFGHIJKLMNOPQRSTUVWXYZ\ddefloop
\def\ddef#1{\expandafter\def\csname barb#1\endcsname{\ensuremath{\bar{\mathbf{#1}}}}}
\ddefloop ABCDEFGHIJKLMNOPQRSTUVWXYZ\ddefloop
\def\ddef#1{\expandafter\def\csname barb#1\endcsname{\ensuremath{\bar{\boldsymbol{#1}}}}}
\ddefloop abcdefghijklmnopqrstuvwxyz\ddefloop %

\def\ddef#1{\expandafter\def\csname war#1\endcsname{\ensuremath{\overline{#1}}}}
\ddefloop ABCDEFGHIJKLMNOPQRSTUVWXYZabcdefghijklmnopqrtsuvwxyz\ddefloop
\def\ddef#1{\expandafter\def\csname warc#1\endcsname{\ensuremath{\overline{\mathcal{#1}}}}}
\ddefloop ABCDEFGHIJKLMNOPQRSTUVWXYZ\ddefloop
\def\ddef#1{\expandafter\def\csname warb#1\endcsname{\ensuremath{\overline{\mathbf{#1}}}}}
\ddefloop ABCDEFGHIJKLMNOPQRSTUVWXYZ\ddefloop
\def\ddef#1{\expandafter\def\csname warb#1\endcsname{\ensuremath{\overline{\boldsymbol{#1}}}}}
\ddefloop abcdefghijklmnopqrstuvwxyz\ddefloop %

\usepackage{pgffor}
\def\greeksymbols{alpha,beta,gamma,gam,delta,dt,eps,epsilon,zeta,eta,theta,th,iota,kappa,kap,lambda,lam,mu,nu,xi,pi,rho,sigma,sig,tau,phi,chi,psi,omega,om,Gamma,Gam,Delta,Dt,Theta,Th,Lambda,Lam,Pi,Sigma,Sig,Phi,Psi,Omega,Om}
\def\greeksymbolsnoeta{alpha,beta,gamma,gam,delta,dt,eps,epsilon,zeta,theta,th,iota,kappa,kap,lambda,lam,mu,nu,xi,pi,rho,sigma,sig,tau,phi,chi,psi,omega,om,Gamma,Gam,Delta,Dt,Theta,Th,Lambda,Lam,Pi,Sigma,Sig,Phi,Psi,Omega,Om} %

\foreach \x in \greeksymbolsnoeta{\expandafter\xdef\csname b\x\endcsname{\noexpand\ensuremath{\noexpand\boldsymbol{\csname \x\endcsname}}}}

\foreach \x in \greeksymbols{\expandafter\xdef\csname h\x\endcsname{\noexpand\ensuremath{\noexpand\hat{\csname \x\endcsname}}}}
\foreach \x in \greeksymbolsnoeta{\expandafter\xdef\csname hb\x\endcsname{\noexpand\ensuremath{\noexpand\hat{\noexpand\boldsymbol{ \csname \x\endcsname}}}}}

\foreach \x in \greeksymbols{\expandafter\xdef\csname bar\x\endcsname{\noexpand\ensuremath{\noexpand\bar{\csname \x\endcsname}}}}
\foreach \x in \greeksymbolsnoeta{%
\expandafter\xdef\csname barb\x\endcsname{\noexpand\ensuremath{\noexpand\bar{\noexpand\boldsymbol{ \csname \x\endcsname}}}}
}

\foreach \x in \greeksymbols{\expandafter\xdef\csname t\x\endcsname{\noexpand\ensuremath{\noexpand\tilde{\csname \x\endcsname}}}}
\foreach \x in \greeksymbolsnoeta{\expandafter\xdef\csname tb\x\endcsname{\noexpand\ensuremath{\noexpand\tilde{\noexpand\boldsymbol{ \csname \x\endcsname}}}}}

\providecommand{\normz}[2][-1]{
\ensuremath{\mathinner{
\ifthenelse{\equal{#1}{-1}}{ %
\!\left\|#2\right\|}{}
\ifthenelse{\equal{#1}{0}}{ %
\|#2\|}{}
\ifthenelse{\equal{#1}{1}}{ %
\bigl\|#2\bigr\|}{}
\ifthenelse{\equal{#1}{2}}{ %
\Bigl\|#2\Bigr\|}{}
\ifthenelse{\equal{#1}{3}}{ %
\biggl\|#2\biggr\|}{}
\ifthenelse{\equal{#1}{4}}{ %
\Biggl\|#2\Biggr\|}{}
}} %
}  %

\providecommand{\floor}[2][-1]{
\ensuremath{\mathinner{
\ifthenelse{\equal{#1}{-1}}{ %
\!\left\lfloor#2\right\rfloor}{}
\ifthenelse{\equal{#1}{0}}{ %
\lfloor#2\rfloor}{}
\ifthenelse{\equal{#1}{1}}{ %
\!\bigl\lfloor#2\bigr\rfloor}{}
\ifthenelse{\equal{#1}{2}}{ %
\!\Bigl\lfloor#2\Bigr\rfloor}{}
\ifthenelse{\equal{#1}{3}}{ %
\!\biggl\lfloor#2\biggr\rfloor}{}
\ifthenelse{\equal{#1}{4}}{ %
\!\Biggl\lfloor#2\Biggr\rfloor}{}
}} %
}

\providecommand{\ceil}[2][-1]{
\ensuremath{\mathinner{
\ifthenelse{\equal{#1}{-1}}{ %
\!\left\lceil#2\right\rceil}{}
\ifthenelse{\equal{#1}{0}}{ %
\lceil#2\rceil}{}
\ifthenelse{\equal{#1}{1}}{ %
\!\bigl\lceil#2\bigr\rceil}{}
\ifthenelse{\equal{#1}{2}}{ %
\!\Bigl\lceil#2\Bigr\rceil}{}
\ifthenelse{\equal{#1}{3}}{ %
\!\biggl\lceil#2\biggr\rceil}{}
\ifthenelse{\equal{#1}{4}}{ %
\!\Biggl\lceil#2\Biggr\rceil}{}
}} %
}

\definecolor{mygrn}{rgb}{0,.8,0}
\definecolor{myred}{rgb}{.8,0,0}

\DeclareMathOperator{\KL}{{\mathsf{KL}}}

\DeclarePairedDelimiterX{\inp}[2]{\langle}{\rangle}{#1, #2}

\newcommand\declareop[3]{%
  \newcommand#1{%
    \mskip\muexpr\medmuskip*#2\relax
    {#3}%
    \mskip\muexpr\medmuskip*#2\relax
}}
\declareop\capprox{1}{{\sr{\const}{\approx}}} %
\declareop\logapprox{1}{{\sr{\mathsf{log}}{\approx}}} %

\def\const{\mathsf{const}}

\def\TV{\mathsf{TV}}

\def\kl{{\mathsf{kl}}}

\usepackage{pifont}%

\newcommand{\sr}{\stackrel}

\makeatletter
\newcommand{\vast}{\bBigg@{3}}
\newcommand{\Vast}{\bBigg@{4}}
\makeatother

\newcommand{\stkout}[1]{\ifmmode\text{\sout{\ensuremath{#1}}}\else\sout{#1}\fi}

\let\vec\undefined %
\DeclareMathOperator{\vec}{\text{\normalfont vec}}

\newenvironment{talign*}
 {\csname align*\endcsname}
 {\endalign}

\def\chrulefill{\leavevmode\leaders\hrule height 0.7ex depth \dimexpr0.4pt-0.7ex\hfill\kern0pt}

\usepackage[normalem]{ulem}

\usepackage{amsmath,amsfonts,bm,bbm,mathrsfs}
\usepackage{amssymb,mathtools}

\def\eqref#1{(\ref{#1})}

\def\vs{{\bm{s}}}

\def\vx{{\bm{x}}}

\def\mB{{\bm{B}}}

\def\mD{{\bm{D}}}
\def\mE{{\bm{E}}}

\def\mI{{\bm{I}}}

\def\mP{{\bm{P}}}

\def\mT{{\bm{T}}}
\def\mU{{\bm{U}}}
\def\mV{{\bm{V}}}
\def\mW{{\bm{W}}}
\def\mX{{\bm{X}}}
\def\mY{{\bm{Y}}}

\DeclareMathAlphabet{\mathsfit}{\encodingdefault}{\sfdefault}{m}{sl}
\SetMathAlphabet{\mathsfit}{bold}{\encodingdefault}{\sfdefault}{bx}{n}

\def\gA{{\mathcal{A}}}

\def\gC{{\mathcal{C}}}
\def\gD{{\mathcal{D}}}
\def\gE{{\mathcal{E}}}
\def\gF{{\mathcal{F}}}
\def\gG{{\mathcal{G}}}
\def\gH{{\mathcal{H}}}
\def\gI{{\mathcal{I}}}

\def\gL{{\mathcal{L}}}
\def\gM{{\mathcal{M}}}

\def\gO{{\mathcal{O}}}
\def\gP{{\mathcal{P}}}
\def\gQ{{\mathcal{Q}}}

\def\gS{{\mathcal{S}}}
\def\gT{{\mathcal{T}}}

\def\gX{{\mathcal{X}}}

\def\sN{{\mathbb{N}}}

\def\sP{{\mathbb{P}}}

\def\sR{{\mathbb{R}}}

\newcommand{\E}{\mathbb{E}}

\newcommand{\Var}{\mathrm{Var}}

\DeclareMathOperator*{\argmax}{arg\,max}
\DeclareMathOperator*{\argmin}{arg\,min}

\theoremstyle{plain}
\newtheorem{assumption}{Assumption}
\newtheorem{theoremplain}{Theorem}[section]
\newtheorem{proposition}{Proposition}[section]
\newtheorem{lemma}{Lemma}[section]
\newtheorem{claim}{Claim}[section]

\newtheorem{remark}{Remark}

\usepackage[most]{tcolorbox}
\tcbuselibrary{theorems, breakable}

\newtcbtheorem[number within=section]{definition}{Definition}%
{enhanced, breakable,
  colback=teal!5!white,
  colframe=teal!60!black,
  fonttitle=\bfseries}{def}

\newtcbtheorem[number within=section]{theorem}{Theorem}%
{enhanced, breakable,
  colback=blue!3!white,
  colframe=blue!60!black,
  fonttitle=\bfseries}{thm}

\def\ceil#1{\lceil #1 \rceil}
\def\floor#1{\lfloor #1 \rfloor}

\def\1{\mathbbm{1}}

\newcommand{\diag}{\mathrm{diag}}

\newcommand{\rank}{\mathrm{rank}}

\newcommand{\Bin}{\mathrm{Bin}}

\usepackage{dsfont}
\usepackage{cleveref}
\usepackage{stmaryrd}
\usepackage{doi}

\usepackage{thmtools}
\usepackage{thm-restate}

\def\bignorm#1{\left\lVert #1 \right\rVert}
\def\bigabs#1{\left| #1 \right|}

\newcommand{\indicator}{\mathds{1}}

\definecolor{darkblue}{rgb}{0.0,0.0,0.65}
\definecolor{darkred}{rgb}{0.65,0.0,0.0}
\definecolor{darkgreen}{rgb}{0.0,0.5,0.0}
\definecolor{lightbrown}{rgb}{0.71,0.4,0.11}
\definecolor{tab:blue}{RGB}{31,119,180}  %
\definecolor{tab:red}{RGB}{214,39,40}  %
\definecolor{tab:green}{RGB}{44,160,44}  %
\definecolor{tab:orange}{RGB}{255,127,14}  %
\hypersetup{
	colorlinks = true,
	citecolor  = darkblue,
	linkcolor  = darkred,
	filecolor  = darkblue,
	urlcolor   = darkblue,
}

\newcommand{\Unif}{\mathrm{Unif}}

\newcommand{\ps}{\mathrm{ps}}
\newcommand{\mix}{\mathrm{mix}}

\newcommand{\thres}{\mathrm{thres}}
\newcommand{\Sym}{\mathrm{Sym}}

\usepackage{tablefootnote,threeparttable}
\usepackage{multicol}
\usepackage{multirow}
\usepackage{makecell}

\newcommand{\kjnew}[1]{{\color{MidnightBlue}#1}}

\definecolor{OursEM1}{HTML}{D55E00}     %
\definecolor{OursEM10}{HTML}{C24700}    %
\definecolor{OursStage1}{HTML}{E5864A}  %

\definecolor{Oracle}{HTML}{009E73}      %

\definecolor{KausikP10}{HTML}{2F65A2}   %
\definecolor{KausikP30}{HTML}{3F87C5}   %
\definecolor{KausikP50}{HTML}{85C3E9}   %

\allowdisplaybreaks

\newif\ifFINAL
\FINALtrue %

\FINALfalse %

\ifFINAL

\def\guide#1{}
\def\kj#1{}
\def\kjnew#1{}
\else
{}
\fi

\begin{document}

\twocolumn[

\aistatstitle{Near-Optimal Clustering in Mixture of Markov Chains}

\aistatsauthor{ Junghyun Lee \And Yassir Jedra \And  Alexandre Prouti\`{e}re \And Se-Young Yun }

\aistatsaddress{ KAIST AI \\ \texttt{jh\_lee00@kaist.ac.kr} \And ICL EEE \\ \texttt{y.jedra@imperial.ac.uk} \And KTH EECS, Digital Futures \\ \texttt{alepro@kth.se} \And KAIST AI \\ \texttt{yunseyoung@kaist.ac.kr}}
]

\begin{abstract}
	We study the problem of clustering $T$ trajectories of length $H$, each generated by one of $K$ unknown ergodic Markov chains over a finite state space of size $S$.
    We derive an instance-dependent, high-probability lower bound on the clustering error rate, governed by the stationary-weighted KL divergence between transition kernels.
    We then propose a two-stage algorithm: Stage I applies spectral clustering via a new injective Euclidean embedding for ergodic Markov chains, a contribution of independent interest enabling sharp concentration results; Stage II refines clusters with a single likelihood-based reassignment step.
    We prove that our algorithm achieves near-optimal clustering error with high probability under reasonable requirements on $T$ and $H$.
    Preliminary experiments support our approach, and we conclude with discussions of its limitations and extensions.
\end{abstract}
\section{INTRODUCTION}
\label{sec:introduction}
Clustering, or community detection for graphs, is a fundamental statistical problem with applications across diverse scientific disciplines, including social science, biology, and statistical physics~\citep{mclachlan2019finite,kiselev2019rna,ezugwu2022clustering,fortunato2010graphs}.
The precise statistical characterization of clustering has been rigorously investigated under various probabilistic frameworks, such as stochastic block models (SBMs, \cite{abbe2018sbm}), Gaussian mixture models (GMMs, \cite{lu2016gmm,loffler2021gmm,chen2024gmm}), and block Markov chains (BMCs, \cite{sanders2020bmc,jedra2023bmdp}).

However, the aforementioned models primarily focus on clustering static data points or nodes, where individual elements often lack inherent informative structure in isolation: the clustering signal typically emerges only when considering the entire data.
In contrast, many real-world applications involving multiple underlying processes, such as astronomy~\citep{yang2020trajectory}, mobile social networks~\citep{tang2021trajectory}, and human activity patterns~\citep{zhou2021human}, deal with \textit{trajectories} with a mixture of temporal information.
The task is to cluster the trajectories based on their generating process or model.
The longer each trajectory (e.g., a user's interaction sequence or a time series) is, the more information it potentially reveals about its generating model, facilitating clustering.

The \textbf{Mixture of Markov Chains (MMC)} provides a fundamental yet powerful probabilistic framework for the problem of trajectory clustering in the most basic scenario with no controllable actions.
In this model, we are given $T$ trajectories of length $H$, where each trajectory is generated by one of $K$ unknown Markov chains defined over a finite state space of size $S$.
This model has been studied in a variety of contexts, originally by \citet{blumen1955industrial} for modeling heterogeneous labor mobility patterns, and later utilized for learning underlying usage patterns from user trails of app usages, music playlists, web browsing, and more~\citep{gupta2016mixture,spaeh2023mixture,maystre2022mmc,cadez2003mmc,poulsen1990mmc,girolami2003mmc,zhou2021human}.

In \textbf{MMC}, two primary and closely intertwined objectives emerge: \textit{learning} the underlying $K$ Markov chain models~\citep{gupta2016mixture,spaeh2023mixture,kausik2023mixture} and \textit{clustering} the observed trajectories according to their generative source~\citep{kausik2023mixture}.
Successful clustering can significantly simplify the learning task by allowing for information aggregation within each identified group of trajectories. 
Conversely, while learning each model accurately can indeed facilitate clustering, it is expected to necessitate stringent requirements on the trajectory length $H$, because each trajectory is solely responsible for its own amount of information helpful for clustering.
Thus, beyond its intrinsic value as an exploratory data analysis technique, accurate clustering can lead to statistically more efficient inference than analyzing each trajectory independently.

As done in prior clustering literature, one can ask two critical questions:
\emph{(i) What is the fundamental limit on the misclassification rate?} \emph{(ii) Given sufficiently large (but not excessively so) trajectories, is there a \textit{computationally tractable} algorithm whose performance (clustering error rate) matches the lower bound? What is the requirement on $T$ and $H$?}

Several studies have algorithmically addressed the second question~\citep{gupta2016mixture,kausik2023mixture,spaeh2023mixture,spaeh2024mixture1,spaeh2024mixture2}, often without rigorous statistical guarantees.
One notable exception is \cite{kausik2023mixture}, where the authors proposed a clustering algorithm that provably performs \textit{exact} clustering (i.e., zero misclassification error) in \textbf{MMC} when $T = \widetilde{\Omega}(K^2 S)$ and $H = \widetilde{\Omega}(K^{3/2} t_\mix)$.
However, they do not provide the clustering error rate across various regimes of $T$, $H$ and appropriate separation parameters.
Perhaps more importantly, they lack a corresponding lower bound and related discussions, making it difficult to assess the optimality of their algorithm.
Furthermore, their algorithm relies on explicit knowledge of problem-specific quantities, often unavailable in practice.
We will elaborate on this comparison in Section~\ref{sec:compare-kausik}.
Consequently, despite their importance, both of the aforementioned questions have remained elusive in the literature.

\paragraph{Contributions.}
In short, we answer both questions fully.
Specifically:
\emph{(a)} We prove an instance-specific high-probability lower bound on the clustering error rate for \textbf{MMC}. This reveals the problem-difficulty quantity $\gD$: the minimum weighted KL divergence between the transition kernels (Section~\ref{sec:lowerbound}).
\emph{(b)} We propose a two-stage clustering algorithm that achieves near-optimal clustering error. Notably, it does not require any \textit{a priori} knowledge of the underlying model, yet fully adapts to the given problem difficulty (Section~\ref{sec:upperbound}). Especially for Stage I, we introduce a new injective Euclidean embedding ($L$-embedding, \Cref{def:l-embedding}) specifically designed for ergodic Markov chains. This embedding, a contribution of independent interest, facilitates sharp concentration results for spectral clustering analysis (Section~\ref{sec:spectral}).
\emph{(c)} Our upper and lower bounds reveal gaps in misclassification errors and the required trajectory length $H$. Building on recent advances in concentration inequalities~\citep{paulin2015markov, fan2021hoeffding} and estimation techniques~\citep{wolfer2021markov} for Markov chains, we elucidate the inherent complexities of clustering in \textbf{MMC} that currently render these gaps unavoidable (Section~\ref{sec:discussions}).

\paragraph{Notation.}
For a positive integer $n \geq 1$, let $[n] := \{1, 2, \cdots, n\}$.
For a set $X$, let $\Delta(X)$ be the set of probability distributions over $X$.
Let $a \vee b := \max\{a, b\}$ and $a \wedge b := \min\{a, b\}$.
We will freely utilize the asymptotic notations $\gO, o, \Omega, \omega, \Theta$, and for aesthetic purpose, we will also use $f \gtrsim g$, $f \lesssim g$, $f \asymp g$, defined as $f = \Omega(g)$, $f = \gO(g)$, $f = \Theta(g)$, respectively; $\log$ in the subscript (e.g., $\lesssim_{\log}$) indicates up to logarithmic factors, corresponding to $\widetilde{\Omega}$, $\widetilde{\gO}$, and $\widetilde{\Theta}$.

\section{PROBLEM SETTING}
\label{sec:setting}
\paragraph{Mixture of Markov chains (MMCs).}
There are $K$ unknown Markov chains. The $k$-th Markov chain is denoted as $\gM^{(k)} = (\gS, H, \mu^{(k)}, p^{(k)})$. $\gS$ is a finite state space of cardinality $S$, $H \geq 2$ is the horizon or (episode length), $\mu^{(k)}$ is a initial state distribution, and $p^{(k)}(\cdot | \cdot)$ and $\mP^{(k)}$ are the transition kernel and matrix.

We now recall an essential concept for spectral analysis of Markov chains.
For an ergodic Markov chain with transition matrix $\mP$, its \textbf{pseudo-spectral gap} is defined as $\gamma_\ps := \max_{m \geq 1} \frac{1}{m} \gamma\left( (\mP^*)^m \mP^m \right)$, where $\gamma(\cdot)$ is the spectral gap of the self-adjoint operator~\citep[Section 3]{paulin2015markov}.
With this, we assume the following:
\begin{assumption}
\label{assumption:ergodic}
    Each $\gM^{(k)}$ is ergodic,\footnote{We assume uniform ergodicity: there exist $M > 0, \rho \in (0, 1)$ such that $\max_{s \in \gS} \TV( P^H(s, \cdot), \pi) \leq M \rho^H$ for all $H \in \sN$, which is implied by aperiodicity and irreducibility~\citep[Theorem 4.9]{markovmixing}.} with pseudo-spectral gap $\gamma_\ps^{(k)} > 0$ and stationary distribution $\pi^{(k)}$.
    Also, $\gamma_\ps := \min_k \gamma_\ps^{(k)}$ and $\pi_{\min} := \min_{k, s} \pi^{(k)}(s) > 0.$
\end{assumption}

\paragraph{Learner's Objective.}
We first clarify that the learner knows $\gS$, and for the simplicity of exposition, we also assume that the learner knows \textit{either} the number of clusters $K$ \textit{or} the minimum pseudo-spectral gap $\gamma_\ps$.\footnote{A fully “parameter-free” algorithm would require neither. In fact, by leveraging recent advances in tight estimation of $\gamma_\ps$~\citep{wolfer2024mixing}, our algorithm can be made parameter-free through an additional initialization step, as we elaborate below Theorem~\ref{thm:initial-spectral}.}
With this, the learner observes $T$ trajectories generated as follows: for each $t \in [T]$, a trajectory of length $H$ is sampled from $\gM^{(f(t))}$.
Specifically, trajectory $\gT_t = (s_{t,1}, \dots, s_{t,H})$ is generated as $s_{t,1} \sim \mu^{(f(t))}$ and $s_{t,h+1} \sim p^{(f(t))}(\cdot | s_{t,h})$, which we denote as $\gT_t \sim \gM^{(f(t))}$.
$f: [T] \rightarrow [K]$ is the unknown, ground-truth decoding function that maps each trajectory to its generating Markov chain.

The learner’s objective is to recover $f$—that is, to cluster the $T$ trajectories according to their generating models. Let $\widehat{f} : [T] \rightarrow [K]$ denote the learner's estimated decoding function. The number of misclassified trajectories is measured as:
$$
E_T(\widehat{f}, f) := \min_{\sigma \in \Sym(K)} \sum_{t=1}^T \mathds{1}[\widehat{f}(t) \neq \sigma(f(t))],
$$
where $\Sym(K)$ is the symmetric group over $[K]$. The goal is to design an algorithm such that, for any confidence level $\delta \in (0, 1)$ and misclassification rate $\varepsilon \in (0, 1)$, the following guarantee holds: $\sP ( E_T(\widehat{f}, f) \leq \varepsilon T ) \geq 1 - \delta$, where the probability $\mathbb{P}$ is over the randomness of both the trajectory generation and the clustering algorithm.
We note in advance that the error rate $\varepsilon$ and the failure probability $\delta$ may both depend on factors such as $T$, $H$, $S$, $K$, and the mixing properties of the Markov chains.

\section{FUNDAMENTAL LOWER BOUND}
\label{sec:lowerbound}
We define an instance of \textbf{clustering in MMC} as a tuple $\Phi_T := ((\gM^{(k)})_{k \in [K]}, f, T)$, where $(\gM^{(k)})_{k \in [K]}$ is a collection of $K$ ergodic Markov chains, $f$ is the decoding function, and $T \in \mathbb{N}$ is the total number of trajectories to be clustered.

To derive instance-specific lower bounds, we must consider algorithms that for a given instance, genuinely adapt to perturbations of the decoding function. Indeed, an algorithm that would return $\widehat{f}=f$, would have no misclassified trajectories, but would fail for other decoding functions. Hence, we now introduce a class of ``good'' algorithms that are robust to small instance perturbations with a given confidence level. Define $\bm\alpha(f) := (\alpha_k(f))_{k \in [K]} \in \Delta([K])$ with $\alpha_k(f) := |f^{-1}(k)| / T$.

\begin{definition}{Stable Algorithms}{stable}
    Let $(\varepsilon, \beta, \delta) \in [0, 1] \times \sR_{\geq 0} \times (0, 1/2]$.
    A clustering algorithm $\gA$ is {\bf $(\varepsilon, \beta, \delta)$-locally stable} at $\Phi_T := ((\gM^{(k)})_{k \in [K]}, f, T)$ if the following holds: 
    
    for all $\Phi_T'=((\gM^{(k)})_{k \in [K]}, f', T)$ such that $\lVert \bm\alpha(f') - \bm\alpha(f) \rVert_2 \leq \beta$,
    \begin{equation}
    \label{eqn:stability}
        \sP_{\Phi_T', \gA} \left( E_T(\widehat{f}_\gA, f') > \varepsilon T \right) \leq \delta,
    \end{equation}
    where $\widehat{f}_\gA$ is the outputted clustering function from our algorithm $\gA$ that takes as input the $T$ trajectories whose ground-truth clustering is given by $f'$.
    Furthermore, we say that $\gA$ is {\bf $(\varepsilon, \beta)$-asymptotically locally stable} at $\Phi := ((\gM^{(k)})_{k \in [K]}, f)$ if there exists a sequence $(\delta_T)_{T \in \sN}$ with $\lim_{T \rightarrow \infty} \delta_T = 0$ such that $\gA$ is $(\varepsilon, \beta, \delta_T)$-locally stable at $\Phi_T$ for all $T \in \sN$.
\end{definition}
Intuitively, perturbations are (slight) changes in the relative cluster sizes $\bm\alpha(f)$, inspired by the prior definitions of locally stable algorithms for clustering in SBMs~\citep[Definition 1]{yun2019optimal}, block Markov chains~\citep{sanders2020bmc}, and recently, block Markov Decision Processes~\citep{jedra2023bmdp}.

We now present our \textit{instance-specific, high-probability} lower bound for clustering in mixture of Markov chains:
\begin{theorem}{Error Rate Lower Bound}{lower-bound}
    Let $(\varepsilon, \delta) \in [0, 1] \times (0, 1/2]$.
    Then, a necessary condition for the existence of a $(\varepsilon, \beta, \delta)$-locally stable algorithm at $\Phi_T := ((\gM^{(k)})_{k \in [K]}, f, T)$ with $\beta \geq 2\sqrt{2} \varepsilon$ is as follows: denoting $\alpha_{\min} = \min_{k \in [K]} \alpha_k(f)$,
    \begin{align}
        &\delta \geq \frac{1}{2} \left( \frac{\alpha_{\min}}{16 e \varepsilon} \right)^{\varepsilon T} \exp\left( -4 \varepsilon T (H - 1) \gD \right) \nonumber \\
        &\Longleftrightarrow
        4 (H - 1) \gD \geq \frac{1}{\varepsilon T}\log\frac{1}{2\delta} + \log\frac{\alpha_{\min}}{16 e \varepsilon}. \label{eqn:necessary}
    \end{align}
    For $(\varepsilon, \beta)$-asymptotically locally stable algorithm with the same $\beta$ as above and $\varepsilon = o(1)$, we have the following necessary condition:
    \begin{equation}
        \liminf_{T \rightarrow \infty} \frac{2 (H - 1) \gD}{\log \frac{\alpha_{\min}}{16 e \varepsilon}} \geq 1.
    \end{equation}
    
    The {\bf information-theoretic divergence} $\gD$ is defined as $\gD := \min_{k \neq k' \in [K]} \gD^{(k, k')}$, where
    \begin{align*}
        \gD^{(k, k')} &:= \frac{1}{H - 1} \KL\left( \mu^{(k)}, \mu^{(k')} \right) \\
        &\quad + \sum_{s \in \gS} \sP_H^{(k)}(s) \KL\left( p^{(k)}(\cdot | s), p^{(k')}(\cdot | s) \right),
    \end{align*}
    where $\sP_H^{(k)}(s) := \frac{1}{H - 1} \sum_{h=1}^{H-1} \sP^{(k)}(s_h = s)$.
\end{theorem}
\begin{proof}[Proof sketch]
    The proof proceeds by first constructing a large number of hypotheses, each corresponding to different cluster allocations of the observations, following the combinatorial arguments of \citet{yun2019optimal}.
    These hypotheses are slight perturbations of the given instance of clustering in MMC.
    We then apply the change-of-measure argument~\citep{lai1985efficient} and the data-processing inequality~\citep[Lemma 1]{garivier2019lower} to relate the error of any “good’’ clustering algorithm to the KL divergences between true and alternate models.
    Finally, we optimize the resulting lower bound over the different allocations.
    The approach is inspired by \citet[Theorem 1]{yun2019optimal} and \citet[Theorem 2]{jedra2023identification}, which establish high-probability lower bounds for clustering in SBMs and for linear system identification, respectively.
    The full proof is deferred to Appendix~\ref{app:lower-bound}.
\end{proof}

\begin{remark}[Chernoff Information]
    \citet{dreveton2024universal} recently highlighted Chernoff information as a universal measure of clustering difficulty for sub-exponential mixtures, extending earlier results in SBM clustering~\citep{abbe2015sbm1,abbe2015sbm2,dreveton2023sbm}.
    While we use $\KL(\cdot, \cdot)$, the two quantities are closely related~\citep{vanErven2014renyi} and become equivalent as $T \to \infty$ in many settings, including labeled SBMs~\citep[Claim 4]{yun2016labeled}.
\end{remark}

\begin{remark}[Comparison to \citet{wang2023divergence}]
    One notable related result is \citet[Theorem 4.1]{wang2023divergence}, which establishes an asymptotic equipartition property~\citep[Theorem 11.8.1]{infotheory} of binary hypothesis testing between two Markov chains from $\{(X_{h,1},X_{h,2})\}_{h \in [H]}$, where $X_{h,1} \overset{i.i.d.}{\sim} \nu$ and $X_{h,2} \sim P(\cdot | X_{h,1})$.
    Our asymptotic lower bound is a generalization of theirs, as we consider the full Markov chain instead of a collection of single hops.
\end{remark}

\paragraph{Asymptotically Exact Recovery.}
The necessary condition for asymptotically exact recovery ($\varepsilon = \gO(1 / T)$) is $H \gD = \Omega(\log T)$.
This mirrors the minimax lower bound for GMM clustering~\citep[Theorem 3.3]{lu2016gmm}, where strong consistency demands that the squared signal-to-noise ratio (SNR) is at least logarithmic w.r.t. the number of data points $n$: $\Delta^2 / \sigma^2 = \Omega(\log n)$, with $\Delta$ denoting the minimum $\ell_2$-separation of the mean vectors and $\sigma^2$ the isotropic variance of the GMM.
Analogously, in our setting, $\gD$ corresponds to $\Delta^2$, and $H$ plays a role similar to $1 / \sigma^2$: longer trajectories reduce variance, much like smaller $\sigma^2$ does in the GMM.

\paragraph{Stationary Divergence.}
Without loss of generality, suppose that $\max_{k,k'} \KL(\mu^{(k)}, \mu^{(k')}) < \infty$.
Then, under Assumptions~\ref{assumption:ergodic} (ergodicity), by the Ergodic Theorem for Markov chains~\citep[Appendix C]{markovmixing}), $\gD^{(k,k')} \rightarrow \gD_\pi^{(k,k')}$ almost surely as $H \rightarrow \infty$, where
\begin{equation}
\label{eqn:D-pi}
    \gD_\pi^{(k,k')} := \sum_{s \in \gS} \pi^{(k)}(s) \KL\left( p^{(k)}(\cdot | s), p^{(k')}(\cdot | s) \right).
\end{equation}
The \textbf{stationary (information-theoretic) divergence} is then defined as $\gD_\pi := \min_{k \neq k'} \gD_\pi^{(k,k')}$, which quantifies the difficulty of clustering in MMC in the stationary regime.
It aggregates the ``local'' difficulty at each state $s$—quantified by the KL-divergence between outgoing transition probabilities—weighted by the long-run frequency $\pi^{(k)}(s)$ of visiting that state.
Thus, a state's contribution to distinguishability is small if \emph{either} the chains behave similarly from that state or the state is rarely visited.
This aggregation is information-geometrically optimal, as it weights the KL divergence -- the optimal statistic for local hypothesis testing under the Neyman-Pearson lemma~\citep{neyman-pearson} -- by the stationary distribution.

\section{TWO-STAGE ALGORITHM}
\label{sec:upperbound}

Motivated by prior clustering literature~\citep{gao2017sbm,yun2016labeled,lu2016gmm}, our algorithm consists of two stages.
The full pseudocodes are provided in Algorithms~\ref{alg:initial-spectral} and \ref{alg:likelihood}.

\subsection{Stage I. Initial Spectral Clustering}
\label{sec:spectral}

\paragraph{$L$-Embedding.}
We first introduce a new Euclidean embedding of ergodic Markov chains:

\begin{definition}{$L$-Embedding}{L-embedding}
\label{def:l-embedding}
    For any ergodic Markov chain $\gM$ with transition probability matrix $\mP$ and stationary distribution $\pi$, its \textbf{$L$-embedding} and its empirical version $\widehat{L}$ are defined as
    \begin{align*}
        &L : \gM \mapsto \vec\left(\diag(\pi)^{1/2} \mP\right), \\
        &\widehat{L} : \gT = (s_1, s_2, \cdots, s_H) \mapsto \left( \frac{\widehat{N}(s, s')}{\sqrt{H \widehat{N}(s)}} \right)_{s, s' \in \gS},
    \end{align*}
    where $\gT \sim \gM$, $\widehat{N}(s) := \sum_{h=1}^H \indicator[s_h = s]$, and $\widehat{N}(s, s') := \sum_{h=1}^{H-1} \indicator[s_h = s, s_{h+1} = s']$.
\end{definition}

With this, we define the ground-truth and empirical data matrices of size $T \times S^2$ by row-concatenating the $L$-embeddings as follows:
\begin{equation}
\label{eqn:W}
    \mW = \begin{bmatrix}
        L(\gM^{(f(t))})
    \end{bmatrix}_{t \in [T],:}, \quad
    \widehat{\mW} = \begin{bmatrix}
        \widehat{L}(\gT_t)
    \end{bmatrix}_{t \in [T],:}.
\end{equation}

\begin{remark}
    To the best of our knowledge, such a Euclidean embedding for ergodic Markov chains has not been previously reported in the literature.
    In Appendix~\ref{app:comparisons}, we discuss its connections and distinctions to existing spectral-related representations, including diffusion maps~\citep{coifman2006diffusion}, the (weighted) Laplacian~\citep{chung1997spectral}, the symmetrized form~\citep{markovmixing}, and doublet frequencies~\citep{wolfer2021markov,vidyasagar2014markov}.
\end{remark}

We now discuss three desirable properties of the $L$-embedding that make it suitable for our purpose.

First, our $L(\cdot)$ uniquely determines $\gM$, up to the initial distribution:

\begin{proposition}
\label{prop:L}
    Over the space of ergodic Markov chains, $L$ is an injective mapping up to the initial distribution: for $\gM = (\gS, H, \mu, p)$ and $\gM' = (\gS, H, \mu', p')$, $L(\gM) = L(\gM') \Longleftrightarrow p = p'$.
\end{proposition}
\begin{proof}
    By definition, we must have that for all $s, s' \in \gS$, $\sqrt{\pi(s)} p(s' | s) = \sqrt{\pi'(s)} p'(s' | s)$.
    Summing over $s' \in \gS$, we have that $\pi(s) = \pi'(s) > 0$ for all $s \in \gS$, where the ergodicity of the chains implies strict positivity.
    This then implies that $p = p'$.
\end{proof}

Second, the $L$-embedding possesses desirable concentration properties.
The empirical mapping $\widehat{L}$ normalizes the transition counts by $\sqrt{\widehat{N}(s)}$.
To understand why this is advantageous, note that we can rewrite each entry of the empirical $L$-embedding as
\begin{equation}
    \frac{\widehat{N}(s, s')}{\sqrt{H \widehat{N}(s)}} = \underbrace{\frac{\widehat{N}(s, s')}{\widehat{N}(s)}}_{\approx p(s' | s)} \underbrace{\sqrt{\frac{\widehat{N}(s)}{H}}}_{\approx \sqrt{\pi(s)}}.
\end{equation}
Because the concentration bound of the empirical transition vector for state $s$, $\left( \widehat{N}(s, s') / \widehat{N}(s) \right)_{s' \in \gS}$, scales as $\tilde{\gO}(1 / \sqrt{H \pi(s)})$ (see Eqn.~\eqref{eqn:p-concentration} in Appendix~\ref{app:initial-spectral}), multiplying by $\sqrt{\pi(s)}$---which is effectively what the empirical $\sqrt{\widehat{N}(s)}$ normalization approximates---homogenizes the rate of concentration across all states $s \in \gS$.
This variance stabilization improves the conditioning of the spectral representation and yields a sharp concentration rate for $\bignorm{\widehat{\mW} - \mW}_{2 \to \infty}$ that completely avoids a detrimental dependence on $\pi_{\min}^{-1}$.
This tightly controlled error is critical for the theoretical guarantees of our initial spectral clustering, as we will highlight in the proof sketch of Theorem~\ref{thm:initial-spectral}.

Third, our $L$ facilitates a more ``direct'' comparison between the $\ell_2$-based separation $\Delta_\mW^2$ and the KL-based separation $\gD$.
The minimum (row-wise) $\ell_2$-separation of $\mW$, which is a crucial quantity for $\ell_2$-distance-based spectral clustering, is defined as:
\begin{equation*}
\resizebox{\columnwidth}{!}{$
    \begin{aligned}
        \Delta_\mW^2 &:= \min_{k \neq k'} \bignorm{L(\gM^{(k)}) - L(\gM^{(k')})}_2^2 \\
        &= \min_{k \neq k'} \sum_{s \in \gS} \bignorm{\sqrt{\pi^{(k)}(s)} p^{(k)}(\cdot | s) - \sqrt{\pi^{(k')}(s)} p^{(k')}(\cdot | s)}_2^2.
    \end{aligned}
$}
\end{equation*}
By Proposition~\ref{prop:L}, $\Delta_\mW = 0$ iff at least two of the $K$ Markov models are exactly the same.
This also provides an intuitive motivation for our choice of $L$.
Informally, $\Delta_\mW^2$ can be viewed as roughly corresponding to $\min_{k \neq k'} \sum_{s \in \gS} \pi^{(k)}(s) \bignorm{p^{(k)}(\cdot | s) - p^{(k')}(\cdot | s)}_2^2$, which, up to a constant, is upper bounded by $\gD_\pi$ (see Proposition~\ref{prop:gaps}).
We discuss in more detail the relations between different gaps in Section~\ref{sec:discussions}\textbf{(3)}.

\paragraph{Algorithm Design.}
Stage I is then the usual (row-wise) spectral clustering~\citep{vonLuxburg2007spectral} applied to $\widehat{\mW}$, divided into two cases.
If $K$ is given, we perform standard spectral clustering (lines 2--4) by computing the spectral embedding using the top-$K$ singular vectors and then applying $K$-means row-wise.
If $K$ is unknown, the procedure consists of three sub-stages~\citep{kannan2009spectral,yun2016labeled,Vuren2024bmc}.
First, we perform \emph{adaptive thresholding}, where the rank of $\mW$ is estimated via singular value thresholding as $\widehat{R}$,\footnote{It could be that $\rank(\mW) \ll K$; see Appendix~\ref{app:rank-K}.} and a low-rank spectral embedding is formed (lines 6--7).
Second, we perform \emph{density-based clustering}, where for each trajectory $t$, we construct a neighborhood $\gQ_t$ based on $\ell_2$ distance in the embedding space, and then greedily form clusters by selecting ${\color{blue}t_{k^\star}}$---the trajectory with the largest uncovered neighborhood---as a cluster ``center'' and assigning its neighbors (lines 8--15).
Finally, for the remaining trajectories, we perform \emph{assignment}, where each trajectory is assigned to its nearest cluster center in $\ell_2$ distance (lines 16--17).

\begin{algorithm2e}[!t]
\caption{Initial Spectral Clustering}
\label{alg:initial-spectral}
\Input{$\{(s_{t,1}, \cdots, s_{t,H}) \}_{t \in [T]}$, $K$ or $\gamma_\ps$}

$\widehat{\mU} \widehat{\bm\Sigma} \widehat{\mV}^\top \gets $SVD of $\widehat{\mW} \in [0, 1]^{T \times S^2}$ (see Eqn.~\eqref{eqn:W}), with $\widehat{\bm\Sigma} := \diag(\{\widehat{\sigma}_i\}_{i \in [\min(T, S^2)]})$\;

\eIf{$K$ is given}{
    $\widehat{K} \gets K$, $\widehat{\mX} \gets \widehat{\mU}_{1:K} \widehat{\bm\Sigma}_{1:K}$\;

    $\widehat{f}_0$ $\gets$ $K$-means clustering of the rows of $\widehat{\mX}$\;
}{
    \tcc{\textcolor{blue}{1. Adaptive Thresholding}}
    $\widehat{R} \gets \sum_{i=1}^{T \wedge S^2} \indicator\left[ \widehat{\sigma}_i \geq \widehat{\sigma}_\thres \triangleq \sqrt{\frac{64 T S}{H \gamma_\ps} \log\frac{T H}{\delta}} \right]$ \;
    
    $\widehat{\mX} \gets \widehat{\mU}_{1:\widehat{R}} \widehat{\bm\Sigma}_{1:\widehat{R}}$
    
    \tcc{\textcolor{blue}{2. Density-based Clustering}}
    Define $\gQ_t \gets \left\{ t' \in [T] : \bignorm{\widehat{\mX}_{t',:} - \widehat{\mX}_{t,:}}_2^2 \leq (\widehat{\sigma}_{\thres})^2 \right\}$ for each $t \in [T]$\;
    
    Initialize $\gS_0 \gets \emptyset$, $k \gets 1$, and $\rho \gets T$\;
    \While{$\rho \geq \frac{32 \widehat{R} T}{\log\frac{TH}{\delta}}$}{
        ${\color{blue}t_k^\star} \gets \argmax_{t \in [T]} \left| \gQ_t \setminus \bigcup_{\ell=0}^{k-1} \gS_\ell \right|$\;
        $\gS_k \gets \gQ_{{\color{blue}t_k^\star}} \setminus \bigcup_{\ell=0}^{k-1} \gS_\ell$\;
        $\rho \gets |\gS_k|$, $k \gets k + 1$\;
    }
    
    $\widehat{K} := k - 1$ \tcp*[r]{Estimated \# of clusters}
    
    $\widehat{f}_0(t) := k$ for $t \in \gS_k$ and $k \in [\widehat{K}]$\;

    \tcc{\textcolor{blue}{3. $\ell_2$-Distance-based Clustering}}
    \For{$t \in [T] \setminus \bigcup_{k=1}^{\widehat{K}} \gS_k$}{
        $\widehat{f}_0(t) := \argmin_{k \in [\widehat{K}]} \bignorm{\widehat{\mX}_{{\color{blue}t_k^\star},:} - \widehat{\mX}_{t,:}}_2$
    }
}

\Return{$\widehat{K}, \ \widehat{f}_0 : [T] \rightarrow [\widehat{K}]$}
\end{algorithm2e}

\paragraph{Theoretical Analysis.}
We now present the error rate of Stage I:
\begin{theorem}{Error Rate of Stage I}{initial-spectral}
    Let $\delta \in (0, 1)$.
    Suppose that Assumption~\ref{assumption:ergodic} holds, and that the trajectories are sufficiently long in the following sense:
    \begin{equation*}
        H \gtrsim \frac{1}{\pi_{\min} \gamma_\ps} \log\frac{1}{\pi_{\min} \delta} \vee \frac{S}{\Delta_\mW^2 \gamma_\ps} \log\frac{H}{\delta} \log\frac{TH}{\delta}.
    \end{equation*}
    Then, we have:
    denoting $R = \rank(\mW)$,
    \begin{equation*}
        \sP\left( E_T(\widehat{f}_0, f) \lesssim \frac{T R S}{H \gamma_\ps \Delta_\mW^2} \log\frac{TH}{\delta} \right) \geq 1 - \delta.
    \end{equation*}
\end{theorem}
\begin{proof}[Proof sketch]
    We first prove $\|\mW - \widehat{\mW}\|_{2 \rightarrow \infty} \lesssim \sqrt{\tfrac{S}{H \gamma_\ps} \log\tfrac{T H}{\delta}}$ (Lemma~\ref{lem:concentration-W}), our key technical novelty.
    From here, the proof largely follows that of \citet[Algorithm 2]{yun2016labeled}.
    Note that the decay rate is \textit{independent} of $\pi_{\min}^{-1}$, which is crucial in obtaining a tight requirement on $H$ when combined with Stage II.
    Such rate is possible as the weighting by $\sqrt{\pi(s)}$ in our embedding $L$ cancels out with both $\widetilde{\gO}\left(1 / \sqrt{H \pi(s)}\right)$ for $\ell_2$-estimation of $p(\cdot | s)$ and $\widetilde{\gO}\left(1 / \sqrt{H \pi(s) \gamma_\ps}\right)$ for the estimation of $\sqrt{\pi(s)}$, which arises from triangle inequality.
    The full proof is deferred to Appendix~\ref{app:initial-spectral}.
\end{proof}

\paragraph{Towards a Parameter-Free Algorithm.}
When $K$ is known, our algorithm does not require knowledge of the pseudo-spectral gap $\gamma_\ps$, and vice versa.
However, to achieve a truly parameter-free algorithm where both $K$ and $\gamma_\ps$ are unknown, one must first estimate $\gamma_\ps$. 
This is because estimating the number of clusters $K$ relies on spectral thresholding derived from Markovian concentration bounds, which inherently depend on $\gamma_\ps$ (line 6 of Algorithm~\ref{alg:initial-spectral}).
Since Stage II is already parameter-free as we will see later, obtaining a reliable estimator for $\gamma_\ps$ is the sole remaining hurdle.
Theoretically, we only require a high-probability lower bound on $\gamma_\ps$; specifically, an estimator $\hat{\gamma}_\ps$ satisfying  $(1/2)\gamma_\ps \le \hat{\gamma}_\ps \le (3/2)\gamma_\ps$ would allow our current analysis to hold up to constant factors.

Crucially, this estimation introduces additional sample complexity requirements on the trajectory length $H$.
A straightforward approach of taking the minimum over per-trajectory estimators, such as the one proposed by \citet[Theorem 1]{wolfer2024mixing}, can furnish the required lower bound.
However, this necessitates a longer trajectory length of $H \gtrsim \gamma_\ps^{-3}(\log T)^2$ (ignoring other factors).
This is strictly more demanding than our known-parameter regime by an extra $\gamma_\ps^{-2}$ factor and a $\log T$ term.
The extra $\log T$ dependence arises from the union bound over $T$ trajectories and can potentially be mitigated via logarithmic subsampling.
Conversely, the extra $\gamma_\ps^{-2}$ factor is likely unavoidable in the fully parameter-free setting, given the minimax lower bounds for estimating the pseudo-spectral gap itself~\citep[Theorem 4]{wolfer2019mixing}.
We defer further details and extended discussions to Appendix~\ref{app:gamma-ps}.

\subsection{Stage II. One-Shot Trajectory Likelihood Improvement}

\paragraph{Algorithm Design.}
By the Neyman-Pearson lemma~\citep{neyman-pearson} and from our information-theoretic lower bound (Theorem~\ref{thm:lower-bound}), we must perform trajectory-wise likelihood testing to achieve the optimal performance.
But, we do not know the transition kernels $p$ beforehand. Hence, we first utilize the ``good enough'' $\widehat{f}_0 : [T] \rightarrow [K]$ from Stage I to estimate the transition kernels $\widehat{p}_0$ (line 1).
Then, we perform likelihood-based reassignment of the cluster labels based on $\widehat{p}_0$ (line 2).
A keen reader may note that this is essentially the hard Expectation-Maximization (EM) algorithm~\citep{celeux1992hardEM,kearns1997hardEM}.

\begin{algorithm2e}[!t]
\caption{Likelihood Improvement}
\label{alg:likelihood}
\Input{$\widehat{f}_0 : [T] \rightarrow [K]$, $\{(s_{t,1}, \cdots, s_{t,H}) \}_{t \in [T]}$}
Estimate $p^{(k)}(\cdot | \cdot)$ for each $k \in [K]$ as follows:
denoting $\widehat{\gC}_0 := (\widehat{f}_0)^{-1}(k)$, for all $s, s' \in \gS$,
\begin{equation*}
    \widehat{p}_0^{(k)}(s' | s) \gets \frac{\sum_{t \in \widehat{\gC}_0} \sum_{h \in [H-1]} \indicator[s_{t,h} = s, s_{t,h+1} = s']}{\sum_{t \in \widehat{\gC}_0} \sum_{h \in [H-1]} \indicator[s_{t,h} = s]}.
\end{equation*}

Refine the cluster estimates via trajectory-wise maximum likelihood estimator:
for each $t \in [T]$,
\begin{equation*}
    \widehat{f}(t) \gets \argmax_{k \in [K]} \left\{ \gL(k; t) \triangleq \sum_{h=1}^{H-1} \log \widehat{p}_0^{(k)}(s_{t,h+1} | s_{t,h}) \right\}
\end{equation*}

\Return{$\widehat{f} : [T] \rightarrow [K]$}
\end{algorithm2e}

\paragraph{Theoretical Analysis.}
For the analysis of Stage II, we consider the following additional assumption:
\begin{assumption}[$\eta$-regularity \textit{across} the chains]
\label{assumption:eta-regularity}
    There exists $\eta_\pi, \eta_p > 1$ such that:
    \begin{equation*}
    \resizebox{\columnwidth}{!}{$
        \begin{aligned}
            \max_{s \in \gS} \max_{k, k'}   \frac{\pi^{(k)}(s)}{\pi^{(k')}(s)} \leq \eta_\pi, \ \max_{s, s' \in \gS} \max_{k, k'} \frac{p^{(k)}(s' | s)}{p^{(k')}(s' | s)} \leq \eta_p.
        \end{aligned}
    $}
    \end{equation*}
\end{assumption}

We now present the final error guarantee of Algorithm~\ref{alg:likelihood}; the proof is deferred to Appendix~\ref{app:likelihood}.
Recalling that $\gD_\pi$ is the stationary divergence (Eqn.~\eqref{eqn:D-pi}),
\begin{theorem}{Error Rate of Stage I+II}{likelihood}
    Suppose Assumptions~\ref{assumption:ergodic} and \ref{assumption:eta-regularity} hold, and that $T, H$ are sufficiently large in the following sense:
    
    \resizebox{1.05\columnwidth}{!}{$
        \begin{aligned}
            &T H \gtrsim \frac{(S \log T)^2}{\gamma_\ps \alpha_{\min} \gD_\pi^2}, \textit{ and } \\
            &H \gtrsim \frac{1}{\gamma_\ps} \left( S^2 \log T \vee \frac{\log T}{\pi_{\min}} \vee \frac{\eta_p \eta_\pi R S \log(TH)}{\alpha_{\min} \gD_\pi \Delta_\mW^2} \right).
        \end{aligned}
    $}
    Then, for some $C_\eta \asymp \frac{1}{\eta_p^2}$, we have that w.h.p.\footnote{with probability tending to $1$ as $T \rightarrow \infty$.},
    \begin{equation}
        E_T(\widehat{f}, f) \lesssim T \exp\left( -  C_\eta \gamma_\ps H \gD_\pi \right).
    \end{equation}
\end{theorem}
One can immediately see that the obtained upper bound on the clustering error rate nearly matches the lower bound (Theorem~\ref{thm:lower-bound}), but there are two gaps, which we elaborate on in the subsequent section.

\paragraph{Necessity of Assumption~\ref{assumption:eta-regularity}.}
We first clarify that Assumption~\ref{assumption:eta-regularity} does not require the chains to be uniform, i.e., it does not impose $\pi^{(k)}(s) \asymp 1/S$, as is commonly assumed in stochastic block models~\citep{yun2014adaptive} or block Markov chains~\citep{sanders2020bmc,jedra2023bmdp}. Instead, our assumption only requires that if one chain exhibits non-uniformity in certain states, the other chains must exhibit similar non-uniformity in those same states. 

Crucially, this assumption is only required for the theoretical analysis of Stage II; Stage I's guarantee (Theorem~\ref{thm:initial-spectral}) relies solely on Assumption~\ref{assumption:ergodic}. Theoretically, Assumption~\ref{assumption:eta-regularity} acts as a proof artifact necessary to bound the summands when applying the Markovian Bernstein concentration inequality~\citep{paulin2015markov} to the empirical likelihood ratios. Consequently, our current proof technique does not cover scenarios in which it fails, such as near-deterministic transitions in which a chain remains in a distinct state. 

Empirically, however, this strong assumption is rarely necessary. When Assumption~\ref{assumption:eta-regularity} fails, it often implies that certain chains are highly non-uniform and thus trivially easy to distinguish, a regime where we expect Stage II to remain robust.
To handle potential numerical instabilities in such extreme scenarios -- such as division by zero when candidate models assign near-zero probabilities to observed transitions -- one can easily regularize the likelihood evaluation via Laplace smoothing~\citep{laplace1812}. Provided Stage I yields a sufficiently accurate initial clustering, this smoothed likelihood ratio faithfully captures the high information content, heavily penalizing incorrect models while inducing negligible bias.
The precise theoretical analysis is left to future work.

\paragraph{What If Reversible?}
Conversely, one might ask whether additional structural assumptions, notably reversibility, a common property in Markov chain Monte Carlo (MCMC) methods~\citep{mcmc}, would simplify the analysis or yield sharper bounds. 

Assuming reversibility would indeed simplify certain aspects of the proof.
First, the classical spectral gap~\citep{lezaud1998chernoff} becomes well-defined, eliminating the need to rely on the pseudo-spectral gap.
Second, one could apply Bernstein's inequality specifically tailored for reversible Markov chains~\citep[Theorem 3.9]{paulin2015markov}, which yields slightly improved constants.

However, beyond these simplifications, we do not expect the reversibility assumption to yield fundamentally sharper bounds. The order of Markovian concentration remains the same regardless of reversibility; only the governing gap parameter (spectral vs. pseudo-spectral) changes.
Furthermore, \citet{wolfer2021markov} established minimax-optimal rates for estimating Markov chain kernels, demonstrating that simple visitation-based empirical transition estimation yields a tight upper bound for any ergodic Markov chain, whereas their lower-bound construction involves reversible chains~\citep[Section 6.4]{wolfer2021markov}.
Because our transition estimation in Stage II mirrors this approach (applied to grouped trajectories based on the initial clustering), restricting the scope to reversible chains is unlikely to improve the theoretical clustering error rate.

\section{DISCUSSIONS}
\label{sec:discussions}

Throughout the discussions, for simplicity, we assume $\alpha_{\min} \asymp 1/K$, i.e., roughly balanced cluster sizes.

\paragraph{(1) Gaps Between Upper and Lower Bounds.}
By comparing our lower bound (Theorem~\ref{thm:lower-bound}) and upper bound (Theorem~\ref{thm:likelihood}), we identify two key gaps. 

The first concerns the clustering error rates: the lower bound scales as $\exp(- 2 H \gD)$, while the upper bound scales as $\exp(- C_\eta \gamma_\ps H \gD_\pi)$, where $C_\eta$ is a constant depending on the regularity parameters $\eta_p$ and $\eta_\pi$. The upper bound thus contains an additional multiplicative factor of $\gamma_\ps$ in the exponential term.
The second gap pertains to the requirement on the horizon length $H$. Specifically, the lower bound holds for any $H \geq 2$ and $T \geq 1$, while the upper bound requires $H = \widetilde{\Omega}( \gamma_\ps^{-1} (S^2 \vee \pi_{\min}^{-1} \vee K S) )$ and $TH = \widetilde{\Omega}( \gamma_\ps^{-1} K S^2 )$, omitting logarithmic and other factors.

Specifically, the first gap arises from employing the Markovian Bernstein concentration inequality of \citet{paulin2015markov}, which relies on decomposing  the trajectories of the Markov chain into nearly independent blocks of length proportional to $\gamma_\ps H$~\citep{yu1994blocking,janson2004dependent,lezaud1998chernoff}. Recent results confirm that this dependence on $\gamma_\ps$ in the exponent is indeed unavoidable for (reversible) Markov chains: there exist chains for which the concentration rate matches this dependency~\citep{fan2018bernstein,fan2021hoeffding,rabinovich2020concentration}.
The second gap stems from the concentration of $\widehat{\mW}$ around $\mW$, which in turn requires accurate estimation of the transition kernels. \citet[Theorem 3.2]{wolfer2021markov} show the following requirement on $H$: if $H = o\left( \frac{1}{\varepsilon^2 \pi_{\min}} + \frac{S}{\gamma_\ps} \right)$, then no estimator can achieve $\varepsilon$-accuracy in $\bignorm{\cdot}_{2 \rightarrow \infty}$.\footnote{Their theorem is stated for $\bignorm{\cdot}_{1 \rightarrow \infty}$, but the same argument applies to $\bignorm{\cdot}_{2 \rightarrow \infty}$ with minimal modification.}

We believe these gaps indicate that the current lower bound is loose, as the lower bound analysis assumes complete knowledge of the $p$'s; we leave it to future work to establish a tighter lower bound.
Conceptually, this fits within a nonparametric testing viewpoint~\citep{tsybakov09}: how does uncertainty in the underlying measures degrade testing power and, in turn, the achievable clustering error?

\paragraph{\texorpdfstring{(2) Comparison with \citet{kausik2023mixture}.}{(2) Comparison with Kausik et al. (2023)}}
\label{sec:compare-kausik}
From an algorithmic perspective, a notable limitation of \citet{kausik2023mixture} is that their approach requires prior knowledge of the problem parameters $K$, $\alpha$, and $\Delta$, whereas our proposed algorithm is entirely parameter-free. 
Turning to the statistical guarantees, we recall their condition for exact cluster recovery:
\begin{assumption}[Assumption 2 of \citealt{kausik2023mixture}]
\label{assumption:kausik}
    $\exists \alpha, \Delta > 0 \ \text{s.t.} \ \forall k \neq k' \in [K] \ \exists s = s(k, k'):$
    \begin{equation*}
         \pi^{(k)}(s), \pi^{(k')}(s) \geq \alpha \ \hbox{and} \ \bignorm{p^{(k)}(\cdot | s) - p^{(k')}(\cdot | s)}_2 \geq \Delta.
    \end{equation*}
\end{assumption}
\begin{theoremplain}[Theorem 3 of \citealt{kausik2023mixture}]
\label{thm:kausik}
    Suppose that Assumptions \ref{assumption:ergodic} and \ref{assumption:kausik} hold.
    Set the occurrence and histogram clustering thresholds as $\beta = \frac{\alpha}{3}$ and $\tau = \frac{3\Delta^2}{4}$.
    Then, with $H \gtrsim K^{3/2} t_\mix \frac{\left( \log (T / (\alpha \delta)) \right)^4}{\alpha^3 \Delta^6}$ and $T \gtrsim K^2 S \frac{\log(1/\delta)}{\alpha^3 \Delta^8}$, their Algorithm 2 attains exact clustering with probability at least $1 - \delta$.
\end{theoremplain}
\citet{kausik2023mixture} treat $\alpha$ and $\Delta$ as constants, independent of the number of states $S$ and the trajectory length $H$. This assumption allows their condition on $H$ to avoid any polynomial dependence on $S$, and ensures that the requirement on $T$ scales only linearly with $S$. In effect, they implicitly assume that for every pair $k \neq k'$, there exists a highly visited state $s(k, k')$ with constant separation and from which clusters can be identified.
However, this assumption does not generally hold.
For example, in uniform-like Markov chains where $\pi(s) \asymp 1/S$, they require\footnote{\citet{kausik2023mixture} wrote their guarantees in terms of the mixing time $t_{\mix}$. Here, for a clear comparison, we write them in terms of $\gamma_\ps^{-1}$.
Indeed, both are of same order, up to $\log\pi_{\min}^{-1}$~\citep[Proposition 3.4]{paulin2015markov}.} $H \gtrsim \gamma_\ps^{-1} K^{3/2} S^3 (\log T)^4$ and $T \gtrsim K^2 S^4 \log T$ (with $\delta = 1/T$), which are impractically large.
In the same regime, our algorithm requires only $H \gtrsim \gamma_\ps^{-1} (S^2 \vee K S) \log T$ and $TH \gtrsim \gamma_\ps^{-1} K S^2$, significant improvements.
An interesting future direction is to ask whether an \textit{instance-wise near-optimal} algorithm can be designed -- achieving the near-optimal error rate guarantee when the cluster recovery condition requires only $H \gtrsim \gamma_\ps^{-1} \log T$ under Assumption~\ref{assumption:kausik} with $\alpha \Delta^2 = \Omega(1)$, yet requires only $H \gtrsim \gamma_\ps^{-1} (S^2 \vee \pi_{\min}^{-1} \vee K S) \log T$ (or better) in the worst case, ignoring other factors.

\begin{figure*}[!t]
    \centering
    \includegraphics[width=\linewidth]{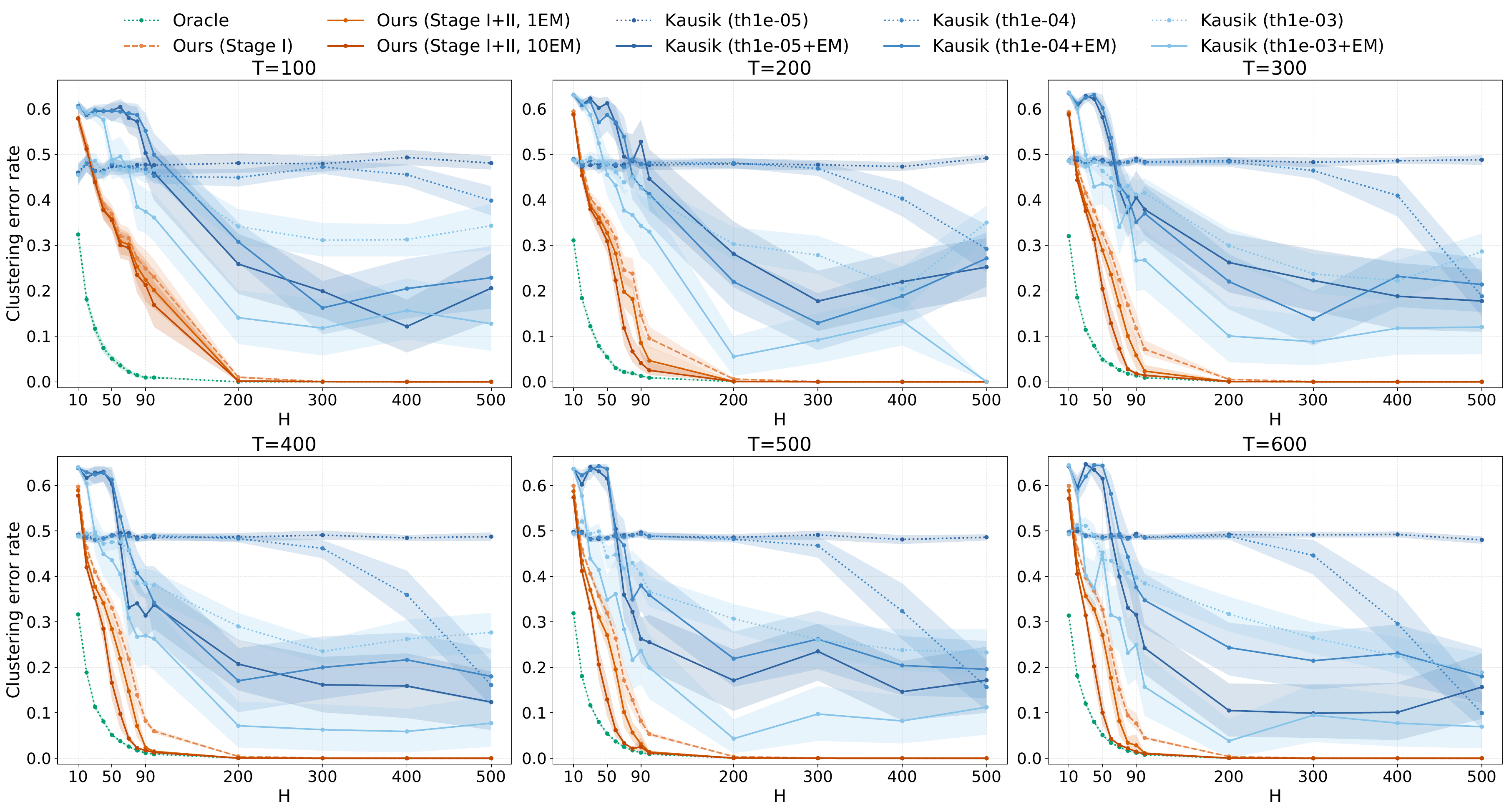}
    \caption{Clustering error rates of our and \citet{kausik2023mixture}'s algorithms on \textbf{Cyclic-Bump MMC}.
    }
    \label{fig:1}
\end{figure*}

\paragraph{(3) Different gaps: $\gD_\pi$, $\Delta_\mW$, and $\alpha \Delta^2$.}
Throughout this paper, we have introduced several notions of separability gaps: the KL gap $\gD_\pi$, the ``aggregated'' $\ell_2$ gap $\Delta_\mW$, and the $\Delta, \alpha$ from \citet{kausik2023mixture}.
The following proposition makes their relationships precise (see Appendix~\ref{app:gaps} for the proof):
\begin{proposition}
\label{prop:gaps}
    We have the following: with $p_{\max} := \max_{k \in [K]} \max_{s, s' \in \gS} p^{(k)}(s' | s)$, 
    \begin{enumerate}
        \item[(a)] $p_{\max} \gD_\pi \gtrsim \alpha \Delta^2$; \ (b) $p_{\max} \gD_\pi + \sqrt{\gD_\pi / \pi_{\min}} \gtrsim \Delta_\mW^2$;
        
        \item[(c)] under Assumption~\ref{assumption:eta-regularity}, $\Delta_\mW^2 \gtrsim \alpha \Delta^2 - \gO((\sqrt{\eta_\pi} - 1)^2)$.
    \end{enumerate}
\end{proposition}
The inequality {\it (b)} holds up to uniform ergodicity contraction constants (see Appendix \ref{app:gaps}).
Because of our particular choice of $\mW$ (Eqn.~\eqref{eqn:W}), the gaps $\Delta_\mW^2$, $\alpha \Delta^2$, and $\gD_\pi$ are comparable.
Indeed, for $\eta$-regular ergodic Markov chains with $\alpha\Delta^2 \gtrsim (\sqrt{\eta_\pi} - 1)^2$, we roughly have $\alpha \Delta^2 \lesssim \Delta_\mW^2 \lesssim p_{\max} \gD_\pi$.

Intuitively, \citet{kausik2023mixture}'s $\alpha \Delta^2$ accounts for separation from a single ``good'' state; $\Delta_\mW$ aggregates information across all states but in a geometrically suboptimal manner; our $\gD_\pi$ reflects the ``correctly'' aggregated information.
Moreover, for uniform-like Markov chains where $p_{\max} \asymp \frac{1}{S}$, Proposition~\ref{prop:gaps}(b) implies a factor-$S$ gap between $\gD_\pi$ and the other two gaps.

Although our optimal clustering error rates depend only on $\gD_\pi$, the success of $\ell_2$-based spectral clustering in Stage I requires $H \gtrsim \frac{1}{\gD_\pi \Delta_\mW^2}$.
For comparison, \citet{kausik2023mixture} requires $H \gtrsim \frac{1}{(\alpha \Delta^2)^3}$, reflecting a heavier dependence on more suboptimal separability gap.
To make the contrast between $\Delta_\mW^2$ and $\alpha\Delta^2$ explicit, we construct an instance of MMC -- which may be of independent interest -- where the three gaps exhibit a strict hierarchy in terms of $S$-factor (see Appendix~\ref{app:gaps2} for the proof):
\begin{proposition}
\label{prop:separation-Delta}
    There exists an instance of \textbf{clustering in MMC} such that $\gD_\pi \asymp S \Delta_\mW^2 \asymp S^2 \alpha \Delta^2$.
\end{proposition}

\paragraph{(4) Computational Scalability.}
While our algorithm achieves near-optimal statistical guarantees, it is not currently optimized for computational efficiency.
Specifically, the SVD step in Stage I (line 1 of Algorithm~\ref{alg:initial-spectral}) requires $\gO(T S^2)$ space and $\gO\big(T S^2 \min(T, S^2)\big)$ time, which poses a scalability bottleneck for large state spaces $S$ or trajectory counts $T$.
We leave the development of scalable variants to future work; see Appendix~\ref{app:future}\textbf{(1)} for some potential directions.

\section{EXPERIMENTS}
\label{sec:experiments}
\paragraph{Setting.}
We evaluate our algorithm on the \textbf{Cyclic-Bump MMC}, a synthetic instance from Proposition~\ref{prop:separation-Delta} (Appendix~\ref{app:cylic-bump}), with $S = 10$, $\zeta = 0.9 / S$, and $K = 3$.
We compare three variants of our method (known $K$): \texttt{\color{OursStage1}Stage I}, \texttt{\color{OursEM1}Stage I+II (EM1)}, and \texttt{\color{OursEM10}Stage I+II (EM10)}, where the latter tests the benefit of multiple EM iterations.
We benchmark against \citet{kausik2023mixture}, denoted \texttt{\color{KausikP10}Kausik (th1e-05)}, \texttt{\color{KausikP30}Kausik (th1e-04)}, and \texttt{\color{KausikP50}Kausik (th1e-03)}, where `th' refers to the threshold level, a hyperparameter that we sweep over $\{10^{-5}, 10^{-4}, 10^{-3}\}$.
`+EM' denotes running 10 EM iterations on top of the subspace-based clustering of \citet{kausik2023mixture}.
We also include the \texttt{\color{Oracle}Oracle} likelihood test, which is statistically optimal by the Neyman-Pearson lemma~\citep{neyman-pearson}.
Further details are relegated to Appendix~\ref{app:exp-details}.

\paragraph{Results.}
As shown in Figure~\ref{fig:1}, \texttt{\color{OursStage1}Stage I} alone already outperforms \citet{kausik2023mixture}.
We also note that their performance is highly sensitive to the tuning of the threshold hyperparameter.
Adding Stage~II refinement (\texttt{\color{OursEM1}EM1}, \texttt{\color{OursEM10}EM10}) further boosts accuracy---especially in the low-$H$, high-$T$ regime---approaching the \texttt{\color{Oracle}Oracle} curve as theoretically predicted.

\paragraph{Additional Experiments.}
In Appendix~\ref{app:experiments}, we provide four ablation studies: the impact of Stage II iterations, the impact of $S$ and $T$ on performance and runtime, unknown $K$, and varying pseudo-spectral gaps $\gamma_\ps$ (Appendices~\ref{app:ablation-1}--\ref{app:ablation-4}).
We also evaluate on the Last.fm 1K~\citep{lastfm1k} dataset in Appendix~\ref{app:experiments-real}, where we \emph{significantly outperform} \citet{kausik2023mixture}.

\section{CONCLUSION}
\label{sec:conclusion}
This paper studies trajectory clustering in MMCs.
We established an instance-specific lower bound on the achievable clustering error and developed a computationally tractable two-stage algorithm that attains near-optimal rates without requiring any prior knowledge of the underlying MMC model.
The key ingredient is Stage~I, which introduces a new injective Euclidean embedding for ergodic Markov chains, thereby enabling sharp concentration guarantees for spectral clustering.
We numerically validate the efficacy of our algorithm.
Further future directions are detailed in Appendix~\ref{app:future}.

\newpage

\subsection*{Acknowledgements}
We thank all the anonymous reviewers for their helpful comments that helped us improve the manuscript.
A. Prouti\`{e}re is supported by Vetenskapr{\aa}det, Digital Futures, and the Wallenberg AI, Autonomous Systems and Software program.

\bibliographystyle{plainnat}
\bibliography{references}

\clearpage
\section*{Checklist}

\begin{enumerate}

  \item For all models and algorithms presented, check if you include:
  \begin{enumerate}
    \item A clear description of the mathematical setting, assumptions, algorithm, and/or model. [Yes] Section 3,4
    \item An analysis of the properties and complexity (time, space, sample size) of any algorithm. [Yes] Section 3,4
    \item (Optional) Anonymized source code, with specification of all dependencies, including external libraries. [Yes] Appendix F.
  \end{enumerate}

  \item For any theoretical claim, check if you include:
  \begin{enumerate}
    \item Statements of the full set of assumptions of all theoretical results. [Yes] Section 3,4
    \item Complete proofs of all theoretical results. [Yes] The whole Appendix
    \item Clear explanations of any assumptions. [Yes] Section 2, 4
  \end{enumerate}

  \item For all figures and tables that present empirical results, check if you include:
  \begin{enumerate}
    \item The code, data, and instructions needed to reproduce the main experimental results (either in the supplemental material or as a URL). [Yes] Appendix F.
    \item All the training details (e.g., data splits, hyperparameters, how they were chosen). [Yes] Appendix F.
    \item A clear definition of the specific measure or statistics and error bars (e.g., with respect to the random seed after running experiments multiple times). [Yes] Appendix F.
    \item A description of the computing infrastructure used. (e.g., type of GPUs, internal cluster, or cloud provider). [No] Simple CPU experiments. 
  \end{enumerate}

  \item If you are using existing assets (e.g., code, data, models) or curating/releasing new assets, check if you include:
  \begin{enumerate}
    \item Citations of the creator If your work uses existing assets. [Yes] Appendix F
    \item The license information of the assets, if applicable. [Not Applicable]
    \item New assets either in the supplemental material or as a URL, if applicable. [Yes] Appendix F
    \item Information about consent from data providers/curators. [Not Applicable]
    \item Discussion of sensible content if applicable, e.g., personally identifiable information or offensive content. [Not Applicable]
  \end{enumerate}

  \item If you used crowdsourcing or conducted research with human subjects, check if you include:
  \begin{enumerate}
    \item The full text of instructions given to participants and screenshots. [Not Applicable]
    \item Descriptions of potential participant risks, with links to Institutional Review Board (IRB) approvals if applicable. [Not Applicable]
    \item The estimated hourly wage paid to participants and the total amount spent on participant compensation. [Not Applicable]
  \end{enumerate}

\end{enumerate}

\clearpage
\appendix
\thispagestyle{empty}

\onecolumn

\tableofcontents
\newpage

\section{\texorpdfstring{PROOF OF THEOREM~\ref{thm:lower-bound} -- FUNDAMENTAL LOWER BOUND}{PROOF OF THEOREM 3.1 -- FUNDAMENTAL LOWER BOUND}}
\label{app:lower-bound}
Let $(\varepsilon, \delta) \in [0, 1] \times (0, 1/2]$, $\bm\alpha \in \Delta([K])$, and $T \in \sN$ be fixed.
Let $c = \varepsilon T$ be the ``target'' number of misclassifications.
Let $f^{(0)} \triangleq f : [T] \rightarrow [K]$ be the ground-truth clustering satisfying $\bm\alpha^{(0)} := \bm\alpha(f^{(0)}) = \bm\alpha$, and let us denote $\gC_k^{(0)} := (f^{(0)})^{-1}(k) = \{ t \in [T] : f^{(0)}(t) = k \}$ be the set of trajectories in the $k$-th cluster.

We recall the following notion of distance between two clustering functions $f, g : [T] \rightarrow [K]$:
\begin{equation}
    E_T(f, g) := \min_{\sigma \in \Sym(K)} \sum_{t \in [T]} \indicator[f(t) \neq \sigma(g(t))].
\end{equation}
This is a proper metric on a suitable space of allocations, as shown in the following proposition (its proof deferred to the end of this section):
\begin{proposition}
\label{prop:distance}
    Let $\gF(T, K) := [K]^T = \{ f \ | \ f : [T] \rightarrow [K] \}$ and $\widetilde{\gF}(T, K) := \gF(T, K) / \sim$, where $f \sim g$ iff $\exists \sigma \in \Sym(K) \ \text{s.t.} \ f = \sigma \circ g$.
    Then,
    $E_T(\cdot, \cdot)$ is a proper metric on $\widetilde{\gF}(T, K)$.
\end{proposition}

Let us arbitrarily fix ${\color{blue} a} \neq {\color{blue} b} \in [K]$ that we will optimize later.
We construct $M_T$ alternate hypotheses $\{ f^{(m)} : [T] \rightarrow [K] \}_{m \in [M_n]}$ by re-allocating some trajectories from cluster ${\color{blue} b}$ to ${\color{blue} a}$.
We denote the probability measure and expectation w.r.t. each $m$-th allocation as $\sP_m$ and $\E_m$, respectively.

The construction will satisfy the following conditions for all $m \in [M_T]$: denoting $\gC_k^{(m)} := (f^{(m)})^{-1}(k)$ for $k \in [K]$,
\begin{enumerate}[
    label=(\textbf{C\arabic*}), %
    ref=(\textbf{C\arabic*})     %
]
    \item $\gC_{\color{blue} a}^{(0)} \subset \bigcap_{m \in [M_T]} \gC_{\color{blue} a}^{(m)}$, $|\gC_{\color{blue} a}^{(m)}| = \alpha_{\color{blue} a} n + 2c$, and $|\gC_{\color{blue} b}^{(m)}| = \alpha_{\color{blue} b} n - 2c$, \label{C1}
    \item $E_T(f^{(m)}, f^{(\ell)}) > 2c, \quad \forall m, \ell \in \{0\} \cup [M_T]$, \label{C2}
    \item $\gC_{k}^{(m)} = \gC_{k}^{(0)}, \quad \forall k \in [K] \setminus \{{\color{blue} a, b}\}$. \label{C3}
\end{enumerate}
The following lemma quantifies the minimum number of such hypotheses:
\begin{lemma}
\label{lem:M}
    $M_T \geq 2 \left( \frac{\alpha_{\color{blue} b} T}{16 e c} \right)^c$.
\end{lemma}
\begin{proof}
    The combinatorial argument here is due to \citet[Appendix B]{yun2019optimal}, which we reproduce here for completeness.
    
    First, note that there are total of $\binom{\alpha_{\color{blue} b} T}{2c}$ partitions satisfying \ref{C1} and \ref{C3} by moving $2c$ trajectories from $\gC_{\color{blue} b}^{(0)}$ to $\gC_{\color{blue} a}^{(m)}$.
    For each such partition, there are at most $\sum_{\ell=0}^c \binom{2c}{\ell} \binom{\alpha_{\color{blue} b} T - 2c}{\ell}$ partitions that still satisfy \ref{C1} and \ref{C3} but violate \ref{C2}, as such partitions must be created via swapping $\ell$ elements from $\gC_{\color{blue} b}^{(m)}$ and $\ell$ elements from $\gC_{\color{blue} a}^{(m)} \setminus \gC_{\color{blue} a}^{(0)}$ with $\ell \leq c$.
    Thus, we have that
    \begin{equation*}
        M_T \geq \frac{\binom{\alpha_{\color{blue} b} T}{2c}}{\sum_{\ell=0}^c \binom{2c}{\ell} \binom{\alpha_{\color{blue} b}T - 2c}{\ell}}
        \overset{(*)}{\geq} \frac{\left( \frac{\alpha_{\color{blue} b} T}{2c} \right)^{2c}}{2^{2c-1} \left( \frac{e \alpha_{\color{blue} b} T}{c} \right)^c}
        = 2 \left( \frac{\alpha_{\color{blue} b} T}{16 e c} \right)^c,
    \end{equation*}
    where at $(*)$, we use the elementary inequalities $\left( \frac{n}{m} \right)^m \leq \binom{n}{m} \leq \left( \frac{e n}{m} \right)^m$ for $n \geq m \geq 1$.
\end{proof}

Then, for a clustering algorithm $\gA$ that outputs a $\widehat{f}_\gA : [T] \rightarrow [K]$, consider the following hypothesis testing procedure:
\begin{equation}
    R_{\color{blue} a}(\gA) = \argmin_{m \in \{0\} \cup [M_T]} \min_{\sigma \in \Sym(K)} \left| \gC_{\color{blue} a}^{(m)} \triangle \widehat{f}_\gA^{-1}(\sigma({\color{blue} a})) \right|,
\end{equation}
where for two sets $A, B$, $A \triangle B := (A \setminus B) \cup (B \setminus A)$ is their symmetric difference.

Intuitively, $R_{\color{blue} a}$ outputs the most likely hypothesis out of $\{0\} \cup [M_T]$, given some estimated cluster for ${\color{blue} a}$.
Let $\gE^{(m)} := \{ R_{\color{blue} a}(\gA) = m \}$ be disjoint events across $m \in \{0\} \cup [M_T]$.
Note that for each hypothesis $m \in [M_T]$, the corresponding $\bm\alpha^{(m)} := \bm\alpha(f^{(m)})$ satisfies $\lVert \bm\alpha - \bm\alpha^{(m)} \rVert_2 = \frac{2\sqrt{2} c}{T} = 2\sqrt{2}\varepsilon \leq \beta$.
As $\gA$ is $(\varepsilon, \beta, \delta)$-locally stable and any Type I/II error w.r.t. each model $\Psi_n^{(m)}$ results in at least $s + 1$ misclassifications (see \ref{C2}), we must have that
\begin{equation}
\label{eqn:stable}
    \min\left\{ \sP_0\left( \gE^{(0)} \right), \ \min_{m \in [M_T]} \sP_m\left( \gE^{(m)} \right) \right\} \geq 1 - \delta.
\end{equation}
Furthermore, we note that
\begin{equation}
\label{eqn:inclusion}
    \sP_0 \left( \bigcup_{m \in [M_T]} \gE^{(m)} \right) \leq \delta,
\end{equation}
as if $\gE^{(m)}$ is true for some $m \in [M_T]$, then
\begin{equation*}
    2c < E_T(f, f^{(m)}) \leq E_T(\widehat{f}_\gA, f) + E_T(\widehat{f}_\gA, f^{(m)})
    \leq E_T(\widehat{f}_\gA, f) + c \Longrightarrow E_T(\widehat{f}_\gA, f) > c,
\end{equation*}
which holds with probability at most $\delta$ due to $\gA$ being $(\varepsilon, \beta, \delta)$-stable.
Note that due to the disjointness of $\gE^{(m)}$'s, we additionally have that
\begin{equation}
\label{eqn:min-prob}
    \frac{1}{M_T} \sum_{m \in [M_T]} \sP_0 \left( \gE^{(m)} \right) \leq \delta
    \Longrightarrow \min_{m \in [M_T]} \sP_0 \left( \gE^{(m)} \right) \leq \frac{\delta}{M_T}.
\end{equation}

We now define the log-likelihood ratio of the given trajectories $\left\{ \gT_t = (s_{t,1}, \cdots, s_{t,H}) \right\}_{t \in [T]}$ under the true model ($m = 0$) and the $m$-th alternate model:
\begin{equation}
    \gQ_T^{(m)} := \log\frac{\sP_m(\{ \gT_t \}_{t \in [T]})}{\sP_0(\{ \gT_t \}_{t \in [T]})}
    = \sum_{t \in \gC_{\color{blue}a}^{(m)} \setminus \gC_{\color{blue}a}^{(0)}} \log\frac{\sP^{({\color{blue}a})}(s_{t,1}, \cdots, s_{t,H})}{\sP^{({\color{blue}b})}(s_{t,1}, \cdots, s_{t,H})}.
\end{equation}
Taking the expectation w.r.t. the $m$-th alternate model, we have
\begin{align*}
    \E_m\left[ \gQ_T^{(m)} \right] &= \sum_{t \in \gC_{\color{blue}a}^{(m)} \setminus \gC_{\color{blue}a}^{(0)}} \E_m\left[ \log\frac{\sP^{({\color{blue}a})}(s_{t,1}, \cdots, s_{t,H})}{\sP^{({\color{blue}b})}(s_{t,1}, \cdots, s_{t,H})} \right] \\
    &= \sum_{t \in \gC_{\color{blue}a}^{(m)} \setminus \gC_{\color{blue}a}^{(0)}} \E_m\left[ \log\frac{\mu^{({\color{blue}a})}(s_{t,1})}{\mu^{({\color{blue}b})}(s_{t,1})} + \sum_{h=1}^{H-1} \log\frac{p^{({\color{blue}a})}(s_{t,h+1} | s_{t,h})}{p^{({\color{blue}b})}(s_{t,h+1} | s_{t,h})} \right] \\
    &= 2c \left[ \KL\left( \mu^{({\color{blue}a})}, \mu^{({\color{blue}b})} \right) + \sum_{h=1}^{H-1} \sum_{s \in \gS} \sP^{({\color{blue}a})}(s_{t,h} = s) \KL\left( p^{({\color{blue}a})}(\cdot | s), p^{({\color{blue}b})}(\cdot | s) \right) \right] \tag{Law of total expectation w.r.t. $s_{t,h}$} \\
    &= 2 (H - 1) c \underbrace{\left[ \frac{1}{H - 1} \KL\left( \mu^{({\color{blue}a})}, \mu^{({\color{blue}b})} \right) + \sum_{s \in \gS} \sP_H^{({\color{blue}a})}(s) \KL\left( p^{({\color{blue}a})}(\cdot | s), p^{({\color{blue}b})}(\cdot | s) \right) \right]}_{=: \gD^{({\color{blue}a}, {\color{blue}b})}}. \tag{$\sP_H^{({\color{blue}a})}(s) := \frac{1}{H - 1} \sum_{h=1}^{H-1} \sP^{({\color{blue}a})}(s_{t,h} = s)$}
\end{align*}

There are two paths from here, one inspired by the arguments of \cite{jedra2023identification} that utilizes data processing inequality, and another inspired by \cite{yun2019optimal} that utilizes Markov's inequality.
The former leads to the non-asymptotic guarantee, while the latter leads to the asymptotic guarantee.

\vspace{0.2cm}
\textbf{Non-asymptotic guarantee.}
We recall the KL-version of the data processing inequality\footnote{This has been referred to as the ``Fundamental Inequality'' in \cite{garivier2019lower}, as using this allows for a strikingly simple proof of distribution-dependent bandit regret lower bounds.}:
\begin{lemma}[Data Processing Inequality; Lemma 1 of \cite{garivier2019lower}]
\label{lem:data-processing}
    Consider a measurable space $(\Gamma, \gG)$ equipped with two probability measure $\sP_1$ and $\sP_2$.
    Then, we have that
    \begin{equation}
        \KL(\sP_1, \sP_2) \geq \sup_Z \kl(\E_1[Z], \E_2[Z]),
    \end{equation}
    where $\sup_Z$ is over all possible $\gG$-measurable random variable $Z : \Omega \rightarrow [0, 1]$ and $\kl$ is the Bernoulli KL divergence, i.e., $\kl(p, q) := p \log\frac{p}{q} + (1 - p) \log\frac{1-p}{1-q}$ for $p, q \in (0, 1)$.
\end{lemma}
In the above lemma, let $(\Gamma, \gG)$ be our observation space of $T$ trajectories, $\sP_1 = \sP_m$, $\sP_2 = \sP_0$, and $Z = \indicator\lbrace \gE^{(m)}\rbrace$.

Then, we have that
\begin{align*}
    2 (H - 1) c \gD^{({\color{blue}a}, {\color{blue}b})} &= \frac{1}{M_T} \sum_{m \in [M_T]} \KL\left( \sP_m, \sP_0 \right) \\
    &\geq \frac{1}{M_T} \sum_{m \in [M_T]} \kl\left( \sP_m(\gE^{(m)}), \sP_0(\gE^{(m)}) \right) \tag{Lemma~\ref{lem:data-processing}} \\
    &\geq \kl\left( \frac{1}{M_T} \sum_{m \in [M_T]} \sP_m(\gE^{(m)}), \frac{1}{M_T} \sum_{m \in [M_T]} \sP_0\left( \gE^{(m)} \right) \right) \tag{joint convexity of $\kl$} \\
    &= \kl\left( \frac{1}{M_T} \sum_{m \in [M_T]} \sP_m(\gE^{(m)}), \frac{1}{M_T} \sP_0\left( \bigcup_{m \in [M_T]} \gE^{(m)} \right) \right) \tag{$\gE^{(m)}$'s are disjoint} \\
    &\overset{(*)}{\ge} \kl\left( 1 - \delta, \frac{\delta}{M_T} \right) \\ 
    &\overset{(**)}{\ge} (1 - \delta) \log \frac{M_T}{\delta} - \log 2 \\
    &\ge \frac{1}{2} \log \frac{1}{\delta} + \frac{c}{2} \log\frac{\alpha_{\color{blue} b} T}{16ec} - \frac{1}{2} \log 2 \tag{Lemma~\ref{lem:M}, $\delta \in (0, 1/2]$} \\
    &\ge \frac{1}{2} \log \frac{1}{2\delta} + \frac{c}{2} \log\frac{\alpha_{\min} T}{16ec}.
\end{align*}

$(*)$ follows from $\sP_m(\gE^{(m)}) \geq 1 - \delta$ (as the algorithm is $(\varepsilon, \beta, \delta)$-stable), Eqn.~\eqref{eqn:inclusion}, and the following well-known properties of $\kl$:
\begin{itemize}
    \item $q \mapsto \kl(p, q)$ is increasing in $[p, 1]$, as
    \begin{equation*}
        \partial_q \kl(p, q) = -\frac{p}{q} + \frac{1 - p}{1 - q}
        = \frac{q - p}{q(1 - q)} \geq 0,
    \end{equation*}
    \item $p \mapsto \kl(p, q)$ is increasing on $p \in [1/2, 1]$ when $q < 1/2$, as 
    \begin{equation*}
        \partial_p \kl(p, q) 
        = \log\frac{p}{q} - \log\frac{1 - p}{1 - q}
        = \log \frac{p}{1 - p} + \frac{1 - q}{q}
        \geq 0.
    \end{equation*}
\end{itemize}
$(**)$ follows from $\kl(p, q) \geq p \log\frac{1}{q} - \log 2$ (Eqn. (11) of \cite{garivier2019lower}).

We conclude by rearranging the resulting inequality in terms of $\delta$ and optimizing over ${\color{blue}a} \neq {\color{blue}b} \in [K]$.

\qed (for non-asymptotic)

\vspace{0.1cm}
\textbf{Asymptotic guarantee.}
Let $m^* = \argmin_{m \in [M_T]} \sP_0[\gE^{(m)}]$.
Then,
\begin{align*}
    \sP_{m^*}[\gQ_T^{(m^*)} \leq \log M_T] &= \sP_{m^*}\left[ \gQ_T^{(m^*)} \leq \log M_T, \gE^{(m^*)} \right] + \sP_{m^*}\left[\gQ_T^{(m^*)} \leq \log M_T, (\gE^{(m^*)})^\complement \right] \\
    &\leq \exp(\log M_T) \sP_0 \left[\gQ_T^{(m^*)} \leq \log M_T, \gE^{(m^*)} \right] + \sP_{m^*} \left[(\gE^{(m^*)})^\complement \right] \tag{change of measure from the alternate model $m^*$ to the true model} \\
    &\leq M_T \sP_0[\gE^{(m^*)}] + \delta_T \\
    &\leq 2\delta_T \tag{Eqn.~\eqref{eqn:min-prob}},
\end{align*}
where $(\delta_T)_{T \in \sN}$ is such that $\delta_T \rightarrow 0$ as $T \rightarrow \infty$.

Thus, by Markov's inequality,
\begin{align*}
    1 - 2\delta_T &\leq \sP_{m^*}\left[\gQ_T^{(m^*)} > \log M_T\right] \\
    &\leq \frac{2 (H - 1) c \gD^{({\color{blue} a}, {\color{blue} b})}}{\log M_T} \\
    &\leq \frac{2 (H - 1) c \gD^{({\color{blue} a}, {\color{blue} b})}}{c \log\frac{\alpha_{\color{blue}b} T}{16 ec} + \log 2}. \tag{Lemma~\ref{lem:M}} \\
    &\leq \frac{2 (H - 1) c \gD^{({\color{blue} a}, {\color{blue} b})}}{c \log\frac{\alpha_{\min} T}{16 ec} + \log 2}.
\end{align*}

Taking the min over ${\color{blue}a} \neq {\color{blue}b} \in [K]$ and the limit $n \rightarrow \infty$, we have that
\begin{equation}
    \liminf_{T \rightarrow \infty} \frac{2 (H - 1) \gD}{\log \frac{\alpha_{\min}}{16e \varepsilon}} \geq 1.
\end{equation}

\qed(for asymptotic)

\begin{proof}[Proof of Proposition~\ref{prop:distance}]
    Positivity and symmetricity are trivial. Thus, it remains to show the triangle inequality.
    Let $f, g, h \in \gF(T, K)$.
    Then,
    \begin{align*}
        E_T(f, h) &= \min_{\sigma \in \Sym(K)} \sum_{t \in [T]} \indicator[f(t) \neq \sigma(h(t))] \\
        &\leq \min_{\nu \in \Sym(K)} \min_{\sigma \in \Sym(K)} \sum_{t \in [T]} \indicator[f(t) \neq \nu(g(t)) \vee \nu(g(t)) \neq \sigma(h(t))] \\
        &\leq \min_{\nu \in \Sym(K)} \min_{\sigma \in \Sym(K)} \sum_{t \in [T]} \left( \indicator[f(t) \neq \nu(g(t))] + \indicator[\nu(g(t)) \neq \sigma(h(t))] \right) \\
        &\leq  \min_{\nu \in \Sym(K)} \sum_{t \in [T]}  \indicator[f(t) \neq \nu(g(t))] + \min_{\nu \in \Sym(K)} \min_{\sigma \in \Sym(K)} \sum_{t \in [T]} \indicator[\nu(g(t)) \neq \sigma(h(t))] \\
        &= E_T(f, g) + \min_{\nu' \in \Sym(K)} \indicator[g(t) \neq \nu'(h(t))] \tag{$\nu' = \nu^{-1} \circ \sigma$} \\
        &= E_T(f, g) + E_T(g, h).
    \end{align*}
\end{proof}

\newpage

\section{\texorpdfstring{PROOF OF THEOREM~\ref{thm:initial-spectral} -- INITIAL SPECTRAL CLUSTERING}{PROOF OF THEOREM 4.1 -- INITIAL SPECTRAL CLUSTERING}}
\label{app:initial-spectral}
We remark in advance that the proofs of all supporting lemmas and propositions are deferred to the end of this section.

We begin by establishing the following concentration result for $\widehat{\mW}$:
\begin{lemma}[Concentration of $\widehat{\mW}$ to $\mW$]
\label{lem:concentration-W}
    Let $\delta \in (0, 1)$.
    For each $s \in \gS$ and $t \in [T]$, the following holds with probability at least $1 - \delta$, \underline{given} that $H \geq \left( \frac{25}{1 - \pi(s)} \vee 52 \right) \frac{2}{\pi(s) \gamma_\ps} \log\frac{8}{\pi_{\min} \delta^2}$:
    \begin{equation}
        \bignorm{\sqrt{\pi^{(f(t))}(s)} p^{(f(t))}(\cdot | s) - \frac{1}{\sqrt{H \widehat{N}_t(s)}} \widehat{N}_t(s, \cdot)}_2 \leq
        4 \sqrt{\frac{2}{H \gamma_\ps} \log\frac{8 (H \pi(s) + 1)}{\pi_{\min} \delta^2}}.
    \end{equation}
    If we \underline{additionally} assume that $H \geq \frac{9 \pi(s)}{\pi_{\min}}$, the following holds with probability at least $1 - \delta$:
    \begin{equation}
        \bigabs{(\mW)_{t,:} - (\widehat{\mW})_{t,:}} \leq \widehat{r} \triangleq 8 \sqrt{\frac{S}{H \gamma_\ps} \log\frac{H}{\delta}}
        \Longrightarrow
        \bignorm{\mW - \widehat{\mW}}_{2 \rightarrow \infty} \leq 8 \sqrt{\frac{S}{H \gamma_\ps} \log\frac{T H}{\delta}}.
    \end{equation}
\end{lemma}

First, as $\bignorm{\mW - \widehat{\mW}}_2 \leq \sqrt{T} \bignorm{\mW - \widehat{\mW}}_{2 \rightarrow \infty}$~\citep[Proposition 6.3]{cape2019twoinfty}, with our choice of $\widehat{\sigma}_\thres := 8 \sqrt{\frac{T S}{H \gamma_\ps} \log\frac{T H}{\delta}}$ (line 3), we have from Lemma~\ref{lem:concentration-W} and Weyl's inequality for singular values~\citep[Problem 7.3.P16]{hornjohnson} that $\sP\left( \widehat{R} \leq R = \rank(\mW) \right) \geq 1 - \delta$.
We will condition on this event throughout the remainder of the proof.

Then, we have that
\begin{align*}
    \bignorm{\widehat{\mX} \widehat{\mV}_{1:R}^\top - \mW}_F^2 &\leq 2 R \bignorm{\widehat{\mX} \widehat{\mV}_{1:R}^\top - \mW}_2^2 \\
    &\leq 4 R \left( \bignorm{\widehat{\mX} \widehat{\mV}_{1:R}^\top - \widehat{\mW}}_2^2 + \bignorm{\widehat{\mW} - \mW}_2^2 \right) \\
    &\leq 8 R (\widehat{\sigma}_{\thres})^2 \tag{with probability at least $1 - \delta$}, \\
    &= 2^9 R \frac{TS}{H \gamma_\ps} \log\frac{TH}{\delta}.
\end{align*}

Let us denote $\widehat{\mY} := \widehat{\mX} \widehat{\mV}_{1:R}$.
We now have the following lemma:
\begin{lemma}
\label{lem:initial-misclassification}
    Additionally assume that $H \gtrsim \frac{S}{\Delta_\mW^2 \gamma_\ps} \log\frac{H}{\delta} \log\frac{TH}{\delta}$.
    Then, if $t \in [T]$ is misclassified in Algorithm~\ref{alg:initial-spectral}, then $\bignorm{\widehat{\mY}_{t,:} - \mW_{t,:}}_2 > \frac{\Delta_\mW}{4}$.
\end{lemma}
Then, we have that with probability at least $1 - \delta$,
\begin{equation*}
    E_T(\widehat{f}_0, f) \leq \frac{\bignorm{\widehat{\mY} - \mW}_F^2}{\frac{\Delta_\mW^2}{2^4}}
    \leq \frac{2^{13} T R S}{H \gamma_\ps \Delta_\mW^2} \log\frac{TH}{\delta}.
\end{equation*}
\qed

\subsection{Proof of Lemma~\ref{lem:concentration-W}}
Note that
\begin{align*}
	&\bignorm{\sqrt{\pi^{(f(t))}(s)} p^{(f(t))}(\cdot | s) - \frac{1}{\sqrt{H \widehat{N}_t(s)}} \widehat{N}_t(s, \cdot)}_2 \\
    &\leq \sqrt{\pi^{(f(t))}(s)} \bignorm{ p^{(f(t))}(\cdot | s) - \frac{\widehat{N}_t(s, \cdot)}{\widehat{N}_t(s)} }_2 + \bigabs{\sqrt{\pi^{(f(t))}(s)} - \sqrt{\frac{\widehat{N}_t(s)}{H}}} \bignorm{\frac{\widehat{N}_t(s, \cdot)}{\widehat{N}_t(s)}}_2 \\
	&\leq \sqrt{\pi^{(f(t))}(s)} \underbrace{\bignorm{ p^{(f(t))}(\cdot | s) - \frac{\widehat{N}_t(s, \cdot)}{\widehat{N}_t(s)} }_2}_{(i)} + \underbrace{\bigabs{\sqrt{\pi^{(f(t))}(s)} - \sqrt{\frac{\widehat{N}_t(s)}{H}}}}_{(ii)} \underbrace{\bignorm{\frac{\widehat{N}_t(s, \cdot)}{\widehat{N}_t(s)}}_1}_{= 1} 
\end{align*}
Let us bound $(i)$ and $(ii)$ separately.
For simplicity, we omit the dependency on $t \in [T]$.

\underline{\textbf{\textit{Bounding $(i)$.}}}
Here, we follow the proof strategy of \citet[Theorem 3.1]{wolfer2021markov}, with two crucial modifications: the quantity being bounded is the $\ell_2$-norm (not $\ell_1$-norm) error, and we only need to bound the error stemming from a single state $s$.

Let $n_s \in \sN$ to be determined later, and let $\sP_\pi$ be the probability measure induced when the Markov chain starts from its stationary distribution.
Then,
\begin{align*}
	&\sP_\pi\left( \bignorm{ p(\cdot | s) - \frac{\widehat{N}(s, \cdot)}{\widehat{N}(s)} }_2 > \varepsilon \right) 
	\leq \sum_{n = n_s}^{3 n_s} \sP_\pi\left( \bignorm{ p(\cdot | s) - \frac{\widehat{N}(s, \cdot)}{\widehat{N}(s)} }_2 > \varepsilon, \ \widehat{N}(s) = n_s \right) + \sP_\pi\left( \widehat{N}(s) \not\in \llbracket n_s, 3 n_s \rrbracket \right).
\end{align*}
For the first term, we utilize the scheme of reducing learning the Markovian transition kernel to learning a discrete distribution, as described in the proof of \citet[Theorem 3.1]{billingsley1961markov}.
Then, denoting $\widehat{p}_n(\cdot | s)$ to be the estimate of $p(\cdot | s)$ from $n$ i.i.d. samples of $p(\cdot | s)$, we have that
\begin{align*}
	\sum_{n = n_s}^{3 n_s} \sP_\pi\left( \bignorm{ p(\cdot | s) - \frac{\widehat{N}(s, \cdot)}{\widehat{N}(s)} }_2 > \varepsilon, \ \widehat{N}(s) = n \right) &\leq \sum_{n = n_s}^{3 n_s} \sP_\pi\left( \bignorm{ p(\cdot | s) - \widehat{p}_n(\cdot | s) }_2 > \varepsilon \right) \\
	&\overset{(*)}{\leq} \sum_{n = n_s}^{3 n_s} \exp\left( - \frac{\varepsilon^2 n}{2} \right) \\
	&\leq (2n_s + 1) \exp\left( - \frac{\varepsilon^2 n_s}{2} \right),
\end{align*}
where $(*)$ follows from the following $\ell_2$-concentration of learning discrete distribution:
\begin{lemma}[Theorem 9 of \cite{canonne2020note}]
	\label{lem:canonne}
	Let $\gS$ be a finite state space, $p \in \gP(\gS)$, and $\delta \in (0, 1)$.
	We are given $\{s_i\}_{i \in [N]}$ with $s_i \overset{i.i.d.}{\sim} p$.
	Let $\hat{p}_N \in \gP(\gS)$ be defined as $\hat{p}_N(s) := \frac{1}{N} \sum_{i \in [N]} \indicator[s_i = s]$ for each $s \in \gS$.
	Then, we have the following:
	\begin{equation}
		\sP\left( \bignorm{p - \hat{p}_N}_2 \geq \varepsilon \right) \leq \exp\left( - \frac{\varepsilon^2 N}{2} \right),
	\end{equation}
\end{lemma}

We now choose $n_s = \frac{H \pi(s)}{2}$ and utilize the Markovian Bernstein concentration, which we recall here:
\begin{lemma}[Theorem 3.4 of \citet{paulin2015markov}]
\label{lem:paulin}
    Suppose of $(X_h)_{h \geq 1}$ is an ergodic Markov chain over $\gS$ with transition probability $p$, initial distribution $\mu$, and pseudo spectral gap $\gamma_\ps$.
    Let $\phi \in L^2(\pi)$ with $\sup_{s \in \gS} |\phi(s) - \E_\pi[\phi]| \leq C$ for some $C > 0$, and $V_\phi := \Var_\pi[\phi]$.
    Then, we have the following:
    for any $u \geq 0$,
    \begin{equation}
        \sP\left( \sum_{h=1}^H (\phi(X_h) - \E_\pi[\phi(X_h)]) \geq u \right) \leq \exp\left( - \frac{u^2 \gamma_\ps}{8 \left( H + \frac{1}{\gamma_\ps} \right) V_\phi + 20 C u} \right).
    \end{equation}
\end{lemma}
For two integers $a \geq b$, denote $\llbracket a, b \rrbracket := \{a, a+1, \cdots, b\}$.
We then have that
\begin{align*}
	\sP_\pi\left( \widehat{N}(s) \not\in \left\llbracket \frac{H \pi(s)}{2}, \frac{3 H \pi(s)}{2} \right\rrbracket \right) &= \sP_\pi\left( \widehat{N}(s) - H \pi(s) \not\in \left\llbracket - \frac{H \pi(s)}{2}, \frac{H \pi(s)}{2} \right\rrbracket \right) \\
	&\leq \sP_\pi\left( \left| \widehat{N}(s) - H \pi(s) \right| > \frac{H \pi(s)}{2} \right) \\
	&\leq 2 \exp\left( - \frac{(H \pi(s) / 2)^2 \gamma_\ps}{8 \left( H + \frac{1}{\gamma_\ps} \right) \pi(s) (1 - \pi(s)) + 20 (H \pi(s) / 2)} \right) \\
	&\leq 2 \exp\left( - \frac{H \pi(s) \gamma_\ps}{104} \right). \tag{Assume that $H > \frac{1}{\gamma_\ps}$}
\end{align*}
Combining everything, we have that
\begin{equation*}
	\sP_\pi\left( \bignorm{ p(\cdot | s) - \frac{\widehat{N}(s, \cdot)}{\widehat{N}(s)} }_2 > \varepsilon \right)
	\leq (H\pi(s) + 1) \exp\left( - \frac{\varepsilon^2 H\pi(s)}{2} \right) + 2 \exp\left( - \frac{H \pi(s) \gamma_\ps}{104} \right),
\end{equation*}
and thus, with $H \geq \frac{1}{\gamma_\ps} \vee \frac{104}{\pi(s) \gamma_\ps} \log\frac{2}{\delta}$, we have that
\begin{equation}
\label{eqn:p-concentration}
	\sP_\pi\left( \bignorm{ p(\cdot | s) - \frac{\widehat{N}(s, \cdot)}{\widehat{N}(s)} }_2 > \sqrt{\frac{2}{H \pi(s)} \log\frac{2(H\pi(s) + 1)}{\delta}} \right)
	\leq \delta.
\end{equation}

Here, we recall the following result that quantifies the price of non-stationarity due to the initial distribution $\mu$ not necessarily being the stationary distribution $\pi$:
\begin{lemma}[Proposition 3.10 of \citet{paulin2015markov}]
	\label{lem:paulin-3.10}
	Let $(X_1, \cdots, X_H)$ be a (time-homogeneous) Markov chain with stationary distribution $\pi$.
	For any measurable $g : \gS^H \rightarrow \sR$ and initial distribution $\mu \in \Delta(\gS)$, we have that for any $u \geq 0$,
	\begin{equation}
		\sP_\mu\left( g(X_1, \cdots, X_H) \geq u\right) \leq \sqrt{ \left[ \sum_{s \in \gS} \frac{\mu(s)^2}{\pi(s)} \right] \sP_\pi\left( g(X_1, \cdots, X_H) \geq u\right)}.
	\end{equation}
	Especially as $|\gS| < \infty$, we can bound $\sum_{s \in \gS} \frac{\mu(s)^2}{\pi(s)} \leq \frac{1}{\pi_{\min}}$.
\end{lemma}
Thus,
\begin{equation*}
	\sP\left( \bignorm{ p(\cdot | s) - \frac{\widehat{N}(s, \cdot)}{\widehat{N}(s)} }_2 > \sqrt{\frac{2}{H \pi(s)} \log\frac{2(H\pi(s) + 1)}{\delta}} \right)
	\leq \sqrt{\frac{\delta}{\pi_{\min}}}.
\end{equation*}
Reparametrizing finally gives
\begin{equation}
	\sP\left( \bignorm{ p(\cdot | s) - \frac{\widehat{N}(s, \cdot)}{\widehat{N}(s)} }_2 > \sqrt{\frac{2}{H \pi(s)} \log\frac{8(H\pi(s) + 1)}{\pi_{\min} \delta^2}} \right)
	\leq \frac{\delta}{2},
\end{equation}
\textit{given} that $H \geq \frac{1}{\gamma_\ps} \vee \frac{104}{\pi(s) \gamma_\ps} \log\frac{8}{\pi_{\min} \delta^2}$.

\underline{\textbf{\textit{Bounding $(ii)$.}}}
Note that
\begin{equation*}
	\bigabs{\sqrt{\pi(s)} - \sqrt{\frac{\widehat{N}(s)}{H}}} = \frac{\bigabs{\pi(s) - \frac{\widehat{N}(s)}{H}}}{\sqrt{\pi(s)} + \sqrt{\frac{\widehat{N}(s)}{H}}}
	\leq \frac{1}{\sqrt{\pi(s)}} \bigabs{\pi(s) - \frac{\widehat{N}(s)}{H}}.
\end{equation*}
Again, we invoke Lemma~\ref{lem:paulin}:
\begin{align*}
	&\sP_\pi\left( \bigabs{\pi(s) - \frac{\widehat{N}(s)}{H}} > \sqrt{\frac{32 \pi(s) (1 - \pi(s))}{H \gamma_\ps} \log\frac{8}{\pi_{\min} \delta^2}} \right) \\
	&\leq 2 \exp\left( - \frac{32 \pi(s) (1 - \pi(s)) H \log\frac{8}{\pi_{\min}\delta^2}}{8 \left( H + \frac{1}{\gamma_\ps} \right) \pi(s) (1 - \pi(s)) + 20 \sqrt{\frac{32 \pi(s) (1 - \pi(s)) H}{\gamma_\ps} \log\frac{8}{\pi_{\min} \gamma_\ps} }} \right) \\
	&\leq 2 \exp\left( - \frac{2 \sqrt{H \pi(s) (1 - \pi(s))} \log\frac{8}{\pi_{\min}\delta^2}}{\sqrt{H \pi(s) (1 - \pi(s))} + 5 \sqrt{\frac{2}{\gamma_\ps} \log\frac{8}{\pi_{\min} \delta^2}}} \right) \tag{$H \geq \frac{1}{\gamma_\ps}$} \\
	&\leq 2 \exp\left( - \frac{2 \sqrt{H \pi(s) (1 - \pi(s))} \log\frac{8}{\pi_{\min}\delta^2}}{2 \sqrt{H \pi(s) (1 - \pi(s))}} \right) \tag{$H \geq \frac{50}{\pi(s) (1 - \pi(s)) \gamma_\ps} \log\frac{8}{\pi_{\min} \delta^2}$} \\
	&= \frac{\pi_{\min} \delta^2}{4}.
\end{align*}
Using Lemma~\ref{lem:paulin-3.10}, we then have
\begin{equation}
\label{eqn:pi-concentration}
	\sP\left( \bigabs{\pi(s) - \frac{\widehat{N}(s)}{H}} > \sqrt{\frac{32 \pi(s) (1 - \pi(s))}{H \gamma_\ps} \log\frac{8}{\pi_{\min} \delta^2}} \right) \leq \frac{\delta}{2}.
\end{equation}

\underline{\textbf{\textit{Combining everything.}}}
By union bound, we have that with probability at least $1 - \delta$,
\begin{align}
	\bignorm{\sqrt{\pi^{(f(t))}(s)} p^{(f(t))}(\cdot | s) - \frac{1}{\sqrt{H \widehat{N}_t(s)}} \widehat{N}_t(s, \cdot)}_2 \nonumber &\leq \sqrt{\frac{2}{H} \log\frac{8(H\pi(s) + 1)}{\pi_{\min} \delta^2}} + \sqrt{\frac{32 (1 - \pi(s))}{H \gamma_\ps} \log\frac{8}{\pi_{\min} \delta^2}} \nonumber \\
	&\leq 4 \sqrt{\frac{2}{H \gamma_\ps} \log\frac{8 (H \pi(s) + 1)}{\pi_{\min} \delta^2}}.
\end{align}
\qed

\subsection{Proof of Lemma~\ref{lem:initial-misclassification}}
Here, we are largely inspired by the proof strategies of \citet[Appendix C]{yun2014adaptive} and \citet[Appendix C.2]{yun2016labeled}.
    
We start by defining the following sets:
\begin{align}
    \gI_k &:= \left\{ t \in f^{-1}(k) : \bignorm{\widehat{\mY}_{t,:} - \mW_{t,:}}_2^2 \leq \frac{1}{4} (\widehat{r})^2 \log\frac{TH}{\delta} \right\}, \quad k \in [K] \\
    \gO &:= \left\{ t \in [T] : \min_{t' \in [T]} \bignorm{\widehat{\mY}_{t,:} - \mW_{t',:}}_2^2 \geq 4 (\widehat{r})^2 \log\frac{TH}{\delta} \right\}.
\end{align}
Intuitively, $\gI_k$ is the subset of trajectories of $f^{-1}(k)$ that will be classified accurately (with high probability), and $\gO$ is the set of trajectories that are clustered at the end (i.e., lines 12-13).

Then the following properties hold:
\begin{itemize}
    \item[(i)] \underline{$\left( \bigcup_{k=1}^K \gI_k \right) \cap \gQ_t = \emptyset$ for any $t \in \gO$.}
    This is because for any $t' \in \gI_k$,
    \begin{align*}
        \bignorm{\widehat{\mY}_{t,:} - (\widehat{\mY})_{t',:}}_2^2
        &\geq \frac{1}{2} \bignorm{\widehat{\mY}_{t,:} - \mW_{t,:}}_2^2 - \bignorm{(\widehat{\mY})_{t',:} - \mW_{t,:}}_2^2 \tag{$(a + b)^2 \leq 2(a^2 + b^2)$} \\
        &\geq 2 (\widehat{r})^2 \log\frac{TH}{\delta} - \frac{1}{4} (\widehat{r})^2 \log\frac{TH}{\delta} \tag{Definition of $\gO$ and $\gI_k$} \\
        &= \frac{7}{4} (\widehat{r})^2 \log\frac{TH}{\delta}
        \geq (\widehat{r})^2 \log\frac{TH}{\delta},
    \end{align*}
    i.e., $t' \in \gQ_t \Rightarrow t' \not\in \bigcup_{k=1}^{K_s} \gI_k$.

    \item[(ii)] \underline{$\left| [T] \setminus \cup_{k=1}^K \gI_k \right| \leq \frac{R T}{\log\frac{TH}{\delta}}$.}
    This is because
    \begin{equation*}
        \left| [T] \setminus \cup_{k=1}^K \gI_k \right| \leq \frac{\bignorm{\widehat{\mY} - \mW}_F^2}{\min_{t \in [T] \setminus \cup_{k=1}^K \gI_k} \bignorm{\widehat{\mY}_{t,:} - \mW_{t,:}}_2^2}
        \leq \frac{8R T  (\widehat{r})^2}{\frac{1}{4} (\widehat{r})^2 \log\frac{TH}{\delta}}
        = \frac{32 R T}{\log\frac{TH}{\delta}}
    \end{equation*}
    
     \item[(iii)] \underline{$\gI_k \subseteq \gQ_t$ for any $t \in \gI_k$.}
    This is because for any $t, t' \in \gI_k$,
    \begin{align*}
        \bignorm{\widehat{\mY}_{t',:} - \widehat{\mY}_{t,:}}_2^2 &\leq 2 \bignorm{\widehat{\mY}_{t',:} - \mW_{t',:}}_2^2 + 2 \bignorm{\widehat{\mY}_{t,:} - \mW_{t,:}}_2^2
        \leq (\widehat{r})^2 \log\frac{TH}{\delta}
    \end{align*}

    \item[(iv)] \underline{If $\gQ_t \cap \gI_k \neq \emptyset$, then $\gQ_t \cap \bigcup_{k' \neq k} \gI_k = \emptyset$.}
    We show this via \textit{reductio ad absurdum}.
    Suppose there exists $k \neq k'$ such that $\gQ_t \cap \gI_k \neq \emptyset$ and $\gQ_t \cap \gI_{k'} \neq \emptyset$.
    Let $t_k \in \gQ_t \cap \gI_k$ and $t_{k'} \in \gQ_t \cap \gI_{k'}$.
    Then, by definition of $\gQ_t$,
    \begin{align*}
        \widehat{r} \sqrt{\log\frac{TH}{\delta}} &\geq \bignorm{\widehat{\mY}_{t_k,:} - \widehat{\mY}_{t_{k'},:}}_2  \\
        &\geq \bignorm{\mW_{t_k,:} - \mW_{t_{k'},:}}_2 - \bignorm{\widehat{\mY}_{t_k,:} - \mW_{t_k,:}}_2 - \bignorm{\widehat{\mY}_{t_{k'},:} - \mW_{t_{k'},:}}_2 \\
        &\geq \Delta_\mW - \widehat{r} \sqrt{\log\frac{TH}{\delta}}.
    \end{align*}
    Recalling the definition of $\widehat{r}$, one can easily check that if $H \geq 2^8 \frac{S}{\Delta_\mW^2 \gamma_\ps} \log\frac{H}{\delta} \log\frac{TH}{\delta}$, above cannot be true, a contradiction.
\end{itemize}

\begin{claim}
    $\widehat{K} = K$.
\end{claim}
\begin{proof}
    Without loss of generality, assume that $|\gI_1| \geq |\gI_2| \geq \cdots \geq |\gI_K|$.
    
    We first show that for each $k \in [K]$,
    \begin{equation*}
        \exists t_k^\star \in \bigcup_{k' \in [K]} \gI_{k'} \setminus \bigcup_{\ell=1}^{k-1} \gS_\ell \quad \text{s.t.} \quad \bigabs{\gQ_{t_k^\star} \setminus \bigcup_{\ell=0}^{k-1} \gS_\ell} \geq |\gI_k|.
    \end{equation*}
    Due to the properties (iii) and (iv), and the greedy nature of lines 7-10, the above is indeed true.
    
    Let $\{t_1^\star, \cdots, t_K^\star\}$ be the selected ``centers'' with $\gI_k \subseteq \gQ_{t_k^\star}$.

    Now, we show that after $k$ has reached $K + 1$ in line 10, the \textbf{while} loop terminates.
    By property (ii), the number of remaining trajectories is
    \begin{equation*}
        \bigabs{[T] \setminus \bigcup_{k \in [K]} \gQ_{t_k^\star}}
        \leq \bigabs{[T] \setminus \bigcup_{k \in [K]} \gI_k}
        \overset{(ii)}{\leq} \frac{32 R T}{\log\frac{TH}{\delta}},
    \end{equation*}
    which is precisely the termination criterion at line 7.
\end{proof}

Thus, if a trajectory $t$ is misclassified in the sense that it gets assigned to $t_{k'}^\star$ instead of $t_k^\star$, then it must be that $\bignorm{\widehat{\mY}_{t,:} - \widehat{\mY}_{t_k^\star}}_2 > \bignorm{\widehat{\mY}_{t,:} - \widehat{\mY}_{t_{k'}^\star}}_2.$
By the triangle inequality, this then implies that
\begin{align*}
    \bignorm{\widehat{\mY}_{t,:} - \mW_{t,:}}_2 + \bignorm{\widehat{\mY}_{t_k^\star,:} - \mW_{t_k^\star,:}}_2 + \cancelto{0}{\bignorm{\mW_{t,:} - \mW_{t_k^\star,:}}_2} 
    &\geq \bignorm{\widehat{\mY}_{t,:} - \widehat{\mY}_{t_k^\star,:}}_2 \\
    &> \bignorm{\widehat{\mY}_{t,:} - \widehat{\mY}_{t_{k'}^\star,:}}_2 \\
    &> \bignorm{\mW_{t,:} - \mW_{t_{k'}^\star,:}}_2 - \bignorm{\widehat{\mY}_{t,:} - \mW_{t,:}}_2 - \bignorm{\widehat{\mY}_{t_{k'}^\star,:} - \mW_{t_{k'}^\star,:}}_2,
\end{align*}
i.e.,
\begin{align*}
    \bignorm{\widehat{\mY}_{t,:} - \mW_t}_2 &> \frac{1}{2} \left( \bignorm{\mW_{t,:} - \mW_{t_{k'}^\star,:}}_2 - \bignorm{\widehat{\mY}_{t_{k'}^\star,:} - \mW_{t_k^\star,:}}_2 - \bignorm{\widehat{\mY}_{t_k^\star,:} - \mW_{t_{k'}^\star,:}}_2 \right) \\
    &\geq \frac{1}{2} \left( \Delta_\mW - \frac{1}{2} \Delta_\mW \right) \tag{$t_k^\star, t_{k'}^\star \in \bigcup_{k=1}^K \gI_k$, see the proof of property (iv)}
    = \frac{1}{4} \Delta_\mW.
\end{align*}
\qed

\subsection{Proof of Lemma~\ref{lem:canonne}}
As the exact constants for the $\ell_2$-distance concentration have been left to exercise in \cite{canonne2020note}, we provide the complete proof here.

First, we have that
\begin{align*}
	\E\left[ \bignorm{\hat{p}_N - p}_2^2 \right] &= \sum_{s \in \gS} \E\left[ (\hat{p}_N(s) - p(s))^2 \right] \\
	&= \frac{1}{N^2} \sum_{s \in \gS} \Var[\Bin(N, p(s))] \tag{$N \hat{p}_N(s) \sim \Bin(N, p(s))$} \\
	&= \frac{1}{N} \sum_{s \in \gS} p(s) (1 - p(s)) \\
	&= \frac{1}{N} \left( 1 - \sum_{s \in \gS} p(s)^2 \right) 
	\leq \frac{1}{N} \left( 1 - \frac{1}{S} \right) \tag{Cauchy-Schwartz inequality},
\end{align*}
i.e., with $N \geq \frac{4}{\varepsilon^2} \left( 1 - \frac{1}{S} \right)$,
\begin{equation*}
	\E\left[ \bignorm{\hat{p}_N - p}_2 \right]
	\overset{(*)}{\leq} \sqrt{\E\left[ \bignorm{\hat{p}_N - p}_2^2 \right]}
	\leq \frac{\varepsilon}{2},
\end{equation*}
where $(*)$ follows from Jensen's inequality.

We now utilize the McDiarmid's inequality~\citep{mcdiarmid} to turn this into a \textit{tight} high-probability guarantee.
For $\vs_N = (s_1, \cdots, s_N)$, let us define
\begin{align*}
	f(\vs_N) := \Vert \hat{p}_N - p \Vert_2 = \sqrt{\sum_{i=1}^N (\hat{p}_N(s_i) - p(s_i))^2}.
\end{align*}
Now, let $\vs^{\neg i} = (s_1, \dots, s_{i-1}, s_i', s_{i+1}, \dots, s_N)$ for $s_i' \in \{0, 1\}$ and $\hat{p}_N^{\neg i}$ be the empirical distribution using $\vx^{\neg i}$.
Observe that 
\begin{align*}
	\vert f(\vs) - f(\vs^{\neg i}) \vert &= \left\vert  \Vert \hat{p}_N - p \Vert_2 - \Vert \hat{p}_N^{\neg i} - p \Vert_2 \right\vert \\
	&\le 
	\Vert \hat{p}_N - \hat{p}_N^{\neg i}  \Vert_2  \tag{Triangle inequality}   \\
	&= \frac{1}{N} |s_i - s_i'|
	\leq \frac{1}{N}.
\end{align*}
Thus, $f$ satisfies the bounded difference property with $c_i = 1 / N$.

Combining everything, we have that
\begin{align*}
	\sP\left( \bignorm{\hat{p}_N - p}_2 \geq \varepsilon \right) &\leq \sP\left( \bignorm{\hat{p}_N - p}_2 - \E[\bignorm{\hat{p}_N - p}_2] \geq \frac{\varepsilon}{2} \right) \tag{with $N \geq \frac{4}{\varepsilon^2}\left(1 - \frac{1}{S} \right)$} \\
	&\leq \exp\left( - \frac{2 \left(\frac{ \varepsilon}{2} \right)^2}{N \left(\frac{1}{N} \right)^2} \right) \\
	&= \exp\left( - \frac{\varepsilon^2 N}{2} \right).
\end{align*}
\qed

\clearpage

\section{\texorpdfstring{PROOF OF THEOREM~\ref{thm:likelihood} -- LIKELIHOOD IMPROVEMENT}{PROOF OF THEOREM 4.2 -- LIKELIHOOD IMPROVEMENT}}
\label{app:likelihood}
The proof largely follows the recipe of likelihood improvement for SBM clustering~\citep{yun2014spectral,yun2014adaptive}, which dates back to the seminal works of \citet{abbe2018sbm,amini2013sbm} and recently extended to other variants such as labeled SBM~\citep{yun2016labeled} and block Markov chains~\citep{jedra2023bmdp,sanders2020bmc}.

Let $\gH \subseteq [T]$ be the maximum subset satisfying the following:
whenever $t \in f^{-1}(k) \cap \gH$,
\begin{equation}
    \sum_{s, s' \in \gS} \widehat{N}_t(s, s') \log \frac{p^{(k)}(s' | s)}{p^{(k')}(s' | s)} \geq C (H - 1) \gD_\pi, \quad \forall k' \neq k.
\end{equation}

The first proposition bounds $|\gH^\complement|$ in-expectation, which then leads to a \textit{w.h.p.} guarantee via Markov's inequality:
\begin{proposition}
\label{prop:H-complement}
    Suppose that {\color{blue}$H \gtrsim \frac{\eta_p^2}{\gamma_\ps \gD_\pi} \log\frac{1}{\pi_{\min}}$}.
    With $C = \frac{1}{2}$, We have that
    \begin{equation}
        \E[|\gH^\complement|] \leq T (K - 1) \exp\left( - \frac{\gamma_\ps (H - 1) \gD_\pi}{128(4\eta_p^2 + 5 \log\eta_p)} \right).
    \end{equation}
    Thus, by Markov's inequality, we have that
    \begin{equation}
    \label{eqn:C1}
        \sP\left( |\gH^\complement| \geq T \exp\left( - \frac{\gamma_\ps (H - 1) \gD_\pi}{256(4\eta_p^2 + 5 \log\eta_p)} \right) \right) \leq K \exp\left( - \frac{\gamma_\ps (H - 1) \gD_\pi}{256(4\eta_p^2 + 5 \log\eta_p)} \right).
    \end{equation}
\end{proposition}
We next show that all trajectories in $\gH$ are correctly classified with high probability:
\begin{proposition}
\label{prop:H}
    Suppose that the following requirements on $H$ and $T$ hold:
    \begin{equation}
        H \gtrsim \frac{1}{\gamma_\ps} \left( S^2 \log T \vee \frac{\log T}{\pi_{\min}} \vee \frac{\eta_p \eta_\pi R S \log(TH)}{\alpha_{\min} H \Delta_\mW^2 \gD_\pi} \right), \
        TH \gtrsim \frac{(S \log T)^2}{\gamma_\ps \alpha_{\min} \gD_\pi^2}.
    \end{equation}
    Then, all trajectories in $\gH$ are correctly classified with probability at least $1 - \frac{1}{\sqrt{\pi_{\min}} T} - \frac{1}{TH}$.
\end{proposition}
Note that the probability in Eqn.~\eqref{eqn:C1} goes to $0$ as $T \rightarrow \infty$. from our requirement that $H \gtrsim \log T$ (ignoring other factors).
Finally, the proof concludes by taking the worst-case, namely, that all the trajectories in $\gH^\complement$ are misclassified, and applying a union bound.

\qed

\subsection{Proof of Proposition~\ref{prop:H-complement}}
First note that
\begin{align*}
    \E[|\gH^\complement|] &= \sum_{t \in [T]} \sP\left( \exists k' \neq f(t) \in [K] \text{ s.t. } \sum_{s, s' \in \gS} \widehat{N}_t(s, s') \log \frac{p^{(f(t))}(s' | s)}{p^{(k')}(s' | s)} < C (H - 1) \gD_\pi \right) \\
    &\leq (K - 1) \sum_{t \in [T]} \max_{k' \neq f(t) \in [K]} \sP\left( \sum_{s, s' \in \gS} \widehat{N}_t(s, s') \log \frac{p^{(f(t))}(s' | s)}{p^{(k')}(s' | s)} < C (H - 1) \gD_\pi \right) \\
    &\leq (K - 1) \sum_{t \in [T]} \max_{k' \neq f(t) \in [K]} \sP\left( \sum_{s, s' \in \gS} \widehat{N}_t(s, s') \log \frac{p^{(f(t))}(s' | s)}{p^{(k')}(s' | s)} < C (H - 1) \gD_\pi(t; k') \right),
\end{align*}
where we denote $\gD_\pi(t; k') := \gD_\pi^{(f(t),k')} = \sum_{s \in \gS} \pi^{(f(t))}(s) \KL(p^{(f(t))}(\cdot | s), p^{(k')}(\cdot | s))$.

As done in \citet{sanders2020bmc,jedra2023bmdp}, define an ``augmented'' Markov chain, $X_{t,h} := (s_{t,h}, s_{t,h+1})$.
For notational simplicity, let us denote $p := p^{(f(t))}$ and $\pi = \pi^{(f(t))}$.
    
This is a Markov chain on $\gS \times \gS$, with transition probability kernel $\tilde{p}(s_{t,h+1}, s_{t,h+2} | s_{t,h}, s_{t,h+1}) = p(s_{t,h+2} | s_{t,h+1})$ and stationary distribution $\tilde{\pi}(s, s') = \pi(s) p(s' | s)$.
We have the following property that we provide the proof at the end:
\begin{lemma}
\label{lem:pseudo-spectral-X}
    If the pseudo-spectral gap of $(s_{t,h})_h$ is $\gamma$, then so is that of $(X_{t,h})_h$.
\end{lemma}

Now define a function $\phi_{k'} : \gS \times \gS \rightarrow \sR$ as
\begin{equation}
    \phi_{k'}(X_{t,h}) := \sum_{s, s' \in \gS} \indicator[s_{t,h} = s, s_{t,h+1} = s'] \log\frac{p(s' | s)}{p^{(k')}(s' | s)}
    = \log\frac{p(s_{t,h+1} | s_{t,h})}{p^{(k')}(s_{t,h+1} | s_{t,h})}.
\end{equation}
To apply the above concentration, we compute the necessary quantities.
First, by Assumption~\ref{assumption:eta-regularity}, we have that for any $(s, s') \in \gS \times \gS$,
\begin{equation*}
    |\phi(s, s') - \E_\pi[\phi]| = \left| \log\frac{p(s' | s)}{p^{(k')}(s' | s)} - \sum_{s, s' \in \gS} \pi(s) p(s' | s) \log\frac{p(s' | s)}{p^{(k')}(s' | s)} \right|
    \leq 2 \log\eta_p.
\end{equation*}
Also, note that 
\begin{align*}
    \Var_\pi[\phi] &\leq \E_\pi[\phi^2] \\
    &= \sum_{s, s'} \pi(s) p(s' | s) \left( \log\frac{p(s' | s)}{p^{(k')}(s' | s)} \right)^2 \\
    &\leq 2 \left( \max_{s' \in \gS} \frac{p(s' | s) \vee p^{(k')}(s' | s)}{p(s' | s) \wedge p^{(k')}(s' | s)} \right)^2 \sum_{s \in \gS} \pi(s) \KL\left( p(\cdot | s), p^{(k')}(\cdot | s) \right) \tag{Lemma 19 of SM6.3 of \citet{sanders2020bmc}} \\
    &\leq 2\eta_p^2 \gD_\pi(t; k'). \tag{Assumption~\ref{assumption:eta-regularity}}
\end{align*}
With this, we have that
\begin{align*}
    &\sP_\pi\left( \sum_{s, s' \in \gS} \widehat{N}_t(s, s') \log \frac{p^{(k)}(s' | s)}{p^{(k')}(s' | s)} < C (H - 1) \gD_\pi(t; k') \right) \\
    &= \sP_\pi\left( \sum_{h=1}^{H-1} (\phi_{k'}(X_{t,h}) - \E[\phi_{k'}(X_{t,h})]) < - (1 - C) (H - 1) \gD_\pi(t; k') \right) \\
    &\leq \exp\left( - \frac{(1 - C)^2 (H - 1)^2 (\gD_\pi(t; k'))^2 \gamma_\ps}{32 \eta_p^2 (H - 1) \gD_\pi(t; k') + 40 (\log \eta_p) (H - 1) \gD_\pi(t; k')} \right) \tag{Lemma~\ref{lem:paulin}, Lemma~\ref{lem:pseudo-spectral-X}, and {\color{blue}$H > \frac{1}{\gamma_\ps}$}} \\
    &\leq \exp\left( - \frac{(1 - C)^2}{8(4\eta_p^2 + 5 \log\eta_p)} \gamma_\ps (H - 1) \gD_\pi(t; k') \right) \\
    &\leq \exp\left( - \frac{(1 - C)^2}{8(4\eta_p^2 + 5 \log\eta_p)} \gamma_\ps (H - 1) \gD_\pi \right).
\end{align*}
We choose $C = \frac{1}{2}$ and invoke Lemma~\ref{lem:paulin-3.10}, which yields
\begin{align*}
    &\sP\left( \sum_{s, s' \in \gS} \widehat{N}_t(s, s') \log \frac{p^{(k)}(s' | s)}{p^{(k')}(s' | s)} < \frac{(H - 1) \gD_\pi(t; k')}{2} \right) \\ &\leq \sqrt{\frac{1}{\pi_{\min}} \exp\left( - \frac{1}{32(4\eta_p^2 + 5 \log\eta_p)} \gamma_\ps (H - 1) \gD_\pi \right)} \\
    &= \sqrt{\exp\left( \log\frac{1}{\pi_{\min}} - \frac{1}{32(4\eta_p^2 + 5 \log\eta_p)} \gamma_\ps (H - 1) \gD_\pi \right)} \\
    &\leq \exp\left( - \frac{\gamma_\ps (H - 1) \gD_\pi}{128(4\eta_p^2 + 5 \log\eta_p)} \right),
\end{align*}
where the last inequality is true when {\color{blue}$H - 1 > \frac{64 (4\eta_p^2 + 5 \log\eta_p)}{\gamma_\ps \gD_\pi} \log\frac{1}{\pi_{\min}}$}.
\qed

\begin{proof}[Proof of Lemma~\ref{lem:pseudo-spectral-X}]
    For notational simplicity, we omit the dependency on $t$.
    Let $\mP$ and $\widetilde{\mP}$ be the (bounded) linear operators corresponding to the original Markov chain $s_h$ and the augmented Markov chain $X_h := (s_h, s_{h+1})$, respectively.
    Then, it is clear that $\widetilde{\mP} = \bigoplus_{s \in \gS} \mP$, where $\bigoplus$ denotes the (matrix) direct sum.

    We now denote $\mP^*$ and $\widetilde{\mP}^*$ be their adjoints on $L^2(\pi)$ and $L^2(\tilde{\pi})$, respectively, i.e., their time reversals~\citep{lezaud1998chernoff}.
    By a well-known property of direct sum~\citep{hornjohnson}, we have that $\widetilde{\mP}^* = \bigoplus_{s \in \gS} \mP^*$.
    This can also be seen via explicit computation:
    \begin{equation*}
        \tilde{p}^*(y, y' | x, x')
        := \frac{\tilde{\pi}(y, y') \tilde{p}(x, x' | y, y')}{\tilde{\pi}(x, x')}
        = \frac{\pi(y) p(y' | y) p(x' | x) \indicator[y' = x]}{\pi(x) \tilde{p}(x' | x)}
        = \underbrace{\frac{\pi(y) p(x | y)}{\pi(x)}}_{= p^*(y | x)} \indicator[y' = x].
    \end{equation*}
    The off-diagonal zero blocks correspond to where $\indicator[y' = x] = 0$.

    Thus, we have that for every $k \in \sN$, $(\widetilde{\mP}^*)^k \widetilde{\mP}^k = \bigoplus_{s \in \gS} (\mP^*)^k \mP^k$, i.e., the spectrum of $(\widetilde{\mP}^*)^k \widetilde{\mP}^k$ is precisely $S$ copies of that of $(\mP^*)^k \mP^k$.
    This immediately implies that $\gamma_\ps(\widetilde{\mP}) = \gamma_\ps(\mP)$.
\end{proof}

\subsection{Proof of Proposition~\ref{prop:H}}
Recall that
\begin{equation}
    \gL(k; t) := \sum_{h=1}^{H-1} \log \widehat{p}_0^{(k)}(s_{t,h+1} | s_{t,h}) = \sum_{s, s' \in \gS} \widehat{N}_t(s, s') \log \widehat{p}_0^{(k)}(s' | s)
\end{equation}
where
\begin{equation}
    \widehat{p}_0^{(k)}(s' | s) := \frac{\sum_{t \in (\widehat{f}_0)^{-1}(k)} \sum_{h \in [H-1]} \indicator[s_{t,h} = s, s_{t,h+1} = s']}{\sum_{t \in (\widehat{f}_0)^{-1}(k)} \sum_{h \in [H-1]} \indicator[s_{t,h} = s]}.
\end{equation}

Then, Algorithm~\ref{alg:likelihood} ensures that $t$ is misclassified if and only if
\begin{equation}
    E_t \triangleq \gL(\widehat{f}(t); t) - \gL(f(t); t) > 0.
\end{equation}
We decompose $E_t$ as
\begin{align*}
    E_t &= \underbrace{\sum_{s, s' \in \gS} \widehat{N}_t(s, s') \log\frac{p^{(\widehat{f}_0(t))}(s' | s)}{p^{(f(t))}(s' | s)}}_{\triangleq E_{1,t}} \\
    &\quad + \underbrace{\sum_{s, s' \in \gS} \widehat{N}_t(s, s') \log\frac{\widehat{p}_0^{(\widehat{f}_0(t))}(s' | s)}{p^{(\widehat{f}_0(t))}(s' | s)}}_{\triangleq E_{2,t}} + \underbrace{\sum_{s, s' \in \gS} \widehat{N}_t(s, s') \log\frac{p^{(f(t))}(s' | s)}{\widehat{p}_0^{(f(t))}(s' | s)}}_{\triangleq E_{3,t}}
\end{align*}

We now upper bound $E_{2,t}$ and $E_{3,t}$.

To do so, we start with the following concentration for $\widehat{N}_t(s, s')$, which can be derived via Lemma~\ref{lem:paulin} (similar to Eqn.~\eqref{eqn:pi-concentration}):
for each $k \in [K]$ and $t \in f^{-1}(k)$
\begin{equation}
\label{eqn:pi-p-concentration}
    \sP\left( \bigabs{ \widehat{N}_t(s, s') - (H - 1) \pi^{(k)}(s) p^{(k)}(s' | s) } \gtrsim \sqrt{\frac{H \pi^{(k)}(s) p^{(k)}(s' | s)}{\gamma_\ps} \log\frac{1}{\pi_{\min} \delta^2} } \right) \leq \delta,
\end{equation}
i.e., $\widehat{N}_t(s, s') \lesssim H \pi^{(k)}(s) p^{(k)}(s' | s) + \sqrt{\frac{H \pi^{(k)}(s) p^{(k)}(s' | s)}{\gamma_\ps} \log T}$ with probability at least $1 - \frac{1}{\sqrt{\pi_{\min}} T^2}$

We then bound the intermediate estimation error of the transition probabilities:
\begin{lemma}
\label{lem:intermediate}
    Let $\delta \in (0, 1)$, and suppose that {\color{blue}$H \gtrsim \frac{1}{\gamma_\ps \pi_{\min}} \log\frac{1}{\delta}$}.
    Then, for each $k \in [K]$ and $s, s' \in \gS$, the following holds with probability at least $1 - \delta$:
    \begin{equation}
        \left| \log\frac{\widehat{p}_0^{(k)}(s' | s)}{p^{(k)}(s' | s)} \right|
        \lesssim \sqrt{\frac{1}{\gamma_\ps \alpha_k T H \pi^{(k)}(s) p^{(k)}(s' | s)} \log\frac{\alpha_k T}{\sqrt{\pi_{\min}} \delta} \log\frac{1}{\sqrt{\pi_{\min}} \delta}} + \frac{\eta_p \eta_\pi e^{(0)}}{\alpha_k T}.
    \end{equation}
\end{lemma}

Then, for each $t \in f^{-1}(k)$, we have that with probability at least $1 - \frac{1}{\sqrt{\pi_{\min}} T^2}$,
\begin{align*}
    E_{2,t} &\lesssim \sum_{s, s' \in \gS} H \pi^{(k)}(s) p^{(k)}(s' | s) \left( \sqrt{\frac{1}{\gamma_\ps \alpha_k T H \pi^{(k)}(s) p^{(k)}(s' | s)}} \log T + \frac{\eta_p \eta_\pi e^{(0)}}{\alpha_k T} \right) \\
    &\quad + \sum_{s, s' \in \gS} \sqrt{\frac{H \pi^{(k)}(s) p^{(k)}(s' | s)}{\gamma_\ps} \log T} \left( \sqrt{\frac{1}{\gamma_\ps \alpha_k T H \pi^{(k)}(s) p^{(k)}(s' | s)}} \log T + \frac{\eta_p \eta_\pi e^{(0)}}{\alpha_k T} \right) \\
    &= \sum_{s, s' \in \gS} \left( \sqrt{\frac{H \pi^{(k)}(s) p^{(k)}(s' | s)}{\gamma_\ps \alpha_k T}} \log T \right) + \frac{H \eta_p \eta_\pi e^{(0)}}{\alpha_k T} \\
    &\quad + \sum_{s, s' \in \gS} \left( \sqrt{\frac{1}{\alpha_k T}} \frac{(\log T)^{3/2}}{\gamma_\ps} + \sqrt{\frac{H \pi^{(k)}(s) p^{(k)}(s' | s)}{\gamma_\ps} \log T} \frac{\eta_p \eta_\pi e^{(0)}}{\alpha_k T} \right) \\
    &\leq S \sqrt{\frac{H}{\gamma_\ps \alpha_k T}} \log T + \frac{H \eta_p \eta_\pi e^{(0)}}{\alpha_k T} + \frac{S^2 (\log T)^{3/2}}{\gamma_\ps \sqrt{\alpha_k T}} + S \sqrt{\frac{H}{\gamma_\ps} \log T} \frac{\eta_p \eta_\pi e^{(0)}}{\alpha_k T} \tag{$\sum_{s, s' \in \gS} \sqrt{\pi^{(k)}(s) p^{(k)}(s' | s)} \leq S$} \\
    &\lesssim S \sqrt{\frac{H}{\gamma_\ps \alpha_k T}} \log T + \frac{H \eta_p \eta_\pi e^{(0)}}{\alpha_k T} \tag{when {\color{blue}$H \gtrsim \frac{S^2 \log T}{\gamma_\ps}$}}
\end{align*}
and the same for $E_{3,t}$.

Now invoking Theorem~\ref{thm:initial-spectral} and the union bound, we have that with probability at least $1 - \frac{1}{\sqrt{\pi_{\min}} T} - \frac{1}{T H}$,
\begin{equation}
    E_{2,t}, E_{3,t} \lesssim S \sqrt{\frac{H}{\gamma_\ps \alpha_{\min} T}} \log T + \frac{H \eta_p \eta_\pi}{\alpha_{\min} T} \frac{T R S}{\gamma_\ps H \Delta_\mW^2} \log(TH).
\end{equation}

With this, we deduce that $t \in \gH$ is misclassified \textit{only if}
\begin{equation*}
    H \gD_\pi \lesssim -E_{1,t} < E_{2,t} + E_{3,t} \lesssim S \sqrt{\frac{H}{\gamma_\ps \alpha_{\min} T}} \log T + \frac{\eta_p \eta_\pi}{\alpha_{\min}} \frac{R S}{\gamma_\ps \Delta_\mW^2} \log(TH),
\end{equation*}
where the first inequality follows from the definition of $\gH$.
By taking the contrapositive, we can deduce that if
\begin{equation*}
    \color{blue}
    H \gD_\pi \gtrsim S \sqrt{\frac{H}{\gamma_\ps \alpha_{\min} T}} \log T \vee \frac{\eta_p \eta_\pi}{\alpha_{\min}} \frac{R S}{\gamma_\ps \Delta_\mW^2} \log(TH)
\end{equation*}
hold, then $t \in \gH$ must be classified correctly after one-shot likelihood improvement.

We now collect all the requirements in {\color{blue} blue} throughout the proof, which gives:
\begin{equation}
    H \gtrsim \frac{S^2 \log T}{\gamma_\ps} \vee \frac{\log T}{\gamma_\ps \pi_{\min}} \vee \frac{\eta_p \eta_\pi R S \log (TH)}{\gD_\pi \alpha_{\min} \gamma_\ps \Delta_\mW^2}, \quad TH \gtrsim \frac{(S \log T)^2}{\gD_\pi^2 \gamma_\ps \alpha_{\min}}.
\end{equation}

The proof then concludes with union bound over $t \in \gH$. 
\qed

\subsection{Proof of Lemma~\ref{lem:intermediate}}
We denote $N_t(s) := (H - 1) \pi^{(f(t))}(s)$ and $N_t(s, s') := (H - 1) \pi^{(f(t))}(s) p^{(f(t))}(s' | s)$ as the expected number of visitations/transitions under the respective chain's stationary distribution.

As $\frac{x}{1 + x} \leq \log(1 + x) \leq x$ for $x > -1$, we have that
\begin{align*}
    \left| \log\frac{\widehat{p}_0^{(k)}(s' | s)}{p^{(k)}(s' | s)} \right| &\leq \left| \frac{\widehat{p}_0^{(k)}(s' | s)}{p^{(k)}(s' | s)} - 1 \right| \\
    &= \left| \frac{\sum_{t \in (\widehat{f}_0)^{-1}(k)} \widehat{N}_t(s, s')}{\sum_{t \in (\widehat{f}_0)^{-1}(k)} \widehat{N}_t(s)} \frac{1}{p^{(k)}(s' | s)} - 1 \right| \\
    &= \left| \underbrace{
    \frac{\sum_{t \in (\widehat{f}_0)^{-1}(k)} \widehat{N}_t(s, s')}{\sum_{t \in (\widehat{f}_0)^{-1}(k)} N_t(s, s')}
    \frac{\sum_{t \in (\widehat{f}_0)^{-1}(k)} N_t(s)}{\sum_{t \in (\widehat{f}_0)^{-1}(k)} \widehat{N}_t(s)}
    }_{\triangleq E_{est}}
   \underbrace{
   \frac{\sum_{t \in (\widehat{f}_0)^{-1}(k)} N_t(s, s')}{\sum_{t \in (\widehat{f}_0)^{-1}(k)} N_t(s)}
   \frac{1}{p^{(k)}(s' | s)}
   }_{\triangleq E_{cluster}} - 1 \right|.
\end{align*}

Throughout the proof, let $\gE_0^{(k)} := \widehat{f}_0^{-1}(k) \setminus f^{-1}(k)$ and $\gC_0^{(k)} := f^{-1}(k) \cap \widehat{f}_0^{-1}(k)$, which satisfies $\widehat{f}_0^{-1}(k) = \gE_0^{(k)} \dot{\cup} \gC_0^{(k)}$.

We first bound $E_{est}$:
\begin{align*}
    \left| \frac{\sum_{t \in \widehat{f}_0^{-1}(k)} \widehat{N}_t(s, s')}{\sum_{t \in \widehat{f}_0^{-1}(k)} N_t(s, s')} - 1 \right| &= \frac{\left| \sum_{t \in \widehat{f}_0^{-1}(k)} \left( \widehat{N}_t(s, s') - N_t(s, s') \right)\right|}{\sum_{t \in \gE_0^{(k)}} N_t(s, s') + \sum_{t \in \gC_0^{(k)}} N_t(s, s')} \\
    &\leq \frac{\left| \sum_{t \in \widehat{f}_0^{-1}(k)} \left( \widehat{N}_t(s, s') - N_t(s, s') \right)\right|}{\sum_{t \in \gC_0^{(k)}} N_t(s, s')} \\
    &\leq \frac{\left| \sum_{t \in \widehat{f}_0^{-1}(k)} \left( \widehat{N}_t(s, s') - N_t(s, s') \right)\right|}{(\alpha_k T - |\gE_0^{(k)}|) (H - 1) \pi^{(k)}(s) p^{(k)}(s' | s)} \tag{$|\gC_0^{(k)}| \geq \alpha_k T - |\gE_0^{(k)}|$} \\
    &\leq \frac{2 \left| \sum_{t \in \widehat{f}_0^{-1}(k)} \left( \widehat{N}_t(s, s') - N_t(s, s') \right)\right|}{\alpha_k T (H - 1) \pi^{(k)}(s) p^{(k)}(s' | s)}. \tag{$|\gE_0^{(k)}| \leq e^{(0)} \leq \frac{\alpha_{\min} T}{2} \leq \frac{\alpha_k T}{2}$}
\end{align*}
Defining $X_t := \widehat{N}_t(s, s') - N_t(s, s')$, note that $X_t$'s are independent across $t$ and satisfy $\E_\pi[X_t] = 0$.
Let us define an event $\gM := \left\{ |X_t| \leq M, \ \forall t \in \widehat{f}_0^{-1}(k) \right\}$, where $M > 0$ is chosen later.
By Hoeffding's inequality~\citep{hoeffding} conditioned on $\gM$, we have that for any $z \geq 0$,
\begin{equation*}
    \sP_\pi\left( \left| \sum_{t \in \widehat{f}_0^{-1}(k)} \left( \widehat{N}_t(s, s') - N_t(s, s') \right)\right| \geq z \middle| \gM \right) \leq \exp\left( - \frac{z^2}{2 M^2 |\widehat{f}_0^{-1}(k)|} \right).
\end{equation*}
Using Lemma~\ref{lem:paulin-3.10} to change $\sP_\pi$ to $\sP$ and the fact that $|\widehat{f}_0^{-1}(k)| \leq \frac{3\alpha_k T}{2}$, we have
\begin{equation*}
    \sP\left( \left| \sum_{t \in \widehat{f}_0^{-1}(k)} \left( \widehat{N}_t(s, s') - N_t(s, s') \right)\right| \geq z \middle| \gM \right) \leq \frac{1}{\sqrt{\pi_{\min}}} \exp\left( - \frac{z^2}{6 M^2 \alpha_k T} \right).
\end{equation*}
We choose $M \asymp \sqrt{\frac{H \pi(s) p(s' | s)}{\gamma_\ps} \log\frac{\alpha_k T}{\sqrt{\pi_{\min}} \delta}}$.
Then, by Eqn.~\eqref{eqn:pi-p-concentration}, we have that $\sP(\gM^\complement) \leq \frac{\delta}{2}.$
By reparametrizing $z \asymp M \sqrt{\alpha_k T \log\frac{1}{\sqrt{\pi_{\min}} \delta}}$ and union bound, we have that
\begin{align*}
    \sP\left( \left| \sum_{t \in \widehat{f}_0^{-1}(k)} \left( \widehat{N}_t(s, s') - N_t(s, s') \right)\right| \geq \sqrt{\frac{\alpha_k T H \pi(s) p(s' | s)}{\gamma_\ps} \log\frac{\alpha_k T}{\sqrt{\pi_{\min}} \delta} \log\frac{1}{\sqrt{\pi_{\min}} \delta}} \right) \leq \delta.
\end{align*}
Then, combining everything, the following holds with probability at least $1 - \delta$:
\begin{equation}
    \left| \frac{\sum_{t \in \widehat{f}_0^{-1}(k)} \widehat{N}_t(s, s')}{\sum_{t \in \widehat{f}_0^{-1}(k)} N_t(s, s')} - 1 \right| \lesssim \sqrt{\frac{1}{\gamma_\ps \alpha_k T H \pi(s) p(s' | s)} \log\frac{\alpha_k T}{\sqrt{\pi_{\min}} \delta} \log\frac{1}{\sqrt{\pi_{\min}} \delta}}.
\end{equation}

For the other term, we have that
\begin{align*}
    \left| \frac{\sum_{t \in \widehat{f}_0^{-1}(k)} N_t(s)}{\sum_{t \in \widehat{f}_0^{-1}(k)} \widehat{N}_t(s)} - 1 \right| &\leq \frac{1}{\sum_{t \in \widehat{f}_0^{-1}(k)} \widehat{N}_t(s)} \left| \sum_{t \in \widehat{f}_0^{-1}(k)} \left( \widehat{N}_t(s) - N_t(s) \right)\right| \\
    &\lesssim \frac{1}{\sum_{t \in \widehat{f}_0^{-1}(k)} \widehat{N}_t(s)} \sqrt{\frac{\alpha_k T H \pi(s)}{\gamma_\ps} \log\frac{\alpha_k T}{\sqrt{\pi_{\min}} \delta} \log\frac{1}{\sqrt{\pi_{\min}} \delta}} \tag{with probability at least $1 - \delta / 2$} \\
    &\overset{(*)}{\lesssim} \frac{1}{\sum_{t \in \widehat{f}_0^{-1}(k)} N_t(s)} \sqrt{\frac{\alpha_k T H \pi(s)}{\gamma_\ps} \log\frac{\alpha_k T}{\sqrt{\pi_{\min}} \delta} \log\frac{1}{\sqrt{\pi_{\min}} \delta}} \tag{with probability at least $1 - \delta/ 2$, assuming that {\color{blue}$H \gtrsim \frac{1}{\gamma_\ps \pi_{\min}} \log\frac{1}{\delta}$}} \\
    &\lesssim \sqrt{\frac{1}{\gamma_\ps \alpha_k T H \pi(s)} \log\frac{\alpha_k T}{\sqrt{\pi_{\min}} \delta} \log\frac{1}{\sqrt{\pi_{\min}} \delta}},
\end{align*}
where $(*)$ follows from Lemma~\ref{lem:paulin}.

We now bound $E_{cluster}$:
\begin{align*}
    \frac{\sum_{t \in (\widehat{f}_0)^{-1}(k)} N_t(s, s')}{\sum_{t \in (\widehat{f}_0)^{-1}(k)} N_t(s)} \frac{1}{p^{(k)}(s' | s)} &= \frac{\sum_{t \in \gC_0^{(k)}} p^{(k)}(s' | s) N_t(s) + \sum_{t \in \gE_0^{(k)}} p^{(\widehat{f}_0(t))}(s' | s) N_t(s)}{\left( \sum_{t \in \gC_0^{(k)}} N_t(s) + \sum_{t \in \gE_0^{(k)}} N_t(s) \right) p^{(k)}(s' | s)} \\
   &\leq \frac{\sum_{t \in \gC_0^{(k)}} N_t(s) + \eta_p \sum_{t \in \gE_0^{(k)}} N_t(s)}{\sum_{t \in \gC_0^{(k)}} N_t(s) + \sum_{t \in \gE_0^{(k)}} N_t(s)} \tag{Assumption~\ref{assumption:eta-regularity}} \\
   &\leq \frac{\sum_{t \in \gC_0^{(k)}} N_t(s) + \eta_p \sum_{t \in \gE_0^{(k)}} N_t(s)}{\sum_{t \in \gC_0^{(k)}} N_t(s)} \\
   &= 1 + \frac{\eta_p \sum_{t \in \gE_0^{(k)}} N_t(s)}{\sum_{t \in \gC_0^{(k)}} N_t(s)}.
\end{align*}
Recalling that $\pi^{(k)}(s)$'s satisfy $\eta$-regularity (Assumption~\ref{assumption:eta-regularity}), we then have that
\begin{align*}
    \frac{\sum_{t \in (\widehat{f}_0)^{-1}(k)} N_t(s, s')}{\sum_{t \in (\widehat{f}_0)^{-1}(k)} N_t(s)} \frac{1}{p^{(k)}(s' | s)} - 1 &\leq \frac{\eta_p \sum_{t \in \gE_0^{(k)}} (H - 1) \pi^{(\widehat{f}_0(t))}(s)}{(\alpha_k T - |\gE_0^{(k)}|) (H - 1) \pi^{(k)}(s)} \tag{$|\gC_0^{(k)}| \geq \alpha_k T - |\gE_0^{(k)}|$} \\
    &\leq \frac{2 \eta_p \eta_\pi |\gE_0^{(k)}|}{\alpha_k T} \tag{Assumption~\ref{assumption:eta-regularity}, $|\gE_0^{(k)}| \leq e^{(0)} \leq \frac{\alpha_{\min} T}{2} \leq \frac{\alpha_k T}{2}$}
\end{align*}
We also have that
\begin{align*}
    \frac{\sum_{t \in (\widehat{f}_0)^{-1}(k)} N_t(s, s')}{\sum_{t \in (\widehat{f}_0)^{-1}(k)} N_t(s)} \frac{1}{p^{(k)}(s' | s)} &= \frac{\sum_{t \in \gC_0^{(k)}} p^{(k)}(s' | s) N_t(s) + \sum_{t \in \gE_0^{(k)}} p^{(\widehat{f}_0(t))}(s' | s) N_t(s)}{\left( \sum_{t \in \gC_0^{(k)}} N_t(s) + \sum_{t \in \gE_0^{(k)}} N_t(s) \right) p^{(k)}(s' | s)} \\
    &\geq \frac{\sum_{t \in \gC_0^{(k)}} N_t(s) + \eta_p^{-1} \sum_{t \in \gE_0^{(k)}} N_t(s)}{\sum_{t \in \gC_0^{(k)}} N_t(s) + \sum_{t \in \gE_0^{(k)}} N_t(s)} \tag{Assumption~\ref{assumption:eta-regularity}} \\
    &= 1 + \frac{(\eta_p^{-1} - 1) \sum_{t \in \gE_0^{(k)}} N_t(s)}{\sum_{t \in \gC_0^{(k)}} N_t(s) + \sum_{t \in \gE_0^{(k)}} N_t(s)} > 0,
\end{align*}
i.e., we can bound $E_{cluster}$ as
\begin{equation}
    |E_{cluster} - 1| \leq \frac{2 \eta_p \eta_\pi e^{(0)}}{\alpha_k T}.
\end{equation}
Combining everything, we are done.
\qed

\newpage
\section{\texorpdfstring{PROOF OF PROPOSITION~\ref{prop:gaps} -- RELATIONSHIP BETWEEN DIFFERENT GAPS}{PROOF OF PROPOSITION 5.1 -- RELATIONSHIP BETWEEN DIFFERENT GAPS}}
\label{app:gaps}

\subsection{\texorpdfstring{(a) Relationship between $\gD_\pi$ and \citet{kausik2023mixture}'s $\alpha, \Delta$}{(a) Relationship between KL-based Separation and Kausik et al. (2023)'s L2-based Separation}}
We first recall the following lemma connecting the KL divergence and the $L_2$-distance between two probability measures on a finite state space $\gS$:
\begin{lemma}
\label{lem:KL-L2}
    For two probability mass functions $p, q$ over a finite state space $\gS$, we have
    \begin{equation}
        \frac{\log\frac{e}{2}}{q_{\max} \vee p_{\max}} L_2(p, q) \leq \KL(p, q) \leq \frac{1}{\min_{s \in \gS} q(s)} L_2(p, q),
    \end{equation}
    where $p_{\max} := \max_{s \in \gS} p(s)$, $q_{\max} := \max_{s \in \gS} q(s)$, and $L_2(p, q) := \sum_{s \in \gS} (p(s) - q(s))^2$
\end{lemma}
Then, note that for each $k \neq k'$,
\begin{equation*}
    \gD^{(k,k')} \geq \pi^{(k)}(s(k, k')) \KL\left( p^{(k)}(\cdot | s(k, k')), p^{(k')}(\cdot | s(k, k')) \right)
    \geq \left( \log\frac{e}{2} \right) \frac{\alpha \Delta^2}{p_{\max}},
\end{equation*}
where we recall that $p_{\max} := \max_{k \in [K]} \max_{s, s' \in \gS} p^{(k)}(s' | s)$.
\qed

\begin{remark}
    Note that the first inequality is tight if such $s(k, k')$ is unique because $s(k, k')$ is the only state from which $k$ and $k'$ can be meaningfully distinguished.
    Also, if there are $W(k, k')$ number of such $s(k,k')$'s, then the lower bound can be improved to $\left(\log\frac{e}{2}\right) W(k, k') \frac{\alpha \Delta^2}{p_{\max}}$.
\end{remark}

\begin{proof}[Proof of Lemma~\ref{lem:KL-L2}]
    The second inequality is due to \citet[Lemma 6.1]{csiszar2006tree}.
    The first inequality is due to a \href{https://mathoverflow.net/questions/393458/is-there-an-inequality-relation-between-kl-divergence-and-l-2-norm/434695#434695}{Stackoverflow post and user Ze-Nan Li's answer}, whose proof we reproduce here for completeness.

    Let us define $\eta(s) := \frac{q(s) - p(s)}{p(s)}$ and partition $\gS$ as follows:
    \begin{equation}
        \gS = \underbrace{\left\{ s \in \gS : \eta(s) > 1 \right\}}_{\triangleq \gS_+} \dot{\cup} \underbrace{\left\{ s \in \gS : \eta(s) \leq 1 \right\}}_{\triangleq \gS_-}.
    \end{equation}
    Then it is easy to see that
    \begin{enumerate}
        \item[(i)] For $s \in \gS_+$, $1 + \eta(s) \leq 2^{\eta(s)} = e^{\eta(s) \log 2}$, and
        \item[(ii)] For $s \in \gS_-$, $1 + \eta(s) \leq e^{\eta(s) - c \eta(s)^2}$, where $c := \log\frac{e}{2}$.
    \end{enumerate}
    Then,
    \begin{align*}
        \KL(p, q) &:= \sum_{s \in \gS} p(s) \log\frac{p(s)}{q(s)} \\
        &= - \sum_{s \in \gS_+} p(s) \log(1 + \eta(s)) - \sum_{s \in \gS_-} p(s) \log(1 + \eta(s)) \\
        &\geq - \log 2 \sum_{s \in \gS_+} p(s) \eta(s) - \sum_{s \in \gS_-} p(s) (\eta(s) - c \eta(s)^2) \\
        &= c \sum_{s \in \gS_+} p(s) \eta(s) - \cancelto{0}{\sum_{s \in \gS} p(s) \eta(s)} + c \sum_{s \in \gS_-} p(s) \eta(s)^2 \\
        &= c \sum_{s \in \gS_+} (q(s) - p(s)) + c \sum_{s \in \gS_-} \frac{(q(s) - p(s))^2}{p(s)} \\
        &\geq c \sum_{s \in \gS_+} q(s) \frac{q(s) - p(s)}{q(s)} + \frac{c}{p_{\max}} \sum_{s \in \gS_-} (q(s) - p(s))^2.
    \end{align*}
    Now, for $s \in \gS_+$, we have that $0 < \frac{q(s) - p(s)}{q(s)} \leq 1$, and thus,
    \begin{align*}
        \KL(p, q) &\geq c \sum_{s \in \gS_+} q(s) \frac{q(s) - p(s)}{q(s)} + \frac{c}{p_{\max}} \sum_{s \in \gS_-} (q(s) - p(s))^2 \\
        &\geq c \sum_{s \in \gS_+} q(s) \left( \frac{q(s) - p(s)}{q(s)} \right)^2 + \frac{c}{p_{\max}} \sum_{s \in \gS_-} (q(s) - p(s))^2 \\
        &\geq \frac{c}{q_{\max} \vee p_{\max}} \underbrace{\sum_{s \in \gS} (q(s) - p(s))^2}_{= L_2(p, q)}.
    \end{align*}
\end{proof}

\subsection{\texorpdfstring{(b) Relationship between $\gD_\pi$ and $\Delta_{\mathbf{W}}$}{(a) Relationship between KL-based Separation and our L2-based Separation}}
We first provide the full statement of (b), then provide its proof:
\begin{proposition}
\label{prop:DeltaW-Dpi}
    We have
    \begin{equation}
        \Delta_\mW^2 \leq \min_{k \neq k'} \left\{ \frac{2 p_{\max}}{\log \frac{e}{2}} \gD_\pi^{(k,k')} + 4 H^2\left( \pi^{(k)}, \pi^{(k')} \right) \right\},
    \end{equation}
    where $H^2(\cdot, \cdot)$ is the Hellinger distance: $H^2(p, q) := \frac{1}{2} \left( \sum_{s \in \gS} (\sqrt{p(s)} - \sqrt{q(s)})^2 \right)^{\frac{1}{2}}$.
    Furthermore, we also have the following slightly more relaxed bound:
    \begin{equation}
        \Delta_\mW^2 \leq 7 \left( p_{\max} \gD_\pi + \left( \ceil{\log_\rho M^{-1}} + \frac{1}{1 - \rho} \right) \sqrt{\frac{1}{2 \pi_{\min}} \gD_\pi} \right),
    \end{equation}
    where $\rho$ and $M$ are the contraction constants from the definition of uniformly ergodic Markov chains\footnote{Recall that the Markov chain with stationary distribution $\pi$ and transition matrix $P$ is uniformly ergodic if there exists $\rho \in (0, 1)$ and $M > 0$ such that $\max_{s \in \gS} \bignorm{P^H(s, \cdot), \pi}_1 \leq M \rho^H$ for all $H \in \sN$.}.
\end{proposition}
\begin{proof}
    For each $s \in \gS$,
    \begin{align*}
        &\bignorm{\sqrt{\pi^{(k)}(s)} p^{(k)}(\cdot | s) - \sqrt{\pi^{(k')}(s)} p^{(k')}(\cdot | s)}_2^2 \\
        &= \bignorm{\sqrt{\pi^{(k)}(s)} \left( p^{(k)}(\cdot | s) - p^{(k')}(\cdot | s) \right) + \left(\sqrt{\pi^{(k)}(s)} - \sqrt{\pi^{(k')}(s)}\right) p^{(k')}(\cdot | s)}_2^2 \\
        &\leq 2 \pi^{(k)}(s) \bignorm{p^{(k)}(\cdot | s) - p^{(k')}(\cdot | s)}_2^2 + 2 \left( \sqrt{\pi^{(k)}(s)} - \sqrt{\pi^{(k')}(s)} \right)^2 \bignorm{p^{(k')}(\cdot | s)}_2^2 \\
        &\leq 2 \pi^{(k)}(s) \bignorm{p^{(k)}(\cdot | s) - p^{(k')}(\cdot | s)}_2^2 + 2 \left( \sqrt{\pi^{(k)}(s)} - \sqrt{\pi^{(k')}(s)} \right)^2 \\
        &\leq \frac{2 p_{\max}}{\log\frac{e}{2}} \pi^{(k)}(s) \KL\left( p^{(k)}(\cdot | s), p^{(k')}(\cdot | s) \right) + 2 \left( \sqrt{\pi^{(k)}(s)} - \sqrt{\pi^{(k')}(s)} \right)^2\tag{Lemma~\ref{lem:KL-L2}}
    \end{align*}
    We conclude the first part by summing over $s \in \gS$ and taking a minimum over $k \neq k'$.

    For the second part, we start by recalling Le Cam's inequality~\citep[Lemma 2.3]{tsybakov09}:
    \begin{equation*}
        H^2\left( \pi^{(k)}, \pi^{(k')} \right) \leq \TV\left( \pi^{(k)}, \pi^{(k')} \right).
    \end{equation*}
    We utilize the perturbation bound for Markov chains~\citep[Corollary 3.1]{mitrophanov2005sensitivity} to obtain
    \begin{align*}
        \TV\left( \pi^{(k)}, \pi^{(k')} \right) &\leq \left( \ceil{\log_\rho M^{-1}} + \frac{1}{1 - \rho} \right) \max_{s \in \gS} \TV\left( p^{(k)}(\cdot | s), p^{(k')}(\cdot | s) \right) \\
        &\leq \left( \ceil{\log_\rho M^{-1}} + \frac{1}{1 - \rho} \right) \sqrt{\frac{1}{2} \max_{s \in \gS} \KL\left( p^{(k)}(\cdot | s), p^{(k')}(\cdot | s) \right)} \tag{Pinsker's inequality} \\
        &\leq \left( \ceil{\log_\rho M^{-1}} + \frac{1}{1 - \rho} \right) \sqrt{\frac{1}{2 \pi_{\min}} \sum_{s \in \gS} \pi^{(k)}(s) \KL\left( p^{(k)}(\cdot | s), p^{(k')}(\cdot | s) \right)} \\
        &= \left( \ceil{\log_\rho M^{-1}} + \frac{1}{1 - \rho} \right) \sqrt{\frac{1}{2 \pi_{\min}} \gD_\pi^{(k,k')}}.
    \end{align*}
    As $\frac{2}{\log\frac{e}{2}} < 7$, we finally have that
    \begin{align*}
        \Delta_\mW^2 &\leq 7 \min_{k \neq k'} \left\{ p_{\max} \gD_\pi^{(k,k')} + \left( \ceil{\log_\rho M^{-1}} + \frac{1}{1 - \rho} \right) \sqrt{\frac{1}{2 \pi_{\min}} \gD_\pi^{(k,k')}} \right\} \\
        &= 7 \left( p_{\max} \gD_\pi + \left( \ceil{\log_\rho M^{-1}} + \frac{1}{1 - \rho} \right) \sqrt{\frac{1}{2 \pi_{\min}} \gD_\pi} \right),
    \end{align*}
    where the last equality follows from the fact that $z \mapsto c_1 z + c_2 \sqrt{z}$ is strictly increasing for $z > 0$ and $c_1, c_2 > 0$.
\end{proof}

\subsection{\texorpdfstring{(c) Relationship between $\Delta_{\mathbf{W}}$ and \citet{kausik2023mixture}'s $\alpha, \Delta$}{(c) Relationship between our and Kausik et al. (2023)'s L2-based Separations}}
Similarly, for each $s \in \gS,$
\begin{align*}
        &\bignorm{\sqrt{\pi^{(k)}(s)} p^{(k)}(\cdot | s) - \sqrt{\pi^{(k')}(s)} p^{(k')}(\cdot | s)}_2^2 \\
        &= \bignorm{\sqrt{\pi^{(k)}(s)} \left( p^{(k)}(\cdot | s) - p^{(k')}(\cdot | s) \right) + \left(\sqrt{\pi^{(k)}(s)} - \sqrt{\pi^{(k')}(s)}\right) p^{(k')}(\cdot | s)}_2^2 \\
        &\geq \frac{1}{2} \pi^{(k)}(s) \bignorm{p^{(k)}(\cdot | s) - p^{(k')}(\cdot | s)}_2^2 - \left( \sqrt{\pi^{(k)}(s)} - \sqrt{\pi^{(k')}(s)} \right)^2 \bignorm{p^{(k')}(\cdot | s)}_2^2 \\
        &\geq \frac{1}{2} \pi^{(k)}(s) \bignorm{p^{(k)}(\cdot | s) - p^{(k')}(\cdot | s)}_2^2 - \pi^{(k)}(s) \left( 1 - \sqrt{\frac{\pi^{(k')}(s)}{\pi^{(k)}(s)}} \right)^2 \\
        &\geq \frac{1}{2} \pi^{(k)}(s) \bignorm{p^{(k)}(\cdot | s) - p^{(k')}(\cdot | s)}_2^2 - \pi^{(k)}(s) \left( (\sqrt{\eta_\pi} - 1)^2 \vee \left( 1 - \frac{1}{\sqrt{\eta_\pi}} \right)^2 \right). \tag{Assumption~\ref{assumption:eta-regularity}}
    \end{align*}
    Summing over $s \in \gS$ and taking the min over $k \neq k'$, we have:
    \begin{align}
        \Delta_\mW^2 &\geq \frac{1}{2} \sum_{s \in \gS} \pi^{(k)}(s) \bignorm{p^{(k)}(\cdot | s) - p^{(k')}(\cdot | s)}_2^2 - \left( (\sqrt{\eta_\pi} - 1)^2 \vee \left( 1 - \frac{1}{\sqrt{\eta_\pi}} \right)^2 \right) \nonumber \\
        &\geq \frac{1}{2} \alpha \Delta^2 - \left( (\sqrt{\eta_\pi} - 1)^2 \vee \left( 1 - \frac{1}{\sqrt{\eta_\pi}} \right)^2 \right).
    \end{align}
    \qed

\newpage
\section{\texorpdfstring{PROOF OF PROPOSITION~\ref{prop:separation-Delta} -- HARD INSTANCE WITH $S$-DEPENDENT SEPARATION ACROSS GAPS}{PROOF OF PROPOSITION 5.2 -- HARD INSTANCE WITH S-DEPENDENT SEPARATION ACROSS GAPS}}
\label{app:gaps2}
Fix $S\ge2$, $K\ge1$, horizon $H$, and $0 < \zeta < 1/S$.
For simplicity, we assume that $S \geq K$, and we also fix the initial distribution to be uniform, i.e., $\mu = \Unif([S])$ with $\gS = [S]$.

\subsection{Construction of Cyclic-Bump MMC}
\label{app:cylic-bump}
We now introduce our construction, which we refer to \textbf{cyclic-bump MMC}.
Intuitively, this is done by starting from an everywhere-uniform chain, then repeatedly adding doubly-stochastic perturbations via a small circulant flow of magnitude $\zeta$.

Rigorously speaking, we first define the uniform chain's transition kernel $U = \frac{1}{S} \mathbf{1} \mathbf{1}^\top$, where $\mathbf{1}$ is an $S$-dimensional all-$1$ column vector.
We set the first chain to be $P^{(1)} = U$.
Then, define the remaining $K - 1$ kernels as
\begin{equation}
    P^{(k+1)} := U + \zeta \left( C^k - C^{k + 1} \right),
\end{equation}
where the $k$-step cyclic right-shift permutation matrix $C^k$ is defined as
\begin{equation}
    \left( C^k \right)_{i,j} := \indicator[j \equiv i + k \mod S].
\end{equation}

\begin{claim}
    For all $k \geq 1$, $P^{(k)}$ defines an ergodic Markov chain with uniform stationary distribution: $\pi^{(k)}(\cdot | s) = \mathrm{Unif}([S])$ for all $s \in [S]$.
\end{claim}
\begin{proof}
    The statement is trivial for $k = 1$, and so we focus on $k \geq 2.$

    \emph{$P^{(k)}$ is well-defined, doubly-stochastic transition matrix.}
    Entrywise-positivity follows from our choice $\zeta < \frac{1}{S}.$
    As for doubly-stochasticity, first note that as each of $C^k$ and $C^{k+1}$ is a permutation matrix, every row and column sums to $1$.
    Thus, every row and column of $C^k - C^{k+1}$ sums to $0$.
    As $U$ is doubly-stochastic, so is $P^{(k)}$.
    
    \emph{Uniform Stationarity.}
    Let $u^\top=\tfrac1S\mathbf 1^\top$ be the row-vector describing the uniform distribution over $[S]$.
    Because $u^\top$ is invariant under any permutation, we have that $u^\top C^k = u^\top$.
    Hence
    \begin{equation*}
    u^\top P^{(k)} \;=\; u^\top U + \zeta \big(u^\top C^k - u^\top C^{k+1}\big)
    \;=\; u^\top + \zeta (u^\top-u^\top) \;=\; u^\top,
    \end{equation*}
    so $\pi^{(k)} = u$ is stationary.
    Finally, as $P^{(k)}$ is entrywise strictly positive, it defines an ergodic Markov chain, hence $\pi^{(k)}$ (as defined previously) is indeed the unique stationary distribution of $P^{(k)}$.
\end{proof}

\subsection{Explicit Computations of the Gaps}
With the above construction, we now compute each gap explicitly.

\paragraph{Computing $\alpha \Delta^2$.}
First, we know that $\alpha = \frac{1}{S}$, and so it suffices to compute $\Delta^2$, which is the (maximum) $\ell_2$-gap of the outgoing transitions.
From direct computations, we have that for $k, k' \geq 2$,
\begin{equation*}
    \bignorm{P^{(k)}(\cdot | s) - P^{(1)}(\cdot | s)}_2^2 = 2\zeta^2
\end{equation*}
and
\begin{equation*}
    \bignorm{P^{(k)}(\cdot | s) - P^{(k')}(\cdot | s)}_2^2 =
    \begin{cases}
        4 \zeta^2, & k' - k \neq \pm 1,\\
        6 \zeta^2, & k' - k = \pm 1.
    \end{cases}
\end{equation*}
All in all, we have that $\alpha \Delta^2 = \frac{6 \zeta^2}{S}$.

\paragraph{Computing $\Delta_\mW^2$.}
By definition and the above computations, we have that
\begin{align*}
    \Delta_\mW^2 &:= \min_{k \neq k'} \sum_{s \in \gS} \bignorm{\sqrt{\pi^{(k)}(s)} P^{(k)}(\cdot | s) - \sqrt{\pi^{(k')}(s)} P^{(k')}(\cdot | s)}_2^2 \\
    &= \frac{1}{S} \min_{k \neq k'} \sum_{s \in \gS} \bignorm{P^{(k)}(\cdot | s) - P^{(k')}(\cdot | s)}_2^2 \\
    &= \frac{1}{S} \sum_{s \in \gS} 2 \zeta^2
    = 2 \zeta^2.
\end{align*}

\paragraph{Computing $D_\pi$.}
From direct computations, we have that for $k, k' \geq 2$,
\begin{align*}
    \KL\left( P^{(k)}(\cdot | s), P^{(1)}(\cdot | s) \right) &= \frac{1}{S} \log\left( (1 + \zeta S) (1 - \zeta S) \right) + \zeta \log\frac{1 + \zeta S}{1 - \zeta S}, \\
    \KL\left( P^{(1)}(\cdot | s), P^{(k)}(\cdot | s) \right) &= \frac{1}{S} \log\frac{1}{(1 + \zeta S)(1 - \zeta S)},
\end{align*}
and
\begin{equation*}
    \KL\left( P^{(k)}(\cdot | s), P^{(k')}(\cdot | s) \right) =
    \begin{cases}
        \zeta \log \frac{(1 + \zeta S)^2}{1 - \zeta S}, & k' - k = 1,\\
        \zeta \log\frac{1 + \zeta S}{(1 - \zeta S)^2}, & k' - k = - 1,\\
        \zeta \log\frac{1 + \zeta S}{1 - \zeta S}, & k' - k \neq \pm 1.
    \end{cases}
\end{equation*}

As the above computation holds across all $s$'s and as $\pi^{(k)}$'s are all uniform, it suffices to take the minimum over the above five outcomes.
As $\zeta S \in (0, 1)$, it is clear that
\begin{equation*}
    \zeta \log \frac{(1 + \zeta S)^2}{1 - \zeta S} \vee \zeta \log\frac{1 + \zeta S}{(1 - \zeta S)^2} > \zeta \log\frac{1 + \zeta S}{1 - \zeta S}
    > \frac{1}{S} \log\left( (1 + \zeta S) (1 - \zeta S) \right) + \zeta \log\frac{1 + \zeta S}{1 - \zeta S},
\end{equation*}
and thus, it suffices to compare between $\KL\left( P^{(k)}(\cdot | s), P^{(1)}(\cdot | s) \right)$ and $\KL\left( P^{(1)}(\cdot | s), P^{(k)}(\cdot | s) \right)$.
The following claim, whose proof is deferred to the end of this subsection, shows that the former quantity is smaller than the latter:
\begin{claim}
\label{claim:compare-kl}
    $\KL\left( P^{(k)}(\cdot | s), P^{(1)}(\cdot | s) \right) \leq \KL\left( P^{(1)}(\cdot | s), P^{(k)}(\cdot | s) \right)$.
\end{claim}

All in all, we then have that
\begin{equation*}
    \gD_\pi = \frac{1}{S} \log\left( (1 + \zeta S) (1 - \zeta S) \right) + \zeta \log\frac{1 + \zeta S}{1 - \zeta S}.
\end{equation*}

\paragraph{Comparing everything.}
Let $\zeta = \frac{1}{2S}$ for simplicity.\footnote{One could do an asymptotic expansion of the KL-divergence for a more generic asymptotic comparison.}
Then, we have that
\begin{equation*}
    \alpha \Delta^2 = \frac{3}{2 S^3}, \quad \Delta_\mW^2 = \frac{1}{2 S^2}, \quad \gD_\pi = \frac{1}{S} \log\frac{3 \sqrt{3}}{4},
\end{equation*}
i.e., for this instance of clustering in MMC, we have that
\begin{equation}
    S^2 \alpha \Delta^2 \asymp S \Delta_\mW^2 \asymp \gD_\pi.
\end{equation}

\qed

\begin{proof}[Proof of Claim~\ref{claim:compare-kl}]
    Let $z = \zeta S \in (0, 1)$.
    Then, one can rewrite the inequality as
    \begin{equation*}
        g(z) \triangleq 2 \log(1 - z^2) + z \log\frac{1 + z}{1 - z} \leq 0, \quad \forall z \in (0, 1).
    \end{equation*}
    As $g(0) = 0$, it suffices to show that $g(\cdot)$ is monotone decreasing for $z \in (0, 1)$.
    To see this, we compute its first and second derivatives:
    \begin{equation*}
        g'(z) = \log\frac{1 + z}{1 - z} - \frac{2z}{1 - z^2}, \quad g''(z) = - \frac{4z^2}{(1 - z^2)^2} \leq 0.
    \end{equation*}
    As $g'(0) = 0$ as well, we are done.
\end{proof}

\subsection{Discussions}

\paragraph{Various Perspectives on Cyclic–Bump MMCs.}
Our construction is most naturally viewed as \emph{circulant perturbation} of an ergodic Markov kernel with a divergence-free \emph{circulation} on the transition graph: we add and cancel small cyclic shifts so that row and column sums remain unchanged and the stationary distribution stays the same, directly connecting to classical analyses of circulant/group-based chains and their spectra~\citep{diaconis1988,markovmixing}.
Geometrically, the perturbation is a small, \emph{local} move inside the Birkhoff--von Neumann polytope\footnote{This is the convex polytope generated by permutation matrices~\citep{birkhoff1946Tres}.}: starting from the point $U$, we stay within the affine subspace of doubly-stochastic matrices by moving along a divergence-free direction given by a scaled \emph{signed difference of permutation vertices}, which is interesting in its own right from a polyhedral viewpoint.
The construction presented here shares principles with irreversible and lifted Markov Chain Monte Carlo (MCMC) methods. Specifically, these methods introduce a cycle flow to modify the chain's spectral properties and accelerate mixing, all while preserving the target distribution~\citep{chen1999lifting,diaconis1988}. Our work diverges from this literature in a crucial way: whereas lifting requires expanding to a larger state space~\citep{chen1999lifting}, our construction modifies the chain while preserving the original state space $[S]$.

\paragraph{What if $K> S$?}
Our construction assumed $K\le S$ so that single-bump shifts can be chosen distinct; when $K>S$,\footnote{Here, we denote $C^k := C^{k \mod{S}}$ for $k > S$.} indices coincide modulo $S$ (e.g., $1$ and $K + 1 \equiv 1 \mod{K}$), and the corresponding single-bump components become identical (hence non-identifiable) if they use the same amplitude.
To retain distinctness while preserving the uniform stationary distribution and ergodicity, one can: (i) vary amplitudes across components, taking
$P^{(k)} = U + \zeta_k \big(C^{k} - C^{k+1}\big)$ with $\zeta_k \neq \zeta_{k'}$ and $0 < \zeta_k < 1/S$; (ii) use multi-bump combinations
$P^{(k)} = U + \sum_{r \in R_k} \zeta_{k,r}\big(C^{k+r}-C^{k+r+1}\big)$
with $\sum_r \zeta_{k,r} < 1/S$ and $R_k \subset [K]$; or (iii) adopt a general small circulant combination
$P^{(k)} = U + \sum_{m=1}^{S-1} a_{k,m} C^{m}$ with $\sum_m a_{k,m}=0$ and $\sum_{m=1}^{S-1}\lvert a_{k,m}\rvert < 1/S$ (a simple sufficient positivity condition). Option (iii) has a natural Fourier/representation-theoretic interpretation on the cyclic group (circulant kernels diagonalized by the discrete Fourier transform); see \citet{diaconis1988} for background.

\paragraph{Why $D_\pi$ is Large while $\Delta_\mW^2$ is Small?}
From an information–geometric perspective, the gap stems from the \emph{Fisher information metric} underlying KL divergence~\citep{infogeom,nielsen2020infogeom}.
This metric reweights directions by the inverse probabilities, amplifying changes along low-mass coordinates.
Rigorously, let $g_u$ be the Fisher information metric on the probability simplex at the uniform distribution $u=(1/S,\ldots,1/S)$, which reduces to $g_u = \mathrm{diag}(1/u) = S \mI$ on the tangent space -- hence locally $\mathrm{KL}(p\|u)\approx \tfrac12 (p-u)^\top g_u (p-u)=\tfrac{S}{2}\|p-u\|_2^2$.
The cyclic bump changes exactly two coordinates per row by $\pm\varepsilon$, which is a tiny \emph{absolute} move -- hence $\Delta_\mW^2=\Theta(\varepsilon^2)$ -- but a \emph{relative} change of order $S$ larger, i.e., $\gD_\pi = \Theta(\varepsilon S)$.
This is why the KL-based separation satisfies $D_\pi\asymp S \Delta_\mW^2$ compared to the $\ell_2$-based separation $\Delta_\mW^2$.

\newpage
\section{DEFERRED EXPERIMENTAL DETAILS \& ADDITIONAL ABLATIONS}
\label{app:experiments}

\subsection{\texorpdfstring{Deferred Experimental Details from Section~\ref{sec:experiments}}{Deferred Experimental Details from Section 6}}
\label{app:exp-details}
For the algorithms of \citet{kausik2023mixture}, the number next to `th' is the threshold level that is used to output an initial estimate of the clustering from pairwise distances; see line 16 of Algorithm 2 of \citet{kausik2023mixture}.
The ``right'' choice of the threshold level explicitly depends on the separation gap (see \citet[Theorem 1]{kausik2023mixture}), and thus, it is inherently heuristic in the absence of such knowledge.
For this reason, we arbitrarily sweep over $\{10^{-5}, 10^{-4}, 10^{-3}\}$.
We vary $H \in \{10, 20, \cdots, 90, 100, 200, \cdots, 500\}$ and $T \in \{100, 200, \cdots, 600\}$.
For statistical significance, we report results with $\delta = 0.05$, shading the 95\% bootstrapped confidence intervals~\citep{efron1979bootstrapping,diciccio1996bootstrapping} over 30 independent repeats for each configuration.
Each repeat corresponds to a fresh random draw of $T$ trajectories.
The codes are available in our GitHub repository.\footnote{\url{https://github.com/nick-jhlee/optimal-mixture-markov}}

\subsection{Ablation \#1. Impact of Number of Iterations for Stage II}
\label{app:ablation-1}
We utilize the same instance of MMC as in Section~\ref{sec:experiments}.
Figure~\ref{fig:ablation-1} examines how the clustering error evolves with additional EM iterations in Stage II.
While our theory only requires a single EM step, we observe consistent empirical gains from running more iterations.
The benefit is most pronounced when $H$ is small or $T$ is large, where each additional iteration pushes the error curve closer to the oracle benchmark.
The improvement largely saturates after about $5$–$6$ iterations, suggesting diminishing returns beyond this point.
Importantly, all Stage~II variants still outperform Stage~I alone, underscoring the added value of EM refinement.

\begin{figure}[h]
    \centering
    \includegraphics[width=\linewidth]{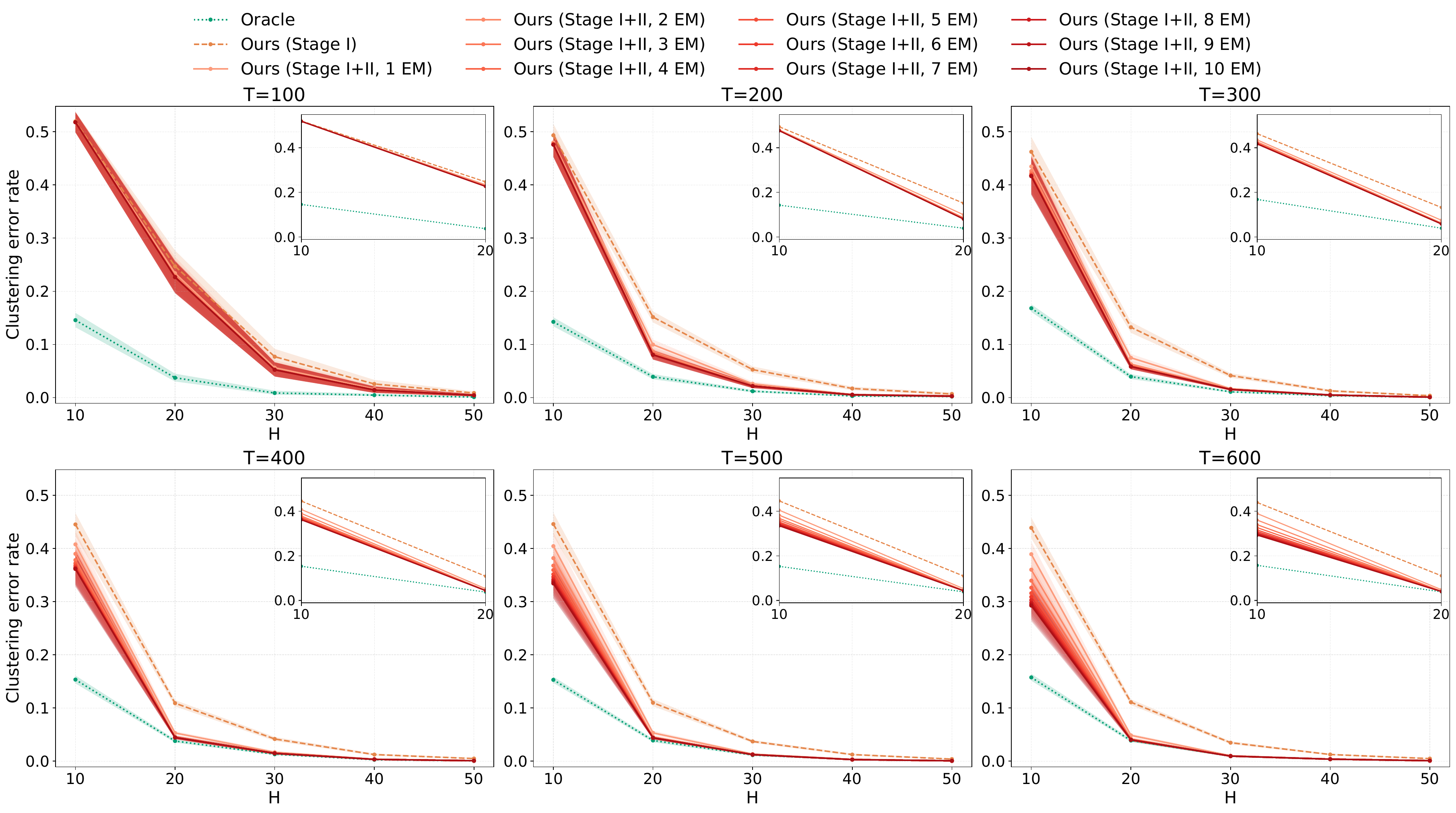}
    \caption{\textbf{(Ablation \#1)} Clustering error on the synthetic MMC instance across EM iterations (up to 10). Panels vary across $T \in \{100, 200, \cdots, 600\}$; the x-axis varies across $H \in \{10, 20, \cdots, 50\}$. Darker red curves indicate more EM iterations; ``Oracle'' and ``Stage I'' are shown for reference.
    }
    \label{fig:ablation-1}
\end{figure}

\subsection{\texorpdfstring{Ablation \#2. Impact of $S$ and $T$ on Performance and Runtime}{Ablation \#2. Impact of S and T on Performance and Runtime}}
\label{app:ablation-2}
To analyze the computational complexity of our two-stage algorithm, we measure wall-clock runtime as a function of the state-space size $S$ and the number of trajectories $T$.
We decompose the total runtime into four components: (i)~embedding construction (line~4 of Algorithm~\ref{alg:initial-spectral}), (ii)~singular value decomposition (SVD) (line~5), (iii)~$K$-means clustering (lines~6–7), and (iv)~likelihood reassignment (Algorithm~\ref{alg:likelihood}).
For the experiment varying $S$, we fix $T=400$ and $H=500$, and sweep $S \in \{10,20,30,40,50,60\}$.
For the experiment varying $T$, we fix $S=40$ and $H=500$, and sweep $T \in \{100,200,300,400,500,600\}$.
In both settings, we run $10$ EM iterations in Stage~II and average over $30$ independent trials.

Figure~\ref{fig:ablation-2} below summarizes the results.
The top row shows that clustering performance degrades as $S$ increases (left), due to the decreasing separation $\zeta = 0.9/S$ in our MMC construction, whereas performance improves as $T$ increases (right), as expected.
The middle row reveals that SVD dominates the runtime for large $S$ (left), while likelihood reassignment dominates for large $T$ (right).
The bottom row quantifies the overall scaling: total runtime grows quadratically in $S$ with $R^2 \approx 0.9583$, confirming the $\mathcal{O}(TS^2)$ SVD cost, and linearly in $T$ with $R^2 \approx 0.9957$, consistent with the $\mathcal{O}(T)$ scaling of all algorithmic steps.

\begin{figure}[h]
    \centering
    \includegraphics[width=0.78\linewidth]{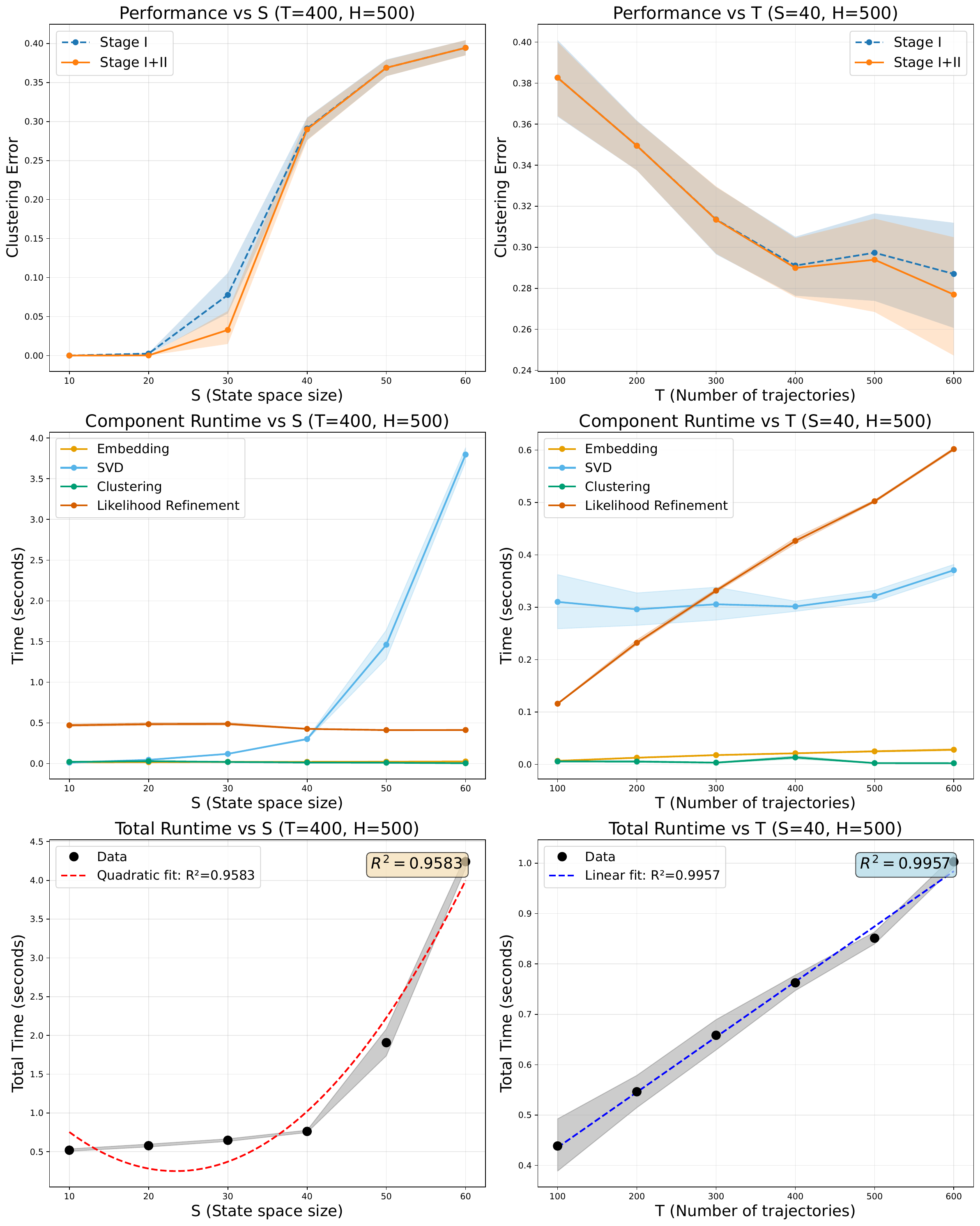}
    \caption{\textbf{(Ablation \#2)} Impact of $S$ and $T$ on the resulting clustering error and runtime.
    }
    \label{fig:ablation-2}
\end{figure}

\subsection{\texorpdfstring{Ablation \#3. Our Algorithm with Unknown $K$}{Ablation \#3. Our Algorithm with Unknown K}}
\label{app:ablation-3}
Again, in the same instance of MMC as in Section~\ref{sec:experiments}, we evaluate the performance of our method when the number of clusters $K$ is unknown a priori.
Recall that Algorithm~\ref{alg:initial-spectral} uses two fixed constants ($64$ for adaptive singular value thresholding at line~6, and $32$ for estimating $K$ at lines~10–14).
For this ablation, we replace these fixed values with tunable hyperparameters $c_1$ and $c_2$, respectively.
In practice, careful tuning of $(c_1, c_2)$ is essential, since the performance of the unknown-$K$ variant is sensitive to the empirical tightness of the concentration bounds that these parameters effectively calibrate.
Within a reasonable range, we believe that adjusting $(c_1, c_2)$ does not alter the theoretical validity of our guarantee; the only difference is that the failure probability scales as $\Theta(\delta)$ rather than exactly $\delta$.

We sweep over $c_1 \in \{10^{-3}, 10^{-2}, 10^{-1}\}$ and $c_2 \in \{10^{-3}, 5\!\cdot\! 10^{-3}, 10^{-2}, 5\!\cdot\! 10^{-2}, 10^{-1}\}$, and extend the trajectory length to
$H \in \{100,200,\dots,1000\} \cup \{1000,2000,\dots,20000\}$ in order to probe the large-$H$ regime where theory predicts consistent recovery of $K$.
We fix $T=100$, and Stage~II is fixed to run $10$ EM iterations.
The known-$K$ variant achieves perfect clustering and is therefore omitted from comparison.

Figure~\ref{fig:ablation-3} summarizes the results.
For small $H$ (top row), the unknown-$K$ variant often fails to recover clusters reliably, especially for moderate choices $c_1 \in \{10^{-2}, 10^{-1}\}$.
In contrast, a more aggressive setting $c_1=10^{-3}$ succeeds even in the $100$’s regime across all considered $c_2$ values, indicating that performance is substantially more sensitive to $c_1$ than to $c_2$.
In the large-$H$ regime (bottom row), $c_1=10^{-2}$ also succeeds consistently across $c_2$, whereas $c_1=10^{-1}$ continues to fail, again underscoring the high sensitivity of empirical performance to the choice of $c_1$.
While there remains some dependence on $c_2$—as reflected in differing convergence speeds with increasing $H$—the overall recovery trends are qualitatively similar.

\begin{figure}[h]
    \centering
    \includegraphics[width=\linewidth]{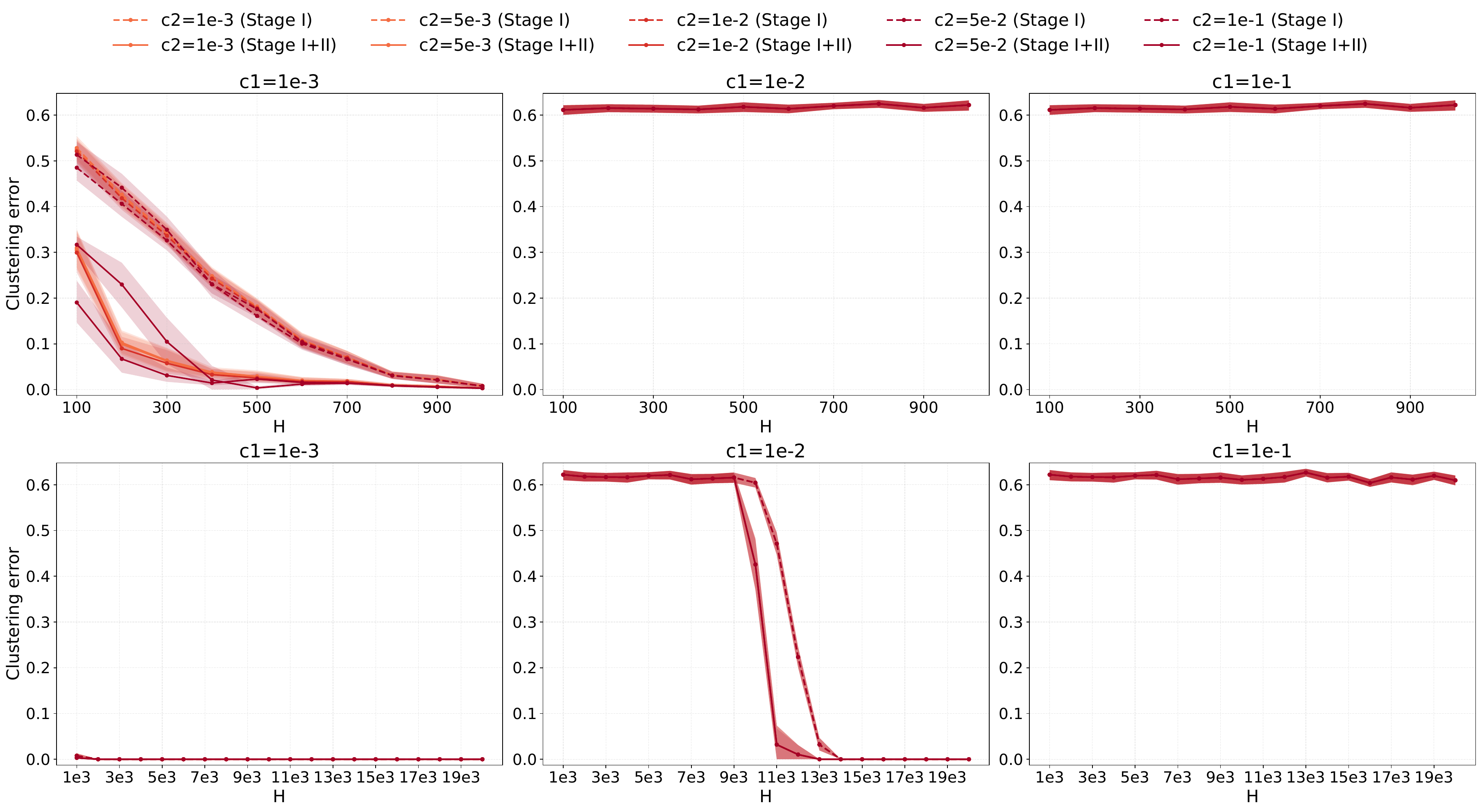}
    \caption{\textbf{(Ablation \#3)} Clustering error of the unknown-$K$ variant on the synthetic MMC instance across different $(c_1, c_2)$ settings.
    Panels vary across $c_1 \in \{10^{-3}, 10^{-2}, 10^{-1}\}$ (columns) and trajectory length $H$ (rows: top for $H \leq 1000$, bottom for $H \geq 1000$).
    Curves indicate different $c_2$ values; Stage I and Stage I+II are shown separately.}
    \label{fig:ablation-3}
\end{figure}

\subsection{\texorpdfstring{Ablation \#4. Impact of $\gamma_\ps$}{Ablation \#4. Impact of Pseudo-Spectral Gap}}
\label{app:ablation-4}
To demonstrate that the dependence on the pseudo-spectral gap $\gamma_\ps$ in our final performance guarantee (Theorem~\ref{thm:likelihood}) is fundamental (at least with our current algorithm design), we construct a sequence of MMC instances where the information-theoretic gap $\gD_\pi$ is strictly constant, but $\gamma_\ps$ varies.

We propose the ``Signal in a Maze'' construction. 
Let $S = 10$ and $K = 2$. We denote the state space as $\gS = \{s_1, s_2, \dots, s_{10}\}$. We partition this space into a signal region $\gS_{\text{diff}} = \{s_1, s_2\}$ and a mixing maze $\gS_{\text{mix}} = \{s_3, \dots, s_{10}\}$. 
Conceptually, the \textit{signal} region contains the only states where the two underlying Markov chains exhibit distinct transition behaviors. The \textit{maze} acts as a sequence of identical, slow-mixing states that artificially bottlenecks the traversal of the chain without altering the stationary distribution. Finally, a small \textit{leakage} parameter $\epsilon$ dictates the probability of transitioning between the signal region and the maze, ensuring the chain remains ergodic while compartmentalizing the signal.

We fix the leakage parameter $\epsilon = 0.05$ and distinct transition probabilities $p_1 = 0.4$ and $p_2 = 0.3$. 
We parameterize the sequence of MMCs by a bottleneck variable $q \in (0, 0.5)$, which dictates the traversal speed through the maze.
For each $k \in \{1, 2\}$ and given $q \in (0, 0.5)$, we define the transition matrix $\mP^{(k)}_q$ using a block-matrix structure:
\begin{equation*}
    \mP^{(k)}_q = \begin{bmatrix}
        \mB^{(k)} & \mE^\top \\
        \mE & \mT_q
    \end{bmatrix},
\end{equation*}
where the constituent blocks are defined as follows:
\begin{enumerate}
    \item \textbf{Signal Block ($\mB^{(k)} \in \sR^{2 \times 2}$):} Governs transitions within $\gS_{\text{diff}}$ and dictates the distinguishing behavior between the chains:
    \begin{equation*}
        \mB^{(k)} = \begin{bmatrix}
            1 - p_k - \epsilon & p_k \\
            p_k & 1 - p_k - \epsilon
        \end{bmatrix}.
    \end{equation*}
    \item \textbf{Leakage Block ($\mE \in \sR^{8 \times 2}$):} Connects the signal region to the ends of the maze. All entries are zero except for the entries mapping $s_3 \leftarrow s_1$ and $s_{10} \leftarrow s_2$, namely $\mE_{1,1} = \epsilon$ and $\mE_{8,2} = \epsilon$.
    
    \item \textbf{Maze Block ($\mT_q \in \sR^{8 \times 8}$):} A symmetric tridiagonal matrix representing the path graph of $\gS_{\text{mix}}$, independent of $k$:
    \begin{equation*}
        \mT_q = \text{tridiag}(q, 1-2q, q) + \text{diag}(-\epsilon, 0, \dots, 0, -\epsilon).
    \end{equation*}
    Explicitly, the main diagonal is $(1-q-\epsilon, 1-2q, \dots, 1-2q, 1-q-\epsilon)$ and the first off-diagonals are all $q$.
\end{enumerate}

By construction, $\mP^{(k)}_q$ is symmetric and doubly stochastic. Thus, the stationary distribution is strictly uniform: $\pi^{(1)} = \pi^{(2)} = \pi = [1/10, \dots, 1/10]^\top$, which is entirely independent of $q$.
Consequently, the stationary-weighted KL divergence evaluates to:
\begin{equation*}
    \gD_\pi = \frac{1}{10} \KL\left(\mP^{(1)}_q(s_1, \cdot), \mP^{(2)}_q(s_1, \cdot)\right) + \frac{1}{10} \KL\left(\mP^{(1)}_q(s_2, \cdot), \mP^{(2)}_q(s_2, \cdot)\right).
\end{equation*}
Because the bottleneck parameter $q$ only appears in the $\mT_q$ block (which is identical across $k=1, 2$), $\gD_\pi$ remains exactly constant for all $q$. 

However, as $q \to 0$, the pseudo-spectral gap $\gamma_\ps$ strictly shrinks, requiring trajectories of increasingly longer length $H$ to traverse the maze and observe the signal states.
For $q \in \{0.01, 0.03, 0.05, 0.07, 0.09, 0.1, 0.2\}$ (which is the range of $q$'s at which we evaluate our algorithm), the pseudo-spectral gaps are numerically computed as:
\begin{equation*}
    \gamma_\ps \in \{ 0.008399, 0.023994, 0.037832, 0.049807, 0.059931, 0.064334, 0.090893 \}.
\end{equation*} 

We evaluate our algorithm on the above instance over $H \in \{100, 200, \dots, 2000\}$, with 30 repeats per configuration. As shown in Figure~\ref{fig:ablation-4}, the rate of error decay against $H$ is slower with lower $\gamma_\ps$ (i.e., higher mixing time).

\begin{figure}[h]
    \centering
    \includegraphics[width=\linewidth]{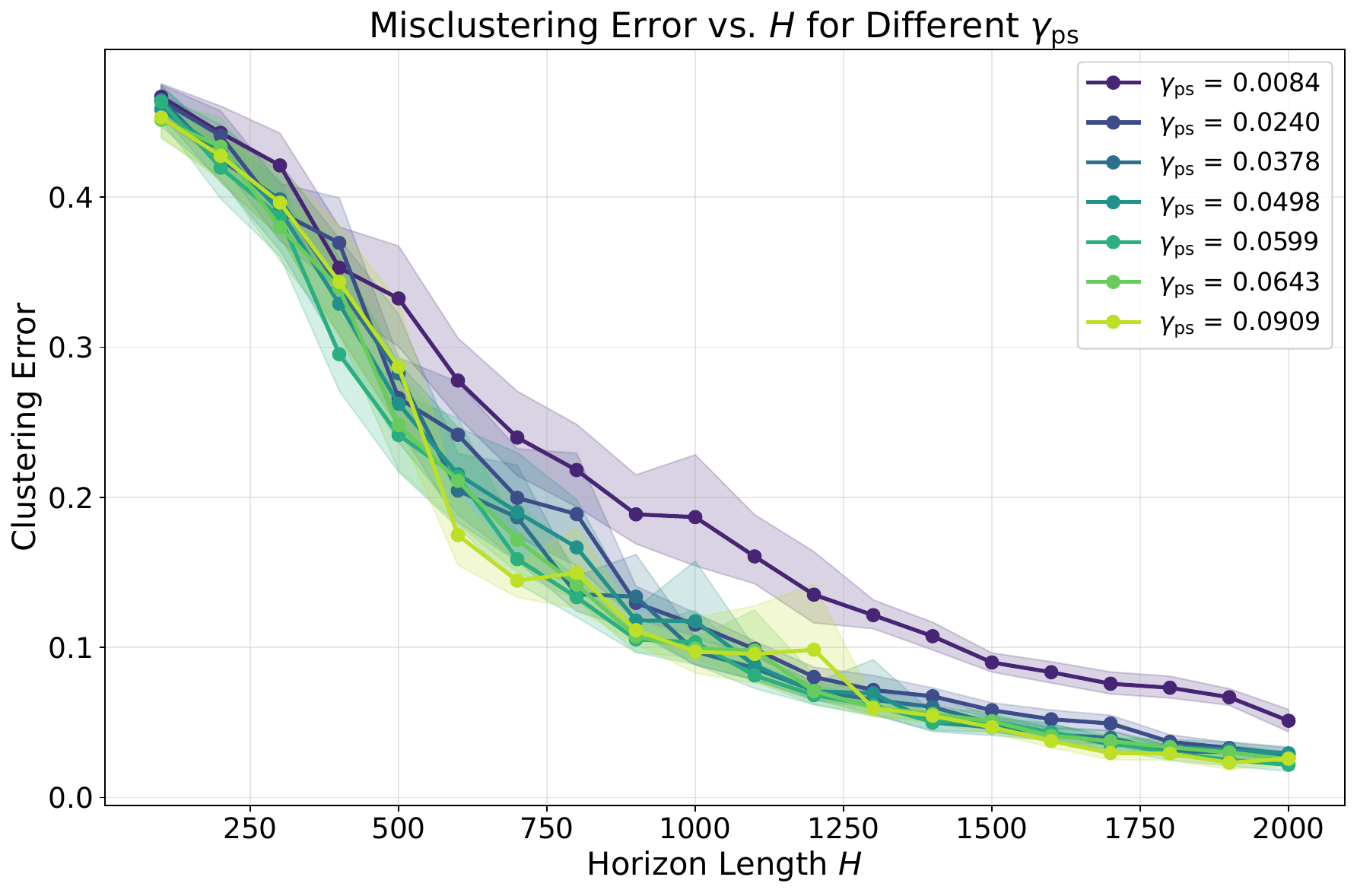}
    \caption{\textbf{(Ablation \#4)} Clustering error on another synthetic MMC instances with varying $\gamma_\ps$'s, over $H \in \{100, 200, \cdots, 2000\}$.}
    \label{fig:ablation-4}
\end{figure}

\subsection{Additional Experiment with Last.fm 1K}
\label{app:experiments-real}

\paragraph{Preprocessing.}
For this dataset, we follow the same pre-processing steps of \citet{kausik2023mixture}, as provided in their GitHub repository.\footnote{\url{https://github.com/hetankevin/mdpmix}}
We explain the preprocessing pipeline here for completeness.
We use the Last.fm 1K listening logs\footnote{\url{http://mtg.upf.edu/static/datasets/last.fm/lastfm-dataset-1K.tar.gz}} (\texttt{userid-timestamp-artid-artname-traid-traname.tsv}) and the Lastfm-ArtistTags2007 files\footnote{\url{https://web.archive.org/web/20110827230919/http://static.echonest.com/Lastfm-ArtistTags2007.tar.gz}} \texttt{tags.txt} and \texttt{ArtistTags.dat}.
The state space is the set of the top 100 tags (genres) from \texttt{tags.txt}. 
Each artist is mapped to exactly one genre by retaining, for each artist, the tag with the highest \texttt{rawtagcount} among the top 100 tags.
We retain the top 10 users by the number of listening events, merge their events with artist tags, and sort the results by user and timestamp.
We keep only transitions where the genre \emph{changes} (and the user is unchanged), so consecutive plays of the same genre are collapsed into a single state.
From each user, we then form non-overlapping trajectories of length $H$ along this sequence, yielding $75$ trajectories per user; the true cluster label of a trajectory is the user index.
Trajectories are optionally shuffled with a fixed seed for evaluation.
This yields $K = 10$ clusters, $S = 100$ states, horizon $H$, and total number of trajectories $T = 750$ (with next-step trajectories) for the mixture-of-Markov-chains experiment.

We vary the horizon $H \in \{10, 20, \ldots, 100\}$, repeat each setting 30 times with different seeds, and report permutation-invariant cluster error rate with bootstrap 95\% confidence intervals.
We compare our method (Stage I: spectral clustering only; Stage I+II: spectral clustering followed by $10$ EM iterations) to the algorithms of \citet{kausik2023mixture} (clustering-only and clustering + $10$ EM iterations), with thresholds sweeped over $\{10^{-6}, 5 \cdot 10^{-5}, 10^{-5}\}$.
For all algorithms, we use Laplace smoothing with $\lambda = 0.1$ for numerical stability.

\paragraph{Results.}
Figure~\ref{fig:lastfm} below shows the overall results.
We first note that we reproduce the general trend of the algorithms reported in \citet{kausik2023mixture}, as shown in their Figures 4 and 5.
Secondly, note that our algorithms, both Stage I and I+II, significantly outperform those of \citet{kausik2023mixture}.
This demonstrates that our algorithm is also practically efficient on real-world datasets.
Importantly, our algorithm (with Laplace smoothing) is shown to perform beyond the theoretically predicted regime, i.e., when the regularity (Assumption~\ref{assumption:eta-regularity}) fails and when $H$ is small.

\begin{figure}[h]
    \centering
    \includegraphics[width=\linewidth]{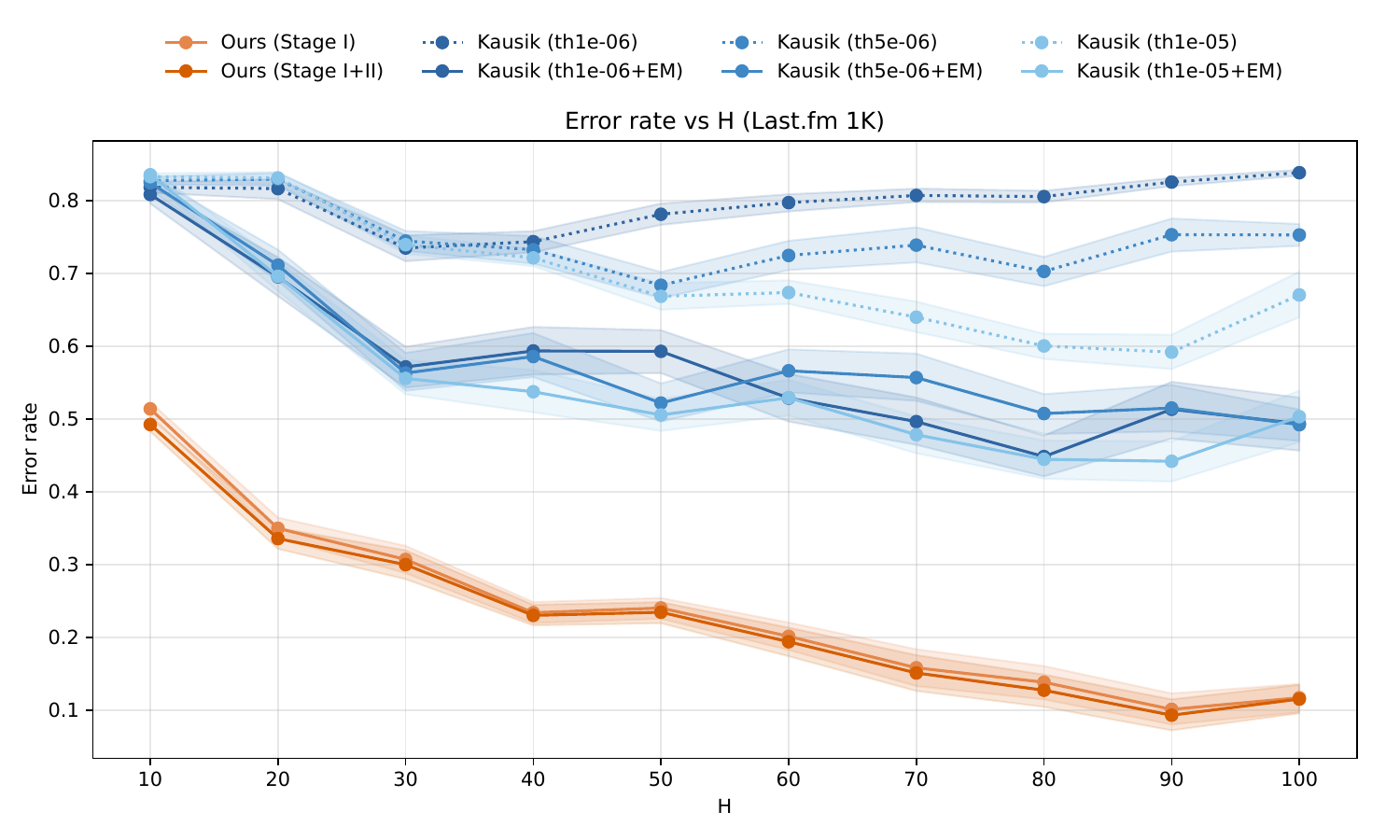}
    \caption{\textbf{(Last.fm 1K)} Clustering errors for Last.fm 1K, over $H \in \{10, 20, \cdots, 100\}$.}
    \label{fig:lastfm}
\end{figure}

\clearpage
\section{\texorpdfstring{RELATIONSHIPS BETWEEN OUR $L$-EMBEDDING AND OTHER SPECTRAL EMBEDDINGS}{RELATIONSHIPS BETWEEN OUR L-EMBEDDING AND OTHER SPECTRAL EMBEDDINGS}}
\label{app:comparisons}

\paragraph{Diffusion Map~\citep{coifman2006diffusion}.}
The objective of a diffusion map fundamentally differs from our approach. Let $\gX$ be a measure space endowed with a kernel $k : \gX \times \gX \rightarrow \sR_{\geq 0}$.
Diffusion maps provide a meaningful Euclidean embedding of \emph{data points} $\gD \subseteq \gX$ via a reversible Markov process (random walk) over $\gX$, induced by the kernel $k$ which encodes prior knowledge of the local geometry.
Thus, its purpose is not to embed a specific Markov chain, but rather to map a measure space onto a Euclidean space using a random walk.

\paragraph{(Weighted) Laplacian~\citep[Chapter 1.5]{chung1997spectral}.}
Similarly, the graph Laplacian is tailored for the spectral analysis of weighted undirected graphs.
While it offers an elegant probabilistic interpretation of eigenvalues and eigenvectors, it is not directly applicable as an embedding for general, potentially non-reversible Markov chains.

\paragraph{Symmetrized Form~\citep[Chapter 12]{markovmixing}.}
Recall that the symmetrized form is $\mD_\pi^{1/2} \mP \mD_\pi^{1/2}$, where $\mD_\pi = \mathrm{diag}(\pi)$.
The left factor, $\mD_\pi^{1/2} \mP$, when vectorized, corresponds precisely to our $L$-embedding.
To the best of our knowledge, however, no prior work has analyzed the properties of the mapping $S : \gM \mapsto \mD_\pi^{1/2} \mP \mD_\pi^{1/2}$ as a Euclidean embedding for Markov chains. 
Specifically, its injectivity and concentration properties remain systematically unexplored.
In contrast, for our $L$-embedding, we explicitly establish injectivity (Proposition~\ref{prop:L}) and derive sharp concentration bounds (Lemma~\ref{lem:concentration-W}), both of which are critical for our Stage I analysis (Theorem~\ref{thm:initial-spectral}).

\paragraph{Doublet Frequency~\citep{vidyasagar2014markov,wolfer2021markov}.}
Defined as $\mD_\pi \mP$, the doublet frequency closely resembles our $L$-embedding.
It satisfies injectivity up to the initial distribution and exhibits similarly well-behaved concentration properties.
However, its induced $\ell_2$-gap scales differently and does not align as naturally with the information-theoretic gap $\gD_\pi$ or \citet{kausik2023mixture}'s notion og gap $\alpha \Delta^2$, whereas our $L$-embedding explicitly facilitates these direct comparisons via $\Delta_\mW$ (Section~\ref{sec:discussions}\textbf{(3)}).

\newpage
\section{\texorpdfstring{RANK OF $\mW$ MAY BE MUCH SMALLER THAN THE NUMBER OF CLUSTERS $K$}{RANK OF W MAY BE MUCH SMALLER THAN THE NUMBER OF CLUSTERS K}}
\label{app:rank-K}
Here, we show this via explicit constructions of $K + 1$ transition matrices whose induced $\mW$ is of rank $2$.

Let $P, P'$ be two distinct ergodic transition matrices with a same stationary distribution $\pi$.
Define the other transition matrices as
\begin{equation}
    P^{(k)} = \left(1 - \frac{k - 1}{K} \right) P + \frac{k - 1}{K} P', \quad k \in \{1, 2, \cdots, K+1\}.
\end{equation}
As $\pi$ is also a stationary distribution of $P^{(k)}$ and ergodicity is preserved under convex combination, it can be easily seen that every row of $\mW$ is a convex combination of the $1$-st and $(K+1)$-th row.
In other words, $\rank(\mW) = 2$, yet there are $K + 1$ distinct transition matrices.

Due to its simple form, we can also compute $\Delta_\mW^2$ in a closed form:
\begin{align*}
    \Delta_\mW^2 &= \min_{k \neq k'} \sum_{s \in \gS} \pi(s) \left\lVert p^{(k)}(\cdot | s) - p^{(k')}(\cdot | s) \right\rVert_2^2 \\
    &= \min_{k \neq k'} \sum_{s \in \gS} \pi(s) \left\lVert \frac{k' - k}{K} p(\cdot | s) + \frac{k - k'}{K} p'(\cdot | s) \right\rVert_2^2 \\
    &= \sum_{s \in \gS} \pi(s) \left\lVert p(\cdot | s) - p'(\cdot | s) \right\rVert_2^2 \min_{k \neq k'}  \left( \frac{k' - k}{K}\right)^2 \\
    &= \frac{1}{K^2} \sum_{s \in \gS} \pi(s) \left\lVert p(\cdot | s) - p'(\cdot | s) \right\rVert_2^2 > 0.
\end{align*}

\newpage
\section{\texorpdfstring{RELAXING THE KNOWLEDGE OF $\gamma_\ps$}{RELAXING THE KNOWLEDGE OF PSEUDO-SPECTRAL GAP}}
\label{app:gamma-ps}
We first recall the na\"{i}ve idea briefly sketched in Section~\ref{sec:discussions}\textbf{(4)}:
\begin{enumerate}
    \item For each $t \in [T]$, compute the estimator $\hat{\gamma}_{\ps,t}$ of \citet{wolfer2024mixing} that estimates $\hat{\gamma}_\ps^{(f(t))}$ while satisfying the following:
    \begin{equation}
    \label{eqn:gamma-ps}
        \mathbb{P}\left(\left|\hat{\gamma}_{\ps,t}-\gamma_\ps^{(f(t))}\right| \leq \frac{1}{2} \gamma_\ps^{(f(t))}\right) \geq  1 - \frac{\delta}{T}.
    \end{equation}
    \item Output $\hat{\gamma}_\ps := \frac{2}{3} \min_{t \in [T]} \hat{\gamma}_{\ps,t}$.
\end{enumerate}

The union bound then immediately implies that $\hat{\gamma}_\ps$ satisfies $\mathbb{P}(\hat{\gamma}_\ps \le \gamma_\ps) \geq 1 - \delta$.
However, for Eqn.~\eqref{eqn:gamma-ps} to hold, we require 
$H \ \gtrsim\ \gamma_\ps^{-3}\,\pi_{\min}^{-1}\,\log\!\frac{1}{\pi_{\min}}\;\log\!\frac{T}{\pi_{\min}}\;\log\!\frac{T}{\pi_{\min}\gamma_\ps}$~\citep[Theorem~1]{wolfer2024mixing}.
Compared to our requirement $H \gtrsim \gamma_\ps^{-1}\,\pi_{\min}^{-1}\,\log T$ (Theorem~\ref{thm:likelihood}), this introduces an extra $\gamma_\ps^{-2}$ factor and an \emph{extra} $\log T$ factor.
As a lower bound of $H \ \gtrsim\ \gamma_\ps^{-3} S \log\!\frac{T}{S}$ is known~\citep[Theorem~4]{wolfer2019mixing}, we conjecture that the extra $\gamma_\ps^{-2}$ factor is unavoidable.

As for the latter, we now describe a simple approach that could mitigate the extra $\log T$ factor:
\begin{enumerate}
    \item Sample an index set $\gI \subset [T]$ of cardinality $|\gI| = \ceil{\alpha_{\min}^{-1} \log\frac{K}{\delta}}$, uniformly at random.
    \item Repeat the ``na\"{i}ve'' approach, with $\frac{\delta}{T}$ replaced with $\frac{\delta}{|\gI|}$ in step 1 and $[T]$ replaced with $\gI$ in step 2.
\end{enumerate}
The following lemma shows that with high probability, $\gI$ ``covers'' all $K$ chains:
\begin{lemma}
\label{lem:cover}
    $\sP\left( \gI \cap f^{-1}(k) \neq \emptyset, \ \forall k \in [K] \right) \geq 1 - \delta$.
\end{lemma}
Under this event, then note that we only need to take the union bound over $\gI$, not $[T]$.
This then results in the \textit{additional}\footnote{Here, we emphasize that this requirement is \textit{in addition} to that of our Theorem~\ref{thm:likelihood}.} requirement of $H \ \gtrsim\ \gamma_\ps^{-3}\,\pi_{\min}^{-1}\,\log\!\frac{1}{\pi_{\min}}\;\log\!\frac{\log\frac{K}{\delta}}{\alpha_{\min}\pi_{\min}}\;\log\!\frac{\log\frac{K}{\delta}}{\alpha_{\min}\pi_{\min}\gamma_\ps}$, which becomes \textit{free} of any $\log T$ factors.

We leave to future work the precise details of the proof of our entire algorithm ($\gamma_\ps$ estimation, Stage I, and Stage II), but we believe that it should be straightforward.

\begin{proof}[Proof of Lemma~\ref{lem:cover}]
    This follows from a simple combinatorial argument:
    \begin{align*}
        \sP\left( \exists k \in [K] \ : \ \gI \cap f^{-1}(k) = \emptyset \right) &\leq \sum_{k \in [K]} \sP\left( \gI \cap f^{-1}(k) = \emptyset \right) \tag{union bound} \\
        &= \sum_{k \in [K]} \frac{\binom{T - \alpha_k T}{|\gI|}}{\binom{T}{|\gI|}} \\
        &= \sum_{k \in [K]} \prod_{i=0}^{|\gI|-1} \frac{T - \alpha_k T - i}{T - i} \\
        &\leq \sum_{k \in [K]} (1 - \alpha_k)^{|\gI|} \tag{the multiplicands monotonically decrease with $i$} \\
        &\leq K e^{-\alpha_{\min} |\gI|}.
    \end{align*}
    Setting $|\gI| = \ceil{\alpha_{\min}^{-1} \log\frac{K}{\delta}}$ and simplifying gives the desired statement.
\end{proof}

\newpage
\section{ADDITIONAL FUTURE DIRECTIONS}
\label{app:future}

In this section, we present several directions for future work that extend beyond those mentioned in the main text.

\paragraph{(1) Towards Computationally Efficient Algorithms.}
As noted in Section~\ref{sec:discussions}\textbf{(4)}, our algorithm prioritizes parameter-free design and statistical efficiency over computational scalability. The construction of the empirical data matrix $\widehat{\mW}$ and its subsequent exact SVD in Stage I entail $\gO(TS^2)$ space and $\gO\big(TS^2 \min(T, S^2)\big)$ time complexities, respectively. For environments with massive state spaces or trajectory counts, this can become prohibitively heavy. 

To alleviate this computational burden, future large-scale deployments could integrate randomized linear algebra techniques, such as matrix sketching or fast low-rank approximations~\citep{sketching}. Alternatively, one could adapt the clustering pipeline to streaming environments~\citep{yun2014streaming}, which would significantly reduce the memory footprint.
Exploring whether incorporating such approximations inherently induces a fundamental trade-off between computational scalability and the tight statistical efficiency bounds established in this work remains an important open question.

\paragraph{(2) Further Study of the $L$-Embedding.}
It is natural to explore how our $L$-based $\ell_2$ distance (Definition~\ref{def:L-embedding}) relates to broader information divergences beyond KL~\citep{wang2023divergence}, as well as to bisimulation metrics~\citep{calo2024bisimulation,calo2025bisimulation}.
Even modest equivalence or comparison results could clarify when $L$-geometry best captures dynamics and when alternative metrics offer sharper statistical or algorithmic behavior.

Relatedly, it is important to determine whether the requirement $H \gtrsim 1/\Delta_\mW^2$—or, more generally, an inverse-squared dependence on an $\ell_2$ separation induced by any Euclidean embedding of the chains—is fundamental or merely an artifact of our two-stage design. Establishing either (i) a matching lower bound phrased in terms of $\Delta_\mW$ (or any embedding that dominates it), or (ii) an alternative initialization whose sample complexity depends directly on the information gap $\gD_\pi$ rather than an $\ell_2$ proxy, would sharpen statistical limits and guide the choice of representations.
We believe this is related to refining the lower bound by considering the uncertainty of the probability kernels; see the last paragraph of Section~\ref{sec:discussions}\textbf{(1)}.

\paragraph{(3) Extension to MDPs.}
Extending our approach to estimation \emph{and} learning in mixtures of MDPs~\citep{banerjee2025markov,kausik2023mixture} is another important direction.
Progress here would connect with spectral methods in structured RL (e.g., POMDPs~\citep{azizzadenesheli2016spectral,azizzadenesheli2017spectral,jedra2023bmdp}, HMMs~\citep{anandkumar2012hmm,hsu2012hmm}, and latent MDPs~\citep{kwon2021lmdp,kwon2024lmdp1,kwon2024lmdp2}) and could reveal performance separations analogous to those observed when state abstraction or clustering is known \emph{a priori}.
We believe that this extension would yield unique challenges due to controllable actions and policy-induced distributions~\citep{jedra2023bmdp}; even harder when one considers history-dependent policies and active exploration-style approaches~\citep{tarbouriech2019active,tarbouriech2020active}.

\end{document}